\documentclass[10pt]{article}
\usepackage[left=1in,top=1in,right=1in,bottom=1in,letterpaper]{geometry}

\usepackage{amssymb,amsmath,enumerate}
\usepackage{latexsym,amsfonts,amscd,amsxtra,amstext}

\usepackage{hyperref}
\hypersetup{colorlinks=true, linkcolor=blue, citecolor=blue, urlcolor=blue, pdftitle={Lower Bounds and Accelerated Algorithms in Distributed Optimization with Communication Compression}, pdfauthor={Yutong He, Xinmeng Huang, Yiming Chen, Wotao Yin, Kun Yuan}}

\usepackage{graphicx}
\usepackage{amsmath}
\usepackage{amssymb}
\usepackage{empheq}
\usepackage{dsfont}
\usepackage{booktabs}
\usepackage{color}
\usepackage[table]{xcolor}
\usepackage[ruled,vlined]{algorithm2e}
\usepackage{threeparttable, booktabs}
\usepackage{subcaption}

\usepackage{mathtools}
\mathtoolsset{showonlyrefs=true}

\SetKwInput{KwInput}{Input}                
\SetKwInput{KwOutput}{Output}              
\SetKwInput{KwRequire}{Require}              

\newcommand{\ba}{\left[ \begin{array}}
\newcommand{\ea}{\\ \end{array} \right]}

\newcommand{\vg}{{\mathbf{g}}}

\newcommand{\vy}{{\mathbf{y}}}

\newcommand{\vS}{{\mathbf{S}}}

\newcommand{\cA}{{\mathcal{A}}}

\newcommand{\cC}{{\mathcal{C}}}

\newcommand{\cF}{{\mathcal{F}}}

\newcommand{\cN}{{\mathcal{N}}}
\newcommand{\cO}{{\mathcal{O}}}

\newcommand{\cU}{{\mathcal{U}}}
\newcommand{\cV}{{\mathcal{V}}}

\newcommand{\cX}{{\mathcal{X}}}

\newcommand{\EE}{\mathbb{E}}
\newcommand{\NN}{\mathbb{N}}
\newcommand{\PP}{\mathbb{P}}
\newcommand{\RR}{\mathbb{R}}

\newcommand{\ie}{\emph{i.e.}}
\newcommand{\eg}{\emph{e.g.}}
\newcommand{\prog}{\mathrm{prog}}

\newcommand{\vvvert}{{\vert\kern-0.25ex\vert\kern-0.25ex\vert}}

\allowdisplaybreaks

\newtheorem{theorem}{Theorem}
\newtheorem{assumption}{Assumption}
\newtheorem{remark}{Remark}
\newtheorem{lemma}{Lemma}
\newtheorem{definition}{Definition}

\newcommand{\BlackBox}{\rule{1.5ex}{1.5ex}} 
\newenvironment{proof}{\par\noindent{\bf Proof\ }}{\hfill\BlackBox\\[2mm]}

\definecolor{highlight_color}{rgb}{.1, 0,.8}
\definecolor{babyblue}{rgb}{0.54, 0.81, 0.94}
\definecolor{mayablue}{rgb}{0.45, 0.76, 0.98}
\definecolor{springgreen}{rgb}{0.0, 1.0, 0.5}
\definecolor{aquamarine}{rgb}{0.5, 1.0, 0.83}
\definecolor{periwinkle}{rgb}{0.68, 0.7, 1.0}
\definecolor{orchid}{rgb}{0.85, 0.44, 0.84}
\definecolor{limegreen}{rgb}{0.2, 0.8, 0.2}

\begin{document}

\title{Lower Bounds and Accelerated Algorithms in Distributed Stochastic Optimization with Communication Compression}

\author{%
Yutong He$^{1}$\thanks{Equal Contribution. Several preliminary results in this paper have been published in the conference paper \cite{Huang2022LowerBA}.},\hspace{2mm} Xinmeng Huang$^{2*}$, Yiming Chen$^{3}$, Wotao Yin$^4$, Kun Yuan$^{1}$\vspace{1mm}\\ 
    $^1$Peking  University\quad $^2$University of Pennsylvania \\
    $^3$MetaCarbon Inc.  \quad $^4$DAMO Academy, Alibaba Group \vspace{1mm}\\
{\small\texttt{\{yutonghe, kunyuan\}@pku.edu.cn \quad xinmengh@sas.upenn.edu}}\\{\small\texttt{ yiming@metacarbon.vip \quad  wotao.yin@alibaba-inc.com}} 
}

\date{}



\maketitle

\begin{abstract}
Communication compression is an essential strategy for alleviating communication overhead by reducing the volume of information exchanged between computing nodes in large-scale distributed stochastic optimization. Although numerous algorithms with convergence guarantees have been obtained, the optimal performance limit under communication compression remains unclear.

In this paper, we investigate the performance limit of distributed stochastic optimization algorithms employing communication compression. We focus on two main types of compressors, unbiased and contractive, and address  the best-possible convergence rates one can obtain with these compressors. We establish the lower bounds for the convergence rates of distributed stochastic optimization in six different settings, combining strongly-convex, generally-convex, or non-convex functions with unbiased or contractive compressor types. To bridge the gap between lower bounds and existing algorithms' rates, we propose NEOLITHIC, a nearly optimal algorithm with compression that achieves the established lower bounds up to logarithmic factors under mild conditions.  Extensive experimental results support our theoretical findings. 
This work provides insights into the theoretical limitations of existing compressors and motivates further research into fundamentally new compressor properties.

%

\end{abstract}

\section{Introduction}
In modern machine learning, distributed stochastic optimization plays a crucial role, as it involves multiple computing nodes processing a vast amount of data and model parameters. 
However, handling such a vast number of data samples and model parameters leads to significant communication overhead, which limits the scalability of distributed optimization systems. 
Communication compression, a technique employed in distributed stochastic optimization, aims to reduce the volume of information exchanged between nodes~\cite{Alistarh2017QSGDCS, Bernstein2018signSGDCO, Seide20141bitSG,Stich2018SparsifiedSW,Richtrik2021EF21AN,Philippenko2021PreservedCM}.
This is accomplished by transmitting compressed gradients, where gradient arrays are compressed at each node before they are shared, or by exchanging reduced-size model parameters.

Despite the emergence of numerous compression methods, their compression capacities are typically characterized by one of two properties: unbiased compressibility and contractive compressibility. These properties, irrespective of other specific characteristics of compression methods, are utilized in convergence analyses of distributed stochastic optimization algorithms that employ compression methods.
We refer the reader to~\cite{Beznosikov2020OnBC, Safaryan2020UncertaintyPF, Xu2020CompressedCF} for a summary of the two properties. Specifically, an \emph{unbiased compressor} $C$ randomly outputs $C(x)$ such that $\EE[C(x)]=x$ for any input vector $x$. In contrast, a \emph{contractive compressor} may yield a biased vector with a smaller variance. While faster theoretical convergence results~\cite{Beznosikov2020OnBC, Mishchenko2019DistributedLW,Horvath2019StochasticDL,Gorbunov2021MARINAFN} have been obtained for algorithms using unbiased compressors, contractive compressors can offer comparable and even superior empirical performance.

Communication compression results in information distortion in communication, which can cause an algorithm to take more iterations to convergence. Therefore,
when introducing an algorithm that incorporates communication compression, research papers typically present a performance upper bound, which serves as a guarantee that the algorithm will not exceed a specified number of iterations when applied to a particular class of input problems. Notable examples include quantized SGD~\cite{Alistarh2017QSGDCS,lu2020moniqua}, sparsified SGD~\cite{Stich2018SparsifiedSW,Wangni2018GradientSF}, and error compensation~\cite{Richtrik2021EF21AN,Karimireddy2019ErrorFF,Tang2019DoubleSqueezePS,Xie2020CSERCS}. However, the study of performance lower bounds remains largely unexplored. In contrast to upper bounds, lower bounds indicate the best-achievable performance of the technique when solving the worst-case optimization problem instances. 
In the context of distributed algorithms with communication compression, it is desirable for such studies to reveal the fundamental limit associated with using communication compression in a distributed algorithm. As all individual compression methods are classified as either an unbiased compressor or a contractive compressor~\cite{Beznosikov2020OnBC, Safaryan2020UncertaintyPF, Xu2020CompressedCF}, our lower-bound analyses focus on these two types of compressors. We aim to address two fundamental open questions:
\begin{itemize}
\item \textit{For a class of optimization problems (specified below) and a type of communication compressor, what is the convergence performance lower bound of the best-defending algorithm against the worst combination of problem instance and communication compressor?}

\item \textit{Is the lower bound tight? Does an existing algorithm achieve this bound? If not, can we develop new algorithms to attain it (up to logarithm factors)?}
\end{itemize}

\begin{table}[t]
\footnotesize
    \centering
    \caption{\small Lower and upper bounds for distributed algorithms 
    with unbiased compressors. Notations $\epsilon$ is the desired accuracy, $\omega$ is a  parameter associated with an unbiased compressor (see Assumption \ref{ass:unbiased}), and 
    $L,\mu, \Delta_f,\Delta_x,n,\sigma^2$ are defined in Section \ref{sec:prob}. In particular, $G$ is the Lipschitz constant of the objective functions, and notation $\tilde{\cO}$ hides logarithmic factors uncorrelated with the precision $\epsilon$. We use NC, GC, and SC as abbreviations for non-convex, generally-convex, and strongly-convex, respectively.}
    \begin{threeparttable}

    \begin{tabular}{cccc}
    \toprule
    \textbf{Method} & \textbf{NC} & \textbf{GC} & \textbf{SC}\\
    \midrule
\rowcolor{red!20}\textbf{L.B.} & $\Omega\left(\frac{L\Delta_f\sigma^2}{n\epsilon^2}+\frac{(1+\omega)L\Delta_f}{\epsilon}\right)$ & $\Omega\left(\frac{\Delta_x\sigma^2}{n\epsilon^2}+\frac{(1+\omega)\sqrt{L\Delta_x}}{\sqrt{\epsilon}}\right)$ & $\Omega\left(\frac{\sigma^2}{\mu n\epsilon}+(1+\omega)\sqrt{\frac{L}{\mu}}\ln\left(\frac{\mu\Delta_x}{\epsilon}\right)\right)$\\ \midrule
    EC-SGD \cite{Karimireddy2019ErrorFF}$^\Diamond$ & $\mathcal{O}\left(\frac{L\Delta_f\sigma^2}{\epsilon^2}+\frac{(1+\omega)L\Delta_f}{\epsilon}\right)$ & $\mathcal{O}\left(\frac{\Delta_x\sigma^2}{\epsilon^2}+\frac{(1+\omega)L\Delta_x}{\epsilon}\right)$ & $\tilde{\mathcal{O}}\left(\frac{\sigma^2}{\mu\epsilon}+\frac{(1+\omega)L}{\mu}\ln\left(\frac{1}{\epsilon}\right)\right)$\\
    Q-SGD \cite{Alistarh2017QSGDCS} & $\mathcal{O}\left(\frac{L\Delta_f((1+\omega)\sigma^2+\omega b^2)}{n\epsilon^2}\right)$\tnote{$\dagger$} & $\mathcal{O}\left(\frac{\Delta_x(\omega G^2+\sigma^2)}{n\epsilon^2}+\frac{L\Delta_x}{\epsilon}\right)$ & --- \\
    MEM-SGD \cite{Stich2018SparsifiedSW} & $\mathcal{O}\left(\frac{L\Delta_f\sigma^2}{n\epsilon^2}+\frac{(1+\omega)L\Delta_fG}{\epsilon^{3/2}}+\frac{L\Delta_f}{\epsilon}\right)$ & --- & \hspace{-3mm}$\mathcal{O}\left(\frac{G^2}{\mu\epsilon}+\frac{(1+\omega)\sqrt{L}G}{\mu\sqrt{\epsilon}}+\frac{(1+\omega)(\mu\Delta_x)^{1/3}}{\epsilon^{1/3}}\right)$\\
    D.S. \cite{Tang2019DoubleSqueezePS} & $\mathcal{O}\left(\frac{L\Delta_f\sigma^2}{n\epsilon^2}+\frac{(1+\omega)^2L\Delta_fG}{\epsilon^{3/2}}+\frac{L\Delta_f}{\epsilon}\right)$ & --- & ---\\
    CSER \cite{Xie2020CSERCS} & $\mathcal{O}\left(\frac{L\Delta_f\sigma^2}{n\epsilon^2}+\frac{(1+\omega)L\Delta_fG}{\epsilon^{3/2}}+\frac{L\Delta_f}{\epsilon}\right)$ & --- & ---\\
    EF21-SGD \cite{Fatkhullin2021EF21WB} & $\mathcal{O}\left(\frac{(1+\omega)^3L\Delta_f\sigma^2}{\epsilon^2}+\frac{(1+\omega)L\Delta_f}{\epsilon}\right)$ & --- & $\tilde{\mathcal{O}}\left(\left(\frac{(1+\omega)^3L\sigma^2}{\mu^2\epsilon}+\frac{(1+\omega)L}{\mu}\right)\ln\left(\frac{1}{\epsilon}\right)\right)$\\

    \rowcolor{red!20}\textbf{Ours} & $\tilde{\mathcal{O}}\left(\frac{L\Delta_f\sigma^2}{n\epsilon^2}+\frac{(1+\omega)L\Delta_f}{\epsilon}\right)$ & $\tilde{\mathcal{O}}\left(\frac{\Delta_x\sigma^2}{n\epsilon^2}+\frac{(1+\omega)\sqrt{L\Delta_x}}{\sqrt{\epsilon}}\right)$ & $\tilde{\mathcal{O}}\left(\frac{\sigma^2}{\mu n\epsilon}+(1+\omega)\sqrt{\frac{L}{\mu}}\ln\left(\frac{1}{\epsilon}\right)\right)$\\
    \bottomrule    
    \end{tabular}
    \begin{tablenotes}
    \item[$\Diamond$]The convergence analysis is under the single worker setting and cannot be extended to the distributed setting.
    \item[$\dagger$]This convergence rate is only achievable when $\epsilon=\mathcal{O}\left(\frac{(1+\omega)\sigma^2+\omega b^2}{n+\omega}\right)${\color{black}, where $b^2:=\sup_x\frac{1}{n}\sum_{i=1}^n\|\nabla f_i(x)-\nabla f(x)\|^2$ bounds gradient dissimilarity}.
    \end{tablenotes}
    \end{threeparttable}
    \label{tab:unbiased}
\end{table}
While identifying the best compressor is undoubtedly a valuable goal, our objective differs in that we are interested in the limits of existing compressor types. It is crucial to recognize that nearly all compressor performance analyses ultimately rely on one of the two properties: unbiasedness or contraction. This observation prompts a natural question: To enhance the convergence rate of distributed optimization with communication compression, should we continue utilizing these properties and focus on their more intelligent integration into distributed algorithms, or should we explore new compressor properties?

To address the question above, we must first determine the theoretical limits imposed by these two properties. If achievable performance remains significantly distant from the limit, it may indicate that current compressors have not been adequately utilized, suggesting that we should concentrate on devising clever integration of distributed algorithms and existing compressors. Conversely, if the limit is nearly reached (as demonstrated by NEOLITHIC), it becomes necessary to identify a new compressor property. This is not to say that all existing compressors are inherently flawed, but rather that their compression analyses have approached the limit. As a result, if one seeks to establish superior compressor performance beyond the lower bounds presented in this paper, it becomes imperative to uncover a fundamentally new compressor property.

As our goal is to uncover the performance boundaries of optimization algorithms with a type of compressor when confronted with the most challenging problem instance, we do not study the compression performance of the compression methods themselves~\cite{Safaryan2020UncertaintyPF}.

\begin{table}[t]
\footnotesize
    \centering
    \caption{\small Lower and upper bounds for distributed algorithms with contractive compressors. Notation $\delta$ is a parameter associated with a contractive compressor (see Assumption \ref{ass:contract}). The other notations are the same as in Table \ref{tab:unbiased}.}
    \begin{threeparttable}
    \begin{tabular}{cccc}
    \toprule
     \textbf{Method} & \textbf{NC} & \textbf{GC} & \textbf{SC}\\
    \midrule
    \rowcolor{blue!20}\textbf{L.B.} & $\Omega\left(\frac{L\Delta_f\sigma^2}{n\epsilon^2}+\frac{L\Delta_f}{\delta\epsilon}\right)$ & $\Omega\left(\frac{\Delta_x\sigma^2}{n\epsilon^2}+\frac{\sqrt{L\Delta_x}}{\delta\sqrt{\epsilon}}\right)$ & $\Omega\left(\frac{\sigma^2}{\mu n\epsilon}+\frac{1}{\delta}\sqrt{\frac{L}{\mu}}\ln\left(\frac{\mu\Delta_x}{\epsilon}\right)\right)$\\ \midrule
    EC-SGD \cite{Karimireddy2019ErrorFF}$^\diamond$ & $\mathcal{O}\left(\frac{L\Delta_f\sigma^2}{\epsilon^2}+\frac{L\Delta_f}{\delta\epsilon}\right)$ & $\mathcal{O}\left(\frac{\Delta_x\sigma^2}{\epsilon^2}+\frac{L\Delta_x}{\delta\epsilon}\right)$ & $\tilde{\mathcal{O}}\left(\frac{\sigma^2}{\mu\epsilon}+\frac{L}{\delta\mu}\ln\left(\frac{1}{\epsilon}\right)\right)$\\
      MEM-SGD \cite{Stich2018SparsifiedSW} & $\mathcal{O}\left(\frac{L\Delta_f\sigma^2}{n\epsilon^2}+\frac{L\Delta_fG}{\delta\epsilon^{3/2}}+\frac{L\Delta_f}{\epsilon}\right)$ & --- & $\mathcal{O}\left(\frac{G^2}{\mu\epsilon}+\frac{\sqrt{L}G}{\delta\mu\sqrt{\epsilon}}+\frac{\mu^{1/3}\Delta_x^{1/3}}{\delta\epsilon^{1/3}}\right)$\\
     D.S. \cite{Tang2019DoubleSqueezePS} & $\mathcal{O}\left(\frac{L\Delta_f\sigma^2}{n\epsilon^2}+\frac{L\Delta_fG}{\delta^2\epsilon^{3/2}}+\frac{L\Delta_f}{\epsilon}\right)$ & --- & ---\\
     CSER \cite{Xie2020CSERCS} & $\mathcal{O}\left(\frac{L\Delta_f\sigma^2}{n\epsilon^2}+\frac{L\Delta_fG}{\delta\epsilon^{3/2}}+\frac{L\Delta_f}{\epsilon}\right)$ & --- & ---\\
     EF21-SGD \cite{Fatkhullin2021EF21WB} & $\mathcal{O}\left(\frac{L\Delta_f\sigma^2}{\delta^3\epsilon^2}+\frac{L\Delta_f}{\delta\epsilon}\right)$ & --- & $\tilde{\mathcal{O}}\left(\left(\frac{L\sigma^2}{\delta^3\mu^2\epsilon}+\frac{L}{\delta\mu}\right)\ln\left(\frac{1}{\epsilon}\right)\right)$\\
     
     \rowcolor{blue!20}\textbf{Ours} & $\tilde{\mathcal{O}}\left(\frac{L\Delta_f\sigma^2}{n\epsilon^2}+\frac{L\Delta_f}{\delta\epsilon}\ln\left(\frac{1}{\epsilon}\right)\right)$ & $\tilde{\mathcal{O}}\left(\frac{\Delta_x\sigma^2}{n\epsilon^2}+\frac{\sqrt{L\Delta_x}}{\delta\sqrt{\epsilon}}\ln\left(\frac{1}{\epsilon}\right)\right)$ & $\tilde{\mathcal{O}}\left(\frac{\sigma^2}{\mu n\epsilon}+\frac{1}{\delta}\sqrt{\frac{L}{\mu}}\ln\left(\frac{1}{\epsilon}\right)\right)$\\
    \bottomrule    
    \end{tabular}
    \begin{tablenotes}
    \item[$\Diamond$]The convergence analysis is for the single-worker setting and cannot be extended to the distributed setting.
    \end{tablenotes}
    \end{threeparttable}
    \label{tab:contractive}
\end{table}

\subsection{Main Results}
This paper tackles the open questions mentioned earlier by presenting a series of lower bounds for distributed stochastic optimization and introducing a new unified algorithm that match these lower bounds up to logarithmic factors under mild conditions. Specifically, our contributions include:
\begin{itemize}
    \item We derive lower bounds for the convergence of distributed algorithms with communication compression in stochastic optimization, considering six different settings resulting from combining strongly-convex, generally-convex, or non-convex objective functions with unbiased or contractive compressor types. All these lower bounds are novel, and we observe a clear gap between them and the established complexities of existing algorithms.

    \item To address this gap, we propose \textbf{NEOLITHIC} (\underline{\textbf{Ne}}arly \underline{\textbf{O}}ptimal a\underline{\textbf{L}}gor\underline{\textbf{ith}}m w\underline{\textbf{i}}th  \underline{\textbf{C}}ompression). NEOLITHIC achieves the established lower bounds up to logarithmic factors under mild conditions, outperforming existing algorithms in the same setting. 
    Notably, NEOLITHIC is the first accelerated algorithm with convergence guarantees in the  setting of stochastic convex optimization.

    \item To support our analyses, we conduct comprehensive experiments. The results demonstrate that NEOLITHIC not only exhibits competitive convergence performance but also remains robust to data heterogeneity, gradient noises, and  choices of compressors.
    
\end{itemize}

We present the lower and upper bounds established in this paper, as well as the complexities of existing state-of-the-art distributed algorithms with \emph{unbiased} compressors, in Table \ref{tab:unbiased}. NEOLITHIC nearly achieves the lower bounds in all strongly-convex, generally-convex, and non-convex scenarios under additional mild assumptions. A similar superiority for contractive compressors can be found in Table \ref{tab:contractive}, except that NEOLITHIC is worse by a factor $\ln(1/\epsilon)$ than the derived lower bound in the generally-convex and non-convex scenarios.

This paper represents a significant advancement over our previous conference paper~\cite{Huang2022LowerBA}, which only studies the non-convex scenario. The novel results of this paper include new convergence lower bounds for convex objective functions and new NEOLITHIC variants employing Nesterov acceleration and multi-stage restarting strategies to attain lower bounds in different settings. To the best of our knowledge, this paper presents the first accelerated algorithm with convergence guarantees in the setting of stochastic convex optimization. Additionally, we generalize the original multi-step compression module to preserve unbiasedness when using unbiased compression, which is a critical factor for enhancing NEOLITHIC's performance with unbiased compressors. Finally, we provide supplementary experiments to further support our theories.

\subsection{Related Work}
\textbf{Distributed stochastic optimization.}
Distributed stochastic optimization is a prevalent technique in large-scale machine learning, where data is distributed across multiple worker nodes, and training is carried out through worker communication. 
However, the high cost of communication between workers can significantly impede the scalability of distributed stochastic algorithms. To address this issue, various communication techniques such as decentralized communication, 
lazy communication, and compressed communication, have been developed and proven to be highly effective. 

Decentralized communication focuses on determining who to communicate with during the optimization process. It allows each node to communicate with immediate neighbors, removing the need for global synchronization across all nodes that can incur significant bandwidth costs or high latency. Well-known decentralized algorithms include decentralized SGD \cite{nedic2009distributed,chen2012diffusion,Yuan2016OnTC,Yuan2016OnTC,Lian2017CanDA,koloskova2020unified,Yuan2022RevistOC,huang2022optimal}, D$^2$/Exact-Diffusion \cite{tang2018d,yuan2020influence,Yuan2021RemovingDH}, stochastic gradient tracking \cite{pu2021distributed, xin2020improved,koloskova2021improved,alghunaim2021unified}, and their momentum variants \cite{lin2021quasi,Yuan2021DecentLaMDM}. Lazy communication, on the other hand, focuses on determining when to communicate in optimization algorithms.  
It aims to save communication overhead by reducing communication frequency between workers. Lazy communication can be achieved by letting each worker either conduct a fixed number of multiple local updates before sending messages  \cite{yu2019linear,stich2019local,mishchenko2022proxskip}, or adaptively skip communications when necessary \cite{chen2018lag,liu2019communication}. Lazy communication is also widely used in federated learning \cite{mcmahan2017communication,karimireddy2020scaffold}. In contrast, this paper studies compressed communication, which investigates what to communicate within each iteration.

\vspace{1.5mm}
\noindent\textbf{Communication compression.}
There are two mainstream approaches to compression: quantization and sparsification. Quantization maps input vectors from a large set (\eg, $32$-bit numbers) to a smaller set of discrete values (\eg, $8$-bit numbers). Many quantization schemes, such as Sign-SGD \cite{Seide20141bitSG,Bernstein2018signSGDCO} which uses only 1 bit to represent each entry, are essentially unbiased operators with random noise. Generalized variants of Sign-SGD, like Q-SGD \cite{Alistarh2017QSGDCS}, TurnGrad \cite{Wen2017TernGradTG} and natural compression \cite{Horvath2019NaturalCF}, compress each entry with more flexible bits to enable a trade-off between compression ratio and precision. 

On the other hand, sparsification can be viewed as a biased but contractive operator. One popular approach to sparsification is to randomly drop some entries to achieve a sparse vector, as suggested by \cite{Wangni2018GradientSF}. Another approach, proposed by \cite{Stich2018SparsifiedSW}, is to transmit a subset of the largest elements in the model or gradient. The theoretical analyses of contractive compressors often make assumptions such as bounded gradients \cite{Zhao2019GlobalMC,Karimireddy2019ErrorFF} or quadratic loss functions \cite{Wu2018ErrorCQ}. Further discussions on both unbiased and biased compressors can be found in \cite{Beznosikov2020OnBC,Safaryan2020UncertaintyPF,Richtarik20223PCTP}. Communication compression can also be combined with other communication-saving techniques, such as decentralization \cite{Liu2021LinearCD,Zhao2022BEERFO} and lazy communication \cite{haddadpour2021federated}.

\vspace{1.5mm}
\noindent\textbf{Error compensation.} 
The technique of error compensation (feedback) is proposed in \cite{Seide20141bitSG} to address the compression error of 1-bit quantization. In \cite{Wu2018ErrorCQ}, the authors analyze stochastic gradient descent with error-compensated quantization for quadratic problems and provide convergence guarantees. Error compensation is also shown to be effective in reducing sparsification-induced error by \cite{Stich2018SparsifiedSW}. The convergence rate of error-compensated SGD in the non-convex scenario is studied in \cite{Alistarh2018TheCO}. Recently, \cite{Richtrik2021EF21AN} proposes a novel error feedback scheme, EF21, that compresses only the increment of local gradients and enjoys better theoretical guarantees in the deterministic regime. Additionally, \cite{Fatkhullin2021EF21WB} proposes an extension of EF21 that accommodates stochastic gradients without a linear speedup in the number of workers.

\vspace{1.5mm}
\noindent\textbf{Lower bounds.}
Lower bounds in optimization set a limit for the performance of a single or a group of algorithms. Prior research on lower bounds has established numerous lower bounds for optimization algorithms, particularly in convex scenarios \cite{Agarwal2015ALB,Diakonikolas2019LowerBF,Arjevani2015CommunicationCO,nesterov2003introductory,Balkanski2018ParallelizationDN,AllenZhu2018HowTM,Foster2019TheCO}. In non-convex scenarios, \cite{carmon2020lower,Carmon2021LowerBF} introduce a zero-chain model and provide a tight bound for first-order methods. \cite{Zhou2019LowerBF,Arjevani2019LowerBF} subsequently extend the approach to finite sum and stochastic problems. Recently, \cite{Lu2021OptimalCI,Yuan2022RevistOC,huang2022optimal} investigate the lower bound for decentralized stochastic optimization. In the field of  distributed stochastic optimization with communication compression, \cite{Philippenko2020BidirectionalCI} provides an algorithm-specific lower bound for strongly convex functions. 
However, prior to our work, no research had studied the algorithm-agnostic lower bounds.

\vspace{1.5mm}
\noindent\textbf{Accelerated algorithms with communication compression.} 
Nesterov acceleration \cite{nesterov1983method, nesterov2003introductory} is a crucial technique for accelerating algorithms to achieve optimal convergence rates in deterministic and convex optimization. However, the study of accelerated algorithms with communication compression is limited to a few works, such as \cite{li2020acceleration, li2021canita, qian2021error}. For example, ADIANA proposed in \cite{li2020acceleration} achieves a faster convergence rate in the strongly-convex scenario, while CANITA proposed in \cite{li2021canita} accelerates distributed communication compression in the generally-convex scenario. Both approaches are restricted to unbiased compressors and deterministic optimization. Furthermore, their analysis relies on the mutual independence of all worker-associated compressors, which may not hold in practical applications. Moreover, \cite{qian2021error} integrates Nesterov acceleration and variance reduction with error compensation to accelerate distributed finite-sum problems using contractive compression. However, this technique employs the special structure of the finite-sum problem and cannot be easily extended to the stochastic online setting. In this paper, we propose NEOLITHIC, the first accelerated algorithm for the stochastic online setting in both strongly-convex and generally-convex scenarios. Furthermore, we show that NEOLITHIC can nearly attain the lower bounds in these settings.

\section{Problem and Assumptions} \label{sec:prob}

In this section, we introduce the problem formulation and assumptions used throughout the paper.  Consider the following distributed stochastic optimization problem
\begin{equation}\label{eqn:prob}
	\min _{x \in \mathbb{R}^{d}}\quad f(x)=\frac{1}{n} \sum_{i=1}^{n} f_{i}(x) \quad\text{with}\quad f_{i}(x)=\mathbb{E}_{\xi_{i} \sim D_{i}} [F(x ; \xi_{i})],
\end{equation}
where the global objective function $f(x)$ is decomposed into $n$ local objective functions $\{f_i(x)\}_{i=1}^n$, and each local $f_i(x)$ is maintained by worker node $i$. Random variable $\xi_i$ represents the local data sample, and it follows a local distribution $D_i$. Since $D_i$ is typically unknown in advance, each node $i$ can only access its  stochastic gradient $\nabla F(x; \xi_i)$ per iteration rather than the true local gradient $\nabla f_i(x)$. In practice, the local data distribution $D_i$ within each node is generally different, and hence, $f_i(x) \neq f_j(x)$ holds for any nodes $i$ and $j$. Next, we introduce the setup under which we study the convergence rate.

\vspace{1mm}
\noindent \textbf{Notations.} In a variable $x_i^{(k)}$, the subscript $i$ indicates the node index and superscript $k$ indicates the iteration index. We let $\|\cdot\|$ denote the $\ell_2$ norm of vectors throughout the paper.

\subsection{Problem Setup}

\subsubsection{Function Class}
We let function class $\mathcal{F}^{\Delta_f}_L$ denote the set of non-convex and smooth functions satisfying Assumption \ref{asp:nonconvex}, and $\cF_{L,\mu}^{\Delta_x}$ ($0\leq \mu\leq L$) denote the set of convex and smooth functions satisfying Assumption  \ref{asp:convex}. Note that when $\mu>0$, $\cF_{L,\mu}^{\Delta_x}$ indicates strongly-convex functions and when $\mu=0$, $\cF_{L,\mu}^{\Delta_x}$ indicates generally-convex functions.

\begin{assumption}[\sc Non-convex and smooth function]\label{asp:nonconvex}
	We assume each $f_i(x)$ in $\cF_L^{\Delta_f}$ is $L$-smooth, \ie,
	\begin{align}\label{ass-eq-L-smooth}
	\left\|\nabla f_i(x)-\nabla f_i(y)\right\|\leq L\|x-y\|, \quad \forall x, y \in \RR^d, \quad \forall\, i\in \{1,\cdots, n\}, 
	\end{align}
	and $f(x^{(0)})-\inf_{x\in\mathbb{R}^d}f(x)\le\Delta_f$.
\end{assumption}
\begin{assumption}[\sc Convex and smooth function] \label{asp:convex}
	We assume each $f_{i}(x)$ in $\cF_{L, \mu}^{\Delta_x}$ is $L$-smooth, {\ie}, $f_i(x)$ satisfies \eqref{ass-eq-L-smooth}, and meanwhile it is  $\mu$-strongly convex, \ie, there exists a constant $\mu \ge 0$ such that 
	\begin{align}\label{ass-eq-mu-convex}
		f_i(y)\geq f_i(x)+\langle \nabla f_i(x), y-x\rangle +\frac{\mu}{2}\|y-x\|^2, \quad \forall x, y \in \RR^d, \quad \forall\, i\in  \{1,\cdots, n\}, 
	\end{align}
and $\|x^{(0)}-x^\star\|^2\leq \Delta_x$ where $x^\star$ is one of the global minimizers of $f(x)=\frac{1}{n}\sum_{i=1}^n f_i(x)$. When $\mu=0$, $f_i(x)$ reduces to a generally-convex function.
\end{assumption}

Typically, we assume that $\Delta_f$ and $\Delta_x$ are finite, which implies the lower bounded properties of function class $\cF_L^{\Delta_f}$ and $\cF_{L,\mu}^{\Delta_x}$. In the following part, we'll use notation $f^\star:=\min_xf(x)$, $f_i^\star:=\min_xf_i(x)$. 

\subsubsection{Gradient Oracle Class}
We assume each worker $i$ has  access to its local gradient $\nabla f_i(x)$ via a stochastic gradient oracle $O_i(x;\zeta_i)$ subject to independent random variables $\zeta_i$, \emph{e.g.}, the mini-batch sampling $\zeta_i\triangleq\xi_i\sim D_i$. We further assume that $O_i(x,\zeta_i)$ is an  unbiased estimator of the full-batch gradient $\nabla f_i(x)$ with a bounded variance. Formally, we let the {stochastic gradient oracle class $O_{\sigma^2}$} denote the set of all oracles $O_i$ satisfying Assumption \ref{asp:gd-noise}.
\begin{assumption}[\sc Gradient stochasticity]\label{asp:gd-noise}
	The stochastic gradient oracles $\{O_i:1\leq i\leq n\}$ satisfy
	\begin{align*}
		\EE_{\zeta_i}[O_i(x;\zeta_i)]=\nabla f_i(x)\quad \text{ and }\quad \EE_{\zeta_i}[\|O_i(x;\zeta_i)-\nabla f_i(x)\|^2]\leq \sigma^2,\quad \forall\, x\in\RR^d\text{ and }i\in \{1,\dots,n\}.
	\end{align*}
\end{assumption}

\subsubsection{Compressor Class}
The two  widely-studied classes of compressors  in literature are the $\omega$-\emph{unbiased} compressor, described by Assumption \ref{ass:unbiased}, \emph{e.g.}, the random quantization operator  \cite{Alistarh2017QSGDCS,Horvath2019NaturalCF}, and  the $\delta$-\emph{contractive} compressor, described by Assumption \ref{ass:contract}, \emph{e.g.},
the rand-$k$
\cite{qian2021error}
operator and top-$k$ operator \cite{Stich2018SparsifiedSW,qian2021error}. 
\begin{assumption}[\sc Unbiased compressor]\label{ass:unbiased}
	We assume the (possibly random) compression operator
	$C: \RR^d\rightarrow \RR^d$ satisfies
	\begin{equation*}
		\EE[C(x)]=x,\quad \EE[\|C(x)-x\|^2]\leq \omega\|x\|^2,\quad \forall\,x\in\RR^d
	\end{equation*}
	for constant $\omega \ge 0$, where the expectation is taken over the randomness of the compression operator $C$.
\end{assumption}
\begin{assumption}[\sc Contractive compressor]\label{ass:contract}
	We assume the (possibly random) compression operator
	$C: \RR^d\rightarrow \RR^d$ satisfies
	\begin{equation*}
		\EE[\|C(x)-x\|^2]\leq (1-\delta)\|x\|^2,\quad \forall\,x\in\RR^d
	\end{equation*}
	for constant $\delta \in(0,1]$, where the expectation is taken over the randomness of the compression operator $C$.
\end{assumption}
We let $\cU_{\omega}$ and $\cC_\delta$ denote
the set of all $\omega$-unbiased compressors and $\delta$-contractive compressors satisfying Assumptions \ref{ass:unbiased} and \ref{ass:contract}, respectively. Note that the identity operator $I$ satisfies $I \in \cU_\omega$ for all $\omega \ge 0$ and $I \in \cC_\delta$ for all $\delta \in (0,1]$. 
Generally speaking, an $\omega$-unbiased compressor is not necessarily contractive when $\omega$ is larger than $1$. However, since $C\in\cU_\omega$ implies $(1+\omega)^{-1}C\in \cC_{(1+\omega)^{-1}}$, the scaled unbiased compressor is contractive though the converse may not hold. For this reason, the class of contractive compressors is strictly richer since it contains all unbiased compressors through scaling.

\subsubsection{Algorithm Class}

We consider a centralized and synchronous algorithm $A$ in which first, every worker  is allowed to communicate only directly with a central server but not between one another; second, all iterations are synchronized, meaning that all workers start each of their iterations simultaneously; further, we assume that output $\hat{x}^{(t)}$ by the server---after $k$ iterations---can be any linear combination of all previous messages received by the server. 

We further require algorithms $A$ to satisfy the so-called ``zero-respecting'' property, which appears in \cite{carmon2020lower,Carmon2021LowerBF,Lu2021OptimalCI} (see formal definition in Appendix \ref{app:lower}). Intuitively, this property implies that the number of non-zero entries of the local parameters in a worker can be increased only by conducting local stochastic gradient queries or synchronizing with the server. The zero-respecting property holds with all algorithms in Tables \ref{tab:unbiased} and \ref{tab:contractive} and most first-order methods based on SGD \cite{Nesterov1983AMF,Kingma2015AdamAM,Huang2021Improved,Zeiler2012ADADELTAAA}. In addition to these properties, algorithm $A$ has to admit communication compression on all worker nodes, \ie, workers can only send compressed messages in communication. Specifically, we endow each worker $i \in \{1,\cdots, n\}$ with a compressor $C_i$. 
If $C_i = I$ for some $i\in\{1,\cdots, n\}$, then worker $i$ conducts lossless communication. Formally, we have the following definition for the  algorithms with communication compression that we consider throughout this paper.

\begin{definition}[\sc Algorithm class]\label{def:unidirect-alg}
	Given compressors $\{C_1,\cdots, C_n\}$, we let  $\cA_{\{C_i\}_{i=1}^n}$ denote the set of all centralized, synchronous, zero-respecting algorithms admitting  compression in which compressor $C_i$, $\forall\,1\leq i\leq n$, is applied to vectors sent from worker $i$ to the server. 
	
\end{definition}

\subsection{Complexity Metric}

With all interested classes introduced above, we are ready to define the complexity metric that we use for convergence analysis. Given a set of local functions $\{f_i\}_{i=1}^n\in\cF_L^{\Delta_f}$ or $\{f_i\}_{i=1}^n\in\cF_{L,\mu}^{\Delta_x}$,  
a set of stochastic gradient oracles $\{O_i\}_{i=1}^n\subseteq\cO_{\sigma^2}$,
a set of compressors $\{C_i\}_{i=1}^n\in \cU_{\omega}$ or $\{C_i\}_{i=1}^n\in\cC_{\delta}$, and an algorithm $A\in\cA_{\{C_i\}_{i=1}^n}$, we let $\hat{x}^{(t)}_{A}$ denote the output of algorithm $A$ after $t$ 
iterations. For convex functions $\{f_i\}_{i=1}^n\in\cF_{L,\mu}^{\Delta_x}$, the iteration complexity of $A$ solving $f(x)=\frac{1}{n}\sum_{i=1}^n f_i(x)$ under $\{(f_i,O_i,C_i)\}_{i=1}^n$ is defined as 
\begin{equation}\label{eqn:measure}
	T_\epsilon (A,\{(f_i,O_i,C_i)\}_{i=1}^n) =\min\left\{t\in \NN: \EE[f(\hat{x}^{(t)}_A)]-\min_x f(x)\leq \epsilon\right\},
\end{equation}
\ie, the smallest number of iterations required by $A$ to find an $\epsilon$-approximate optimum of $f(x)$ in expectation.
For non-convex functions $\{f_i\}_{i=1}^n\in\cF_L^{\Delta_f}$, it is generally impossible to find the global optimum. As a result, 
we define the iteration complexity of $A$ solving $f(x)=\frac{1}{n}\sum_{i=1}^n f_i(x)$ under $\{f_i, O_i, C_i)\}_{i=1}^n$ as 
\begin{equation}\label{eqn:measure-nonconvex}
	T_\epsilon(A,\{(f_i,O_i,C_i)\}_{i=1}^n)=\min\left\{t\in \NN: \EE[\|\nabla f(\hat{x}^{(t)}_A)\|^2]\le \epsilon\right\},
\end{equation}
\ie, the smallest number of iterations required by $A$ to find an $\epsilon$-stationary point of $f(x)$ in expectation. 

\section{Lower Bounds}\label{sec:lower-bound}
With all interested classes introduced above, we are ready to derive the lower bounds of the iteration complexities in distributed stochastic optimization with different convexities of objectives and properties of compressors. 

\subsection{Unbiased Compressor}
Our first result is for algorithms that admit $\omega$-unbiased compressors.

\begin{theorem}[\sc Unbiased compressor]\label{thm:lower-bounds}
 For any $L\ge\mu\ge0$, $n\ge2$, $\omega\ge0$, and $\sigma\ge0$,  the following results hold (proof is in Appendix \ref{app:lower-bounds}).
\begin{itemize}
\item \textbf{Strongly-convex: }For any $\Delta_x>0$, there exists a constant $c_\kappa$ only depends on $\kappa\triangleq L/\mu$, a set of local  functions $\{f_i\}_{i=1}^n\subseteq\cF_{L,\mu}^{\Delta_x}$, stochastic gradient oracles $\{O_i\}_{i=1}^n\subseteq\cO_{\sigma^2}$, unbiased compressors $\{C_i\}_{i=1}^n\subseteq\cU_\omega$, such that the output $\hat{x}$ of any $A\in\cA_{\{C_i\}_{i=1}^n}$ starting from $x^{(0)}$ requires
    \begin{equation}\label{eq-lower-bound-unbiased-sc}
        T_\epsilon(A,\{(f_i,O_i,C_i)\}_{i=1}^n)=\Omega\left(\frac{\sigma^2}{\mu n\epsilon} + (1+\omega)\sqrt{\frac{L}{\mu}}\ln\left(\frac{\mu\Delta_x}{\epsilon}\right)\right)
    \end{equation}
    iterations to reach  $\EE[f(\hat{x})]-f^\star\le\epsilon$ for any $0<\epsilon\leq c_\kappa L \Delta_x$.
    
    \item \textbf{Generally-convex: }For any $\Delta_x>0$, there exists a constant $c=\Theta(1)$, a set of local functions $\{f_i\}_{i=1}^n\subseteq\cF_{L,0}^{\Delta_x}$, stochastic gradient oracles $\{O_i\}_{i=1}^n\subseteq\cO_{\sigma^2}$, unbiased compressors $\{C_i\}_{i=1}^n\subseteq\cU_\omega$, such that the output $\hat{x}$ of any $A\in\cA_{\{C_i\}_{i=1}^n}$ starting from $x^{(0)}$ requires
    \begin{equation}\label{eq-lower-bound-unbiased-gc}
        T_\epsilon(A,\{(f_i,O_i,C_i)\}_{i=1}^n)=\Omega\left(\frac{\Delta_x\sigma^2}{n\epsilon^2} + (1+\omega)\left(\frac{L\Delta_x}{\epsilon}\right)^{\frac{1}{2}}\right)
    \end{equation}
    iterations to reach  $\EE[f(\hat{x})]-f^\star\le\epsilon$ for any $0<\epsilon\leq c\mu \Delta_x$.
    \item \textbf{Non-convex: } For any $\Delta_f>0$, there exists a constant $c=\Theta(1)$, a set of local functions $\{f_i\}_{i=1}^n\subseteq\cF_L^{\Delta_f}$, stochastic gradient oracles $\{O_i\}_{i=1}^n\subseteq\cO_{\sigma^2}$, unbiased compressors $\{C_i\}_{i=1}^n\subseteq\cU_\omega$, such that the output $\hat{x}$ of any $A\in\cA_{\{C_i\}_{i=1}^n}$ starting from $x^{(0)}$ requires
    \begin{equation}\label{eq-lower-bound-unbiased-nc}
        T_\epsilon(A,\{(f_i,O_i,C_i)\}_{i=1}^n)=\Omega\left(\frac{\Delta_fL\sigma^2}{n\epsilon^2} + \frac{(1+\omega)\Delta_f L}{\epsilon}\right)
    \end{equation}
    iterations to reach  $\EE[\|\nabla f(\hat{x})\|^2]\le\epsilon$ for any $0<\epsilon\leq cL \Delta_f$.
\end{itemize}
\end{theorem}

\noindent \textbf{Influence of communication compression.} 
The lower bounds presented in Theorem \ref{thm:lower-bounds} consist of two terms: a sample complexity term (the first term) which determines the number of gradient samples required to achieve an $\epsilon$-accurate solution, and a communication complexity term (the second term) that determines the number of communication rounds. Communication compression only affects the communication complexity but not the sample complexity. Moreover, the rounds of communication needed to attain an $\epsilon$-accurate solution increase linearly with compression error $\omega$, which aligns with our intuition. In addition, the sample complexity decreases linearly as the number of nodes $n$ increases. As more nodes join distributed stochastic optimization, the communication complexity gradually becomes the dominant term.

\vspace{1.5mm}
\noindent \textbf{Lower bounds for deterministic optimization.} It is worth noting that Theorem \ref{thm:lower-bounds} also establishes the convergence lower bound for distributed {\em deterministic} optimization with unbiased compression when the gradient noise $\sigma^2 = 0$. Notably, none of these lower bounds have been derived in existing literature.

\vspace{1.5mm}
\noindent \textbf{Consistency with prior works.} 
The lower bounds established in Theorem \ref{thm:lower-bounds} are consistent with the best-known lower bounds in previous literature. When $\omega = 0$, our result reduces to the tight bound for distributed training without compression \cite{li2014communication}. When $n=1$ and $\omega=0$, our result reduces to the lower bound established in \cite{Arjevani2019LowerBF,Foster2019TheCO} for single-node stochastic optimization. When $n=1, \omega=0$ and $\sigma^2=0$, our result recovers the tight bound
for deterministic optimization \cite{carmon2020lower,nesterov1983method,nesterov2003introductory}. 

\subsection{Contractive Compressor}
To obtain lower bounds for contractive compressors, we need the following lemma \cite[Lemma 1]{Safaryan2020UncertaintyPF}.

\begin{lemma}[\sc Compressor relation]\label{lem:un-con}
	It holds that $\delta\,\cU_{\delta^{-1}-1} \triangleq\{\delta \,C:C\in \cU_{\delta^{-1}-1}\}\subseteq \cC_{\delta}$.
\end{lemma}
The Lemma above establishes that any compressor that is $(\delta^{-1}-1)$-unbiased is $\delta$-contractive when it is scaled by a factor of $\delta$. Consequently, if an algorithm $A$ is compatible with all compressors that are $\delta$-contractive, then it is  automatically compatible with all compressors that are in the set $\delta\,\cU_{\delta^{-1}-1}$, thanks to Lemma \ref{lem:un-con}. This relationship, combined with Theorem \ref{thm:lower-bounds}, enables us to obtain a lower bound that is specific to $\delta$-contractive compressors.

\begin{theorem}\label{thm:lower-bounds-contractive}
For any $L\ge\mu\ge0$, $n\ge2$, $0<\delta\le1$, $\sigma>0$, the  results below hold (proof is in Appendix \ref{app-lower-bounds-contractive}).
\begin{itemize}
\item \textbf{Strongly-convex: }For any $\Delta_x>0$, there exists a constant ${c_\kappa}$ only depends on $\kappa\triangleq L/\mu$, a set of local functions $\{f_i\}_{i=1}^n\subseteq\cF_{L,\mu}^{\Delta_x}$, stochastic gradient oracles $\{O_i\}_{i=1}^n\subseteq\cO_{\sigma^2}$, contractive compressors $\{C_i\}_{i=1}^n\subseteq\cC_\delta$, such that the output $\hat{x}$ of any $A\in\cA_{\{C_i\}_{i=1}^n}$ starting from $x^{(0)}$ requires
    \begin{equation}\label{eq-lower-bound-contractive-sc}
        T_\epsilon(A,\{(f_i,O_i,C_i)\}_{i=1}^n)=\Omega\left(\frac{\sigma^2}{\mu n\epsilon} + \frac{1}{\delta}\sqrt{\frac{L}{\mu}}\ln\left(\frac{\mu\Delta_x}{\epsilon}\right)\right)
    \end{equation}
    iterations to reach $\EE[f(\hat{x})]-f^\star\le\epsilon$ for any $0<\epsilon\leq {c_\kappa}\mu \Delta_x $.
    \item \textbf{Generally-convex: }For any $\Delta_x>0$, there exists a constant $c=\Theta(1)$, a set of local functions $\{f_i\}_{i=1}^n\subseteq\cF_{L,0}^{\Delta_x}$, stochastic gradient oracles $\{O_i\}_{i=1}^n\subseteq\cO_{\sigma^2}$, contractive compressors $\{C_i\}_{i=1}^n\subseteq\cC_\delta$, such that the output $\hat{x}$ of any $A\in\cA_{\{C_i\}_{i=1}^n}$ starting from $x^{(0)}$ requires
    \begin{equation}\label{eq-lower-bound-contractive-gc}
        T_\epsilon(A,\{(f_i,O_i,C_i)\}_{i=1}^n)=\Omega\left(\frac{\Delta_x\sigma^2}{n\epsilon^2} + \frac{1}{\delta}\left(\frac{L\Delta_x}{\epsilon}\right)^{\frac12}\right)
    \end{equation}
    iterations to reach $\EE[f(\hat{x})]-f^\star\le\epsilon$ for any $0<\epsilon\leq cL\Delta_x$.
    \item \textbf{Non-convex: }For any $\Delta_f>0$, there exists a constant $c=\Theta(1)$, a set of local functions $\{f_i\}_{i=1}^n\subseteq\cF_L^{\Delta_f}$, stochastic gradient oracles $\{O_i\}_{i=1}^n\subseteq\cO_{\sigma^2}$, contractive compressors $\{C_i\}_{i=1}^n\subseteq\cC_\delta$, such that the output $\hat{x}$ of any $A\in\cA_{\{C_i\}_{i=1}^n}$ starting from $x^{(0)}$ requires
    \begin{equation}\label{eq-lower-bound-contractive-nc}
        T_\epsilon(A,\{(f_i,O_i,C_i)\}_{i=1}^n)=\Omega\left(\frac{\Delta_fL\sigma^2}{n\epsilon^2} + \frac{\Delta_fL}{\delta\epsilon}\right)
    \end{equation}
    iterations to reach $\EE[\|\nabla f(\hat{x})\|^2]\le\epsilon$ for any $0<\epsilon\leq cL \Delta_f$.
\end{itemize}
\end{theorem}
\noindent \textbf{Influence of communication compression.} 
The lower bounds presented in Theorem \ref{thm:lower-bounds-contractive} are comprised of two terms: a sample complexity term (the first term) and a communication complexity term (the second term). It is worth noting that the number of communication rounds required to achieve an $\epsilon$-accurate solution is inversely proportional to $\delta$. Thus, a less precise contractive compressor with a smaller $\delta$ value will incur more communication rounds to attain an $\epsilon$-accurate solution, which well aligns with our intuition.

\vspace{1.5mm}
\noindent \textbf{Lower bounds for deterministic optimization.} Theorem \ref{thm:lower-bounds-contractive} also establishes the convergence lower bound for distributed {\em deterministic} optimization with contractive compression when the gradient noise $\sigma^2 = 0$. Notably, none of these lower bounds for deterministic optimization have been derived in existing literature.

\section{NEOLITHIC: Nearly Optimal Algorithms}

By comparing the complexities in Tables \ref{tab:unbiased} and \ref{tab:contractive} with the established lower bounds in Theorems \ref{thm:lower-bounds} and \ref{thm:lower-bounds-contractive}, it becomes clear that existing algorithms are not optimal as there is a noticeable gap between their convergence rates and our established lower bounds. To bridge this gap, we propose NEOLITHIC in this section. NEOLITHIC achieves convergence rates that nearly match the established lower bounds in Theorems \ref{thm:lower-bounds} and \ref{thm:lower-bounds-contractive} up to logarithmic factors and under mild conditions. It can work with strongly-convex, generally-convex, and non-convex scenarios, and is compatible with both unbiased and contractive compressors. Before the development of the NEOLITHIC algorithm, we first introduce a novel multi-step compression module, which plays a critical role in helping NEOLITHIC to attain state-of-the-art convergence rates.

\begin{algorithm}[t]
\caption{Multi-step compression module: $\text{MSC}(x,C,R)$}
\KwIn {Vector $x$ to be compressed; contractive operator \colorbox{blue!20}{$C\in\cC_\delta$} or unbiased operator \colorbox{red!20}{$C\in\cU_\omega$}; number of compressed communication steps $R$.}
\textbf{Initialize} $v^{(0)}=0$\;
\For{$r=1,\cdots,R$}{
Compress $x-v^{(r-1)}$ to achieve $c^{(r-1)}=C(x-v^{(r-1)})$\;
Send $c^{(r-1)}$ to the receiver\;
\hspace{-1mm}\colorbox{blue!20}{Update $v^{(r)} = v^{(r-1)}+c^{(r-1)}$ if $C\in\cC_\delta$}\;
\hspace{-1mm}\colorbox{red!20}{Or update $v^{(r)} = v^{(r-1)}+(1+\omega)^{-1}c^{(r-1)}$ if $C\in\cU_\omega$}\;
}
\Return \colorbox{blue!20}{$v^{(R)}$ if $C\in\cC_\delta$} or \colorbox{red!20}{$\left[1-(\omega/(1+\omega))^R\right]^{-1}v^{(R)}$ if $C\in\cU_\omega$}.
\label{alg:MSC}
\end{algorithm}

\subsection{Multi-Step Compression}
Multi-step compression (MSC) aims to achieve a flexible trade-off between compression precision and communication rounds. It is based on an base compressor which can be either unbiased or contractive, and can achieve an arbitrarily high compression precision by simply increasing the communication rounds, without the need to tune other parameters associated with the base compressor. For this reason, MSC is a great fit for scenarios where high compression precision is required in algorithmic development, such as with Nesterov acceleration. The MSC module  is listed in Algorithm \ref{alg:MSC}. 

\begin{itemize}
    \item \textbf{Input arguments.} The input  $x$ is the vector to be compressed, $R$ is the number of compressed communication steps, and $C\in\cC_\delta$ or $C\in\cU_\omega$ is the base contractive or unbiased compressor. 

    \item \textbf{Transmitted variables.} The MSC module conducts a total of $R$  compressed communication steps, where each step $r \in \{1,\cdots, R\}$ involves transmitting a compressed vector $c^{(r)}$ to the receiver. After MSC finishes, the receiver obtains a set of compressed vectors $\{c^{(r)}\}_{r=0}^{R-1}$. It is important to note that each transmitted vector $c^{(r-1)}$ is a compressed array under the compressor rule specified by Assumptions \ref{ass:unbiased} or \ref{ass:contract}. This property ensures that MSC can be used to develop optimal algorithms that attain the lower bounds without violating  the protocol of communication compression. In addition, MSC reduces to the normal single-step contractive or unbiased compressor utilized in existing literature \cite{Tang2019DoubleSqueezePS,Stich2018SparsifiedSW,Fatkhullin2021EF21WB,Xie2020CSERCS,Jiang2018ALS} when $R=1$. 

    \item \textbf{Returned values.} The output of the MSC module with contractive compressors is the vector $v^{(R)}$, while with unbiased compressors, it is $\left[1-(\omega/(1+\omega))^R\right]^{-1}v^{(R)}$. The scaling factor $1-(\omega/(1+\omega))^R$ is necessary to preserve  unbiasedness of the output, \ie, $\EE[v^{(R)}]=v^{(0)}$. It is important to note that the returned value is not directly transmitted to the receiver. Instead, the receiver will recover $v^{(R)}$ upon receiving the set of compressed messages $\{c^{(r)}\}_{r=0}^{R-1}$. Specifically, it is easy to verify that $v^{(R)} = \sum_{r=0}^{R-1}c^{(r)}$ with contractive compressors, and $v^{(R)} = (1+\omega)^{-1}\sum_{r=0}^{R-1}c^{(r)}$ with unbiased compressors.
\end{itemize}

The following lemma establishes that the compression error diminishes exponentially fast as $R$ increases. When $R=1$, the following lemma reduces to the property specified in Assumptions \ref{ass:unbiased} and \ref{ass:contract}. 
\begin{lemma}[\sc MSC property]\label{lm:MSC}
Under Assumptions \ref{ass:unbiased} and \ref{ass:contract}, the following results hold for any $R\ge 1$ (Proof is in Appendix \ref{app-MSC-lemma}).
\begin{itemize}
    \item \textbf{Contractive compressor:} If $C\in\cC_\delta$, the returned value of MSC module satisfies
\begin{align}
    \EE[\|\mathrm{MSC}(x,C,R)-x\|^2]\le(1-\delta)^R\|x\|^2,\quad \forall\,x\in\RR^d.
\end{align}
\item \textbf{Unbiased compressor:} If $C\in\cU_\omega$, the returned value of MSC module satisfies
\begin{align}
    \mathbb{E}[\mathrm{MSC}(x,C,R)]&=x,\quad\forall x\in\mathbb{R}^d\\
        \mathbb{E}[\|\mathrm{MSC}(x,C,R)-x\|^2]&\le(1+\omega)\left(\frac{\omega}{1+\omega}\right)^R\|x\|^2,\quad \forall\, x\in\mathbb{R}^d
\end{align}
where the expectation is taken over the randomness of the compression operator $C$.
\end{itemize}
\end{lemma}

Lemma \ref{lm:MSC} demonstrates that the MSC module can attain arbitrary precision by increasing the number of compressed communication steps. Furthermore, the compression error decreases exponentially with $R$, implying that a slight increase in $R$ can result in significant improvements in compression precision. The MSC module is closely related to the EF21 compression strategy \cite{Richtrik2021EF21AN}. If we switch the roles of $v$ and $v^\star$ in  \cite[Eq. (8)]{Richtrik2021EF21AN} and let $R=1$, we obtain a single step of the MSC module with contractive compressor. However, the main contribution of MSC lies in utilizing multiple such compression rounds to increase the compression precision and enable the development of algorithms that can approach the established lower bounds.

\subsection{The NEOLITHIC Algorithm}

NEOLITHIC builds upon the vanilla stochastic compressed gradient descent method \cite{li2020acceleration,Alistarh2017QSGDCS} and incorporates three key enhancements: Nesterov acceleration \cite{nesterov1983method,nesterov2003introductory}, stochastic gradient accumulation \cite{Lu2021OptimalCI,Yuan2021RemovingDH}, and multi-step compression (MSC) discussed in the last section.

In NEOLITHIC, the server runs a standard Nesterov accelerated algorithm while receiving compressed stochastic gradients from each worker. Since Nesterov acceleration is not well-suited to inexact gradients with large stochastic variance and severe compression bias, each worker must refine their stochastic gradient estimate and communication compression,  thus motivating the use of gradient accumulation and MSC within each worker. To improve the stochastic gradient estimate, each worker accumulates $R$ stochastic gradients per iteration. To compensate for compression error, each worker uses MSC to transmit messages. As the number of communication rounds $R$ increases, each worker can provide arbitrarily-accurate gradient estimates, which makes Nesterov acceleration useful to improve the convergence rate. 

\begin{algorithm}[t]
\caption{NEOLITHIC}
\KwIn{Hyperparamters $\eta$, $p$, $\{\gamma_k\}_{k=0}^{K-1}$, $R$}
Initialize $x^{(0)}=z^{(0)}$\;
\For{$k=0,1,\cdots,K-1$}{
\textbf{On server:}\\
\quad Generate point to query $y^{(k)}=\left(1-\frac{\gamma_k}{p}\right)x^{(k)}+\frac{\gamma_k}{p}z^{(k)}$\tcc*{Server sends $y^{(k)}$ to workers}
\textbf{On all workers in parallel:}\\
\quad Query stochastic gradients $g_i^{(k)}=\frac{1}{R}\sum_{r=0}^{R-1}\nabla F (y^{(k)};\xi_i^{(k,r)})$\tcc*{Gradient accumulation}
\quad Multi-step Compression $\hat{g}_i^{(k)}\hspace{-0.5mm} = \hspace{-0.5mm} \mathrm{MSC}(g_i^{(k)}\hspace{-1mm},C_i,R)$\tcc*{Worker $i$ sends $\{c_i^{(r)}\}_{r=0}^{R-1}$ to server}
\textbf{On server:}\\
\quad Gather gradients $\hat{g}^{(k)}=\frac{1}{n}\sum_{i=1}^n\hat{g}_i^{(k)}$\;
\quad Update model parameter $x^{(k+1)}=y^{(k)}-\frac{\eta}{p}\hat{g}^{(k)}$\;
\quad Update auxiliary parameter $z^{(k+1)}=\frac{1}{\gamma_k}x^{(k+1)}+\left(\frac{1}{p}-\frac{1}{\gamma_k}\right)x^{(k)}+\left(1-\frac{1}{p}\right)z^{(k)}$\;
}
\Return $\hat{x}^{(K)}=x^{(K)}$ for convex functions or $\hat{x}^{(K)}\sim {\rm Unif}(\{x^{(k)}\}_{k=0}^{(K)})$ for non-convex functions.
\label{alg:NEOLITHIC}
\end{algorithm}

The NEOLITHIC algorithm is listed in Algorithm \ref{alg:NEOLITHIC}. The hyperparameters $p$ and $\gamma_k$ will take different values for strongly-convex, generally-convex, and non-convex scenarios, see Section \ref{sec:convergence}. Compared to other algorithms listed in Tables \ref{tab:unbiased} and \ref{tab:contractive}, the proposed NEOLITHIC takes $R$ times more  gradient queries and communication rounds than them per iteration. Given the same budgets to query gradient oracles and conduct communication as the other algorithms, say $T$ times on each worker, we shall consider $K = T/R$ iterations in NEOLITHIC for fair comparison. 

\begin{remark}[\sc Extension to bidirectional compression] The NEOLITHIC algorithm presented in Algorithm \ref{alg:NEOLITHIC} employs unidirectional compression, which only compresses messages from workers to the server. However, our preliminary work \cite{Huang2022LowerBA} demonstrates that NEOLITHIC can be extended to the bidirectional compression scenario.
\end{remark}

\subsection{Multi-Stage NEOLITHIC}




While NEOLITHIC is capable of achieving state-of-the-art convergence rates, our analysis of  NEOLITHIC requires a constant $\eta$, which depends on $K$, the total number of outer loops, but not  $k$, the index of each loop, see Appendix \ref{app:upper-sconvex-single} for the details. This results in a sub-optimal rate of  NEOLITHIC in the strongly-convex case. To overcome this issue, we introduce a variant of NEOLITHIC called multi-stage NEOLITHIC, which restarts NEOLITHIC with different values of $\eta$ to provide improved theoretical performance guarantees.

Multi-stage NEOLITHIC is inspired by multi-stage accelerated stochastic approximation \cite{ghadimi2013optimal} which possess an optimal convergence rate for solving strongly-convex stochastic composite optimization problems. Our proposed algorithm involves running NEOLITHIC in each stage, with the results obtained from the previous stage serving as the initialization.  Furthermore, the stopping criterion is gradually tightened by a factor of two for each new stage and the choice of $\eta$ varies in stages.
Multi-stage NEOLITHIC is listed in Algorithm \ref{alg:NEOLITHIC-mul}. The superscript $[s]$ indicates the stage-index, while superscript $(k)$ indicates the iteration index within each stage. All stage-wise hyperparameters will be determined in convergence analysis.

\begin{algorithm}
\caption{Multi-stage NEOLITHIC}
\KwIn{number of stages $S$; stage-wise hyperparameters $\{K^{[s]}\}_{s=0}^{S-1}$, $\{R^{[s]}\}_{s=0}^{S-1}$, $\{\eta^{[s]}\}_{s=0}^{S-1}$, $\{p^{[s]}\}_{s=0}^{S-1}$, $\{\gamma_k^{[s]}\}_{0\le s<S-1,0\le k<K^{[s]}}$; Initialize $x^{[0]}$\;} 

\For{$s=0,\cdots,S-1$}{
    Call NEOLITHIC (Algorithm \ref{alg:NEOLITHIC}) with initialization $x^{(0)}=z^{(0)}=x^{[s]}$ and hyperparameters $K^{[s]}$, $R^{[s]}$, $\eta^{[s]}$, $p^{[s]}$, and $\{\gamma_{k}^{[s]}\}_{k=0}^{K^{[s]}-1}$\;
    Let the output of the above NEOLITHIC algorithm be $x^{[s+1]}$\;
}
\Return $x^{[S]}$.
\label{alg:NEOLITHIC-mul}
\end{algorithm}

\section{Convergence Analysis}\label{sec:convergence}

This section aims to establish the complexities of NEOLITHIC for different objective functions, with both contractive and unbiased compressors. 

\subsection{Strongly-Convex Scenario}
The following theorem establishes that the complexity lower bounds for strongly-convex scenario can be nearly-attained by multi-stage NEOLITHIC.

\begin{theorem}[\sc Strongly-convex scenario]\label{thm:upper-sc-contractive} Given $n \ge 1$, precision $\epsilon>0$, and $G^\star=f^\star-\frac{1}{n}\sum_{i=1}^nf_i^\star$ in which $f^\star$ and $f_i^\star$ are minimum of function $f$ and $f_i$ in problem \eqref{eqn:prob}, if we let $S=\lceil\log_2(L\Delta_x/\epsilon)\rceil$ and set stage-wise hyperparameters $K^{[s]}$, $R^{[s]}$, $\eta^{[s]}$, $\{\gamma_k\}_{k=0}^{K^{[s]}-1}$ as stated in Appendix \ref{app:upper-sconvex}, the following results hold (proof is in Appendix \ref{app:upper-sconvex}). 
\begin{itemize}
    \item \textbf{Contractive compressor:} For any $\{f_i\}_{i=1}^n\subseteq\cF_{L,\mu}^{\Delta_x}$  with $\mu>0$, $\{O_i\}_{i=1}^n \subseteq \cO_\sigma^2$, and $\{C_i\}_{i=1}^n \subseteq \cC_\delta$, multi-stage NEOLITHIC requires 
    \begin{align}\label{eq-upper-bound-contractive-sc}
    \tilde{\cO}\left(\frac{\sigma^2}{\mu n\epsilon} + \frac{1}{\delta}\sqrt{\frac{L}{\mu}}\ln\left(\frac{1}{\epsilon}\right)\right)
    \end{align}
    total number of iterations to achieve an $\epsilon$-approximate optimum, where notation $\tilde{\cO}(\cdot)$ hides all logarithm factors of $n,\mu,L,\sigma,G^\star,\Delta_x$ and $\delta$.

    \item \textbf{Unbiased compressor:} For any $\{f_i\}_{i=1}^n\subseteq\cF_{L,\mu}^{\Delta_x}$ with $\mu>0$, $\{O_i\}_{i=1}^n \subseteq \cO_\sigma^2$, and $\{C_i\}_{i=1}^n \subseteq \cU_\omega$, multi-stage NEOLITHIC requires 
    \begin{align}\label{eq-upper-bound-unbiased-sc}
    \tilde{\cO}\left(\frac{\sigma^2}{\mu n\epsilon} + (1+\omega)\sqrt{\frac{L}{\mu}}\ln\left(\frac{1}{\epsilon}\right)\right)
    \end{align}
    total number of iterations to achieve an $\epsilon$-approximate optimum, where notation $\tilde{\cO}(\cdot)$ hides all logarithm factors of $n,\mu,L,\sigma,G^\star,\Delta_x$ and $\omega$.
\end{itemize}

\end{theorem}

\begin{remark} [\sc total number of iterations]
Theorem \ref{thm:upper-sc-contractive} presents the complexity in terms of the total number of iterations required by the entire optimization process. To illustrate the total number of iterations, we assume the number of loops in NEOLITHIC is fixed per stage and the number of iterations within MSC is constant per NEOLITHIC loop. Based on this assumption, we can determine the total number of iterations as follows:
\begin{align}
\text{Total number of iterations} = \text{Number of stages}&\times \text{Number of loops in NEOLITHIC} \\
& \times  \text{Number of iterations within MSC}. 
\end{align}
This count is also equivalent to the total number of sample queries and the total number of communication/compression rounds. 
\end{remark}

\begin{remark} [\sc Tightness of the lower bounds]
Comparing the upper bounds in \eqref{eq-upper-bound-contractive-sc} and \eqref{eq-upper-bound-unbiased-sc} with the corresponding lower bounds in \eqref{eq-lower-bound-unbiased-sc} and \eqref{eq-lower-bound-contractive-sc} for the strongly-convex scenario, we observe that multi-stage NEOLITHIC nearly attains the lower bounds up to logarithm terms. This implies that lower bounds in \eqref{eq-lower-bound-unbiased-sc} and \eqref{eq-lower-bound-contractive-sc} are nearly tight, and multi-stage NEOLITHIC is nearly optimal.
\end{remark}

\begin{remark} [\sc Performance of NEOLITHIC]
The validity of multi-stage NEOLITHIC's near-optimal convergence for the strongly-convex scenario in Theorem \ref{thm:upper-sc-contractive} relies on the theoretical performance of NEOLITHIC (Algorithm \ref{alg:NEOLITHIC}). In Appendix \ref{app:upper-sconvex-single}, we demonstrate that NEOLITHIC can achieve an upper bound that is slower than the lower bound by a factor of $\ln(1/\epsilon)$. In contrast, multi-stage NEOLITHIC can effectively remove the $\ln(1/\epsilon)$ term and nearly achieve the lower bounds. 
\end{remark}

\begin{remark} [\sc Independent unbiased compressors] 
In the deterministic scenario, where $\sigma^2=0$, multi-stage NEOLITHIC achieves a convergence rate of $\tilde{\cO}\big((1+\omega)\sqrt{L/\mu}\ln(1/\epsilon)\big)$ with unbiased compressors. Although some existing literature can outperform NEOLITHIC in certain scenarios, \eg, \cite{li2020acceleration} converges with rate $\tilde{\cO}(\omega(1+\sqrt{L/(n\mu)})\ln(1/\epsilon))$ when $n\leq \omega$, this improvement is based on an additional assumption that all local unbiased compressors $\{C_i\}_{i=1}^n$ are {independent and cannot share the same randomness}. Similarly, recent works \cite{li2021canita} and \cite{Gorbunov2021MARINAFN,Tyurin2022DASHADN} have achieved cheaper complexities than our derived lower bounds in Theorem \ref{thm:lower-bounds} and \ref{thm:lower-bounds-contractive} for generally-convex and non-convex scenarios, respectively, by using independent unbiased compressors.

In contrast, NEOLITHIC does not require such an assumption and can be applied to both dependent and independent compressors. This is particularly useful in practical distributed machine learning systems, where the high-performance ring-allreduce protocol \cite{patarasuk2009bandwidth} is used to conduct global averaging. This protocol cannot support independent compressors, as it requires all compressed vectors to share the same element indices. 
\end{remark}

\subsection{Generally-Convex Scenario}
The following theorem establishes convergence of NEOLITHIC for the generally-convex scenario.  
\begin{theorem}[\sc Generally-convex scenario]\label{thm:upper-gc-contractive}
Given $n \ge 1$, precision $\epsilon>0$, and $G^\star=f^\star-\frac{1}{n}\sum_{i=1}^nf_i^\star$ in which $f^\star$ and $f_i^\star$ are minimum of function $f$ and $f_i$ in problem \eqref{eqn:prob}, the following results hold (proof is in Appendix \ref{app:upper-gconvex}).
\begin{itemize}
    \item \textbf{Contractive compressor:} For any $\{f_i\}_{i=1}^n\subseteq\cF_{L,\mu}^{\Delta_x}$  with $\mu=0$, $\{O_i\}_{i=1}^n \subseteq \cO_\sigma^2$, and $\{C_i\}_{i=1}^n \subseteq \cC_\delta$, if we let $p=5$, $\gamma_k=10/(k+2)$, and set $R$ and $\eta$ as in Appendix \ref{app:upper-convex-contractive}, then NEOLITHIC requires 
    \begin{align}\label{eq-upper-bound-contractive-gc}
        \tilde{\cO}\left(\frac{\Delta_x\sigma^2}{n\epsilon^2} + \frac{1}{\delta}\cdot\left(\frac{L\Delta_x}{\epsilon}\right)^{\frac12}\ln\left(\frac{1}{\epsilon}\right)\right)
    \end{align}
    total number of iterations to achieve an $\epsilon$-approximate optimum, where notation $\tilde{\cO}(\cdot)$ hides all logarithm factors of $n,L,\sigma,G^\star,\Delta_x$ and $\delta$.

    \item \textbf{Unbiased compressor:} For any $\{f_i\}_{i=1}^n\subseteq\cF_{L,\mu}^{\Delta_x}$  with $\mu=0$, $\{O_i\}_{i=1}^n \subseteq \cO_\sigma^2$, and $\{C_i\}_{i=1}^n \subseteq \cU_\omega$, if we let $p=2$, $\gamma_k=6/(k+3)$ and set $R$, $\eta$ as in Appendix \ref{app:upper-convex-unbiased}, then NEOLITHIC requires 
    \begin{align}\label{eq-upper-bound-unbiased-gc}
        \tilde{\cO}\left(\frac{\Delta_x\sigma^2}{n\epsilon^2} + (1+\omega)\left(\frac{L\Delta_x}{\epsilon}\right)^{\frac12}\right)
    \end{align}
    total number of iterations to achieve an $\epsilon$-approximate optimum, where notation $\tilde{\cO}(\cdot)$ hides all logarithm factors of $n,L,\sigma,G^\star,\Delta_x$ and $\omega$.
\end{itemize} 
\end{theorem}
\begin{remark} [\sc total number of iterations]
Theorem \ref{thm:upper-gc-contractive} presents the complexity in terms of the total number of iterations required by the entire optimization process. To illustrate the total number of iterations, we assume the number of iterations within MSC is constant per NEOLITHIC loop. Based on this assumption, we can determine the total number of iterations as follows:
\begin{align}
\text{Total number of iterations} = \text{Number of loops in NEOLITHIC} &\times \text{Number of iterations in MSC}
\end{align}
This count is also equivalent to the total number of sample queries and the total number of communication/compression rounds.
\end{remark} 

\begin{remark} [\sc Tightness of the lower bounds]
Comparing the upper bound \eqref{eq-upper-bound-unbiased-gc} with the lower bound \eqref{eq-lower-bound-unbiased-gc} when utilizing unbiased compressors, we observe that NEOLITHIC nearly attains the lower bound up to logarithmic terms, implying that the lower bound \eqref{eq-lower-bound-unbiased-gc} is nearly tight, and NEOLITHIC is nearly optimal for the generally-convex scenario with unbiased compressors. However, the upper bound \eqref{eq-upper-bound-contractive-gc} is worse than the lower bound by a factor of $\ln(1/\epsilon)$. Although we believe that the presence of the $\ln(1/\epsilon)$ term is attributed to the analysis, we cannot eliminate it using existing techniques including the strategy of multi-stage NEOLITHIC.
\end{remark}

\subsection{Non-Convex Scenario}

\begin{theorem}[\sc Non-convex scenario]\label{thm:upper-nc-contractive}
Given $n \ge 1$, precision $\epsilon>0$, and $G^\star\triangleq f^\star-\frac{1}{n}\sum_{i=1}^nf_i^\star$ in which $f^\star$ and $f_i^\star$ are minimum of function $f$ and $f_i$ in problem \eqref{eqn:prob}, the following results hold (proof is in Appendix \ref{thm:upper-nc}).
\begin{itemize}
    \item \textbf{Contractive compressor:} For any $\{f_i\}_{i=1}^n\subseteq\cF_{L}^{\Delta_f}$, $\{O_i\}_{i=1}^n \subseteq \cO_\sigma^2$, and $\{C_i\}_{i=1}^n \subseteq \cC_\delta$, if we let $\gamma_k\equiv\gamma=p=1$, and set $R$ and $\eta$ as in Appendix \ref{thm:upper-nc}, then NEOLITHIC requires 
    \begin{align}\label{eq-upper-bound-contractive-nc}
        \tilde{\cO}\left(\frac{\Delta_fL\sigma^2}{n\epsilon^2} + \frac{\Delta_fL}{\delta\epsilon}\ln\left(\frac{1}{\epsilon}\right)\right)
    \end{align}
    total number of iterations to achieve an $\epsilon$-approximate stationary point, where notation $\tilde{\cO}(\cdot)$ hides all logarithm factors of $n,L,\sigma,G^\star ,\Delta_f$ and $\delta$.

    \item \textbf{Unbiased compressor:} For any $\{f_i\}_{i=1}^n\subseteq\cF_{L}^{\Delta_f}$, $\{O_i\}_{i=1}^n \subseteq \cO_\sigma^2$, and $\{C_i\}_{i=1}^n \subseteq \cU_\omega$, if we let $\gamma_k\equiv\gamma=p=1$, and set $R$ and $\eta$ as in Appendix \ref{thm:upper-nc}, then NEOLITHIC requires 
    \begin{align}\label{eq-upper-bound-unbiased-nc}
        \tilde{\cO}\left(\frac{\Delta_fL\sigma^2}{n\epsilon^2} + \frac{(1+\omega)\Delta_fL}{\epsilon}\right)
    \end{align}
    total number of iterations to achieve an $\epsilon$-approximate stationary point, where notation $\tilde{\cO}(\cdot)$ hides all logarithm factors of $n,L,\sigma,G^\star,\Delta_f$ and $\omega$.
\end{itemize} 
\end{theorem}

\begin{remark}[\sc Tightness of the lower bounds]

In contrast to the preliminary results in the conference paper \cite{Huang2022LowerBA}, where the upper bounds are derived under additional data heterogeneity assumptions beyond those required for the lower bounds, Theorem \ref{thm:upper-nc-contractive} removes the data heterogeneity assumption and establishes  upper bounds under assumptions that closely align with those used to construct the lower bounds in Theorems \ref{thm:lower-bounds} and \ref{thm:lower-bounds-contractive}. As a result, given the closeness between the upper and lower bounds, we conclude that the lower bound \eqref{eq-lower-bound-unbiased-nc} is nearly tight, and NEOLITHIC is nearly optimal for the non-convex setting with unbiased compressors. Similar to the generally-convex case, the upper bound \eqref{eq-upper-bound-contractive-nc} is worse than the lower bound by a factor of $\ln(1/\epsilon)$; however, it is derived under milder assumptions than most prior works \cite{Zhao2019GlobalMC,Karimireddy2019ErrorFF,Stich2018SparsifiedSW,Xie2020CSERCS,Tang2019DoubleSqueezePS,Basu2020QsparseLocalSGDDS} such as bounded gradients, and it also achieves a tighter convergence rate than EF21-SGD \cite{Fatkhullin2021EF21WB}.

\end{remark}

\section{Experiments}
In this section, we conduct simulations to evaluate the efficiency of NEOLITHIC algorithms across various tasks. Our primary focus is on comparing NEOLITHIC with other algorithms such as QSGD, MEM-SGD, Double-Squeeze, and EF21-SGD.

\subsection{Strongly-Convex Scenario: Distributed Least Square}\label{subsec:syn}
Consider a strongly-convex distributed least-square problem in which each worker $i$ holds a local function 
$$
f_i(x)=\frac{1}{2n}\|A_ix-b_i\|^2, \quad \forall\, i \in \{1,\cdots, n\}.
$$
Let $M$ denote the size of each local dataset, and $d$ denote the dimension of the model, represented by $x$. In this experiment, we set $n=30$, $M=100$, and $d=10$. To generate the matrices $A_i$, we first generate a random matrix $A\in\RR^{nM\times d}$ by independently sampling each entry from a normal distribution. Next, we modify its condition number through SVD decomposition to justify the theoretical advantages in the scaling of $\kappa=L/\mu$, and partition it into $A_1,\cdots,A_n$. For each node, we generate an optimal solution $x_i^\star$ based on a randomly generated reference solution $x_0^\star$ and a randomly generated noise term, as follows: $x_i^\star=x_0^\star+e_i$, where $x_0^\star\sim\cN(0,I_d)$ and $e_i\sim\cN(0,0.01I_d)$. The $b_i$'s are generated such that $b_i-A_ix_i^\star\sim\cN(0,0.01I_d)$.
We consider two settings of gradient oracles. In the small noise setting, the gradient oracle $O_i$ at $x$ returns $\nabla f_i(x)+0.001\varepsilon$, where $\varepsilon\sim\cN(0,I_d)$ is independently drawn for all queries. In the large noise setting, we use $\nabla f_i(x)+0.1\varepsilon$ instead. 

\renewcommand{\dblfloatpagefraction}{.9}
\begin{figure}[t]
\centering
\begin{subfigure}{.4\textwidth}
\includegraphics[width=\textwidth]{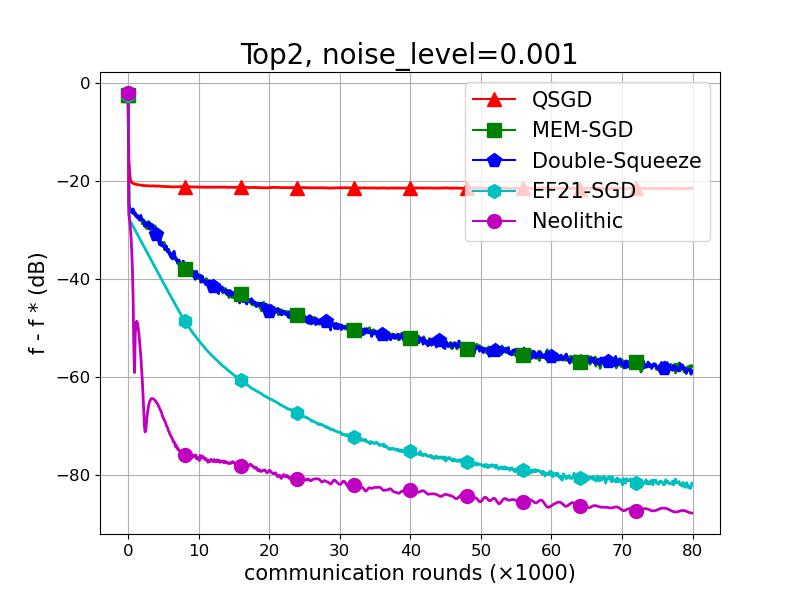}
\end{subfigure}
\begin{subfigure}{.4\textwidth}
\includegraphics[width=\textwidth]{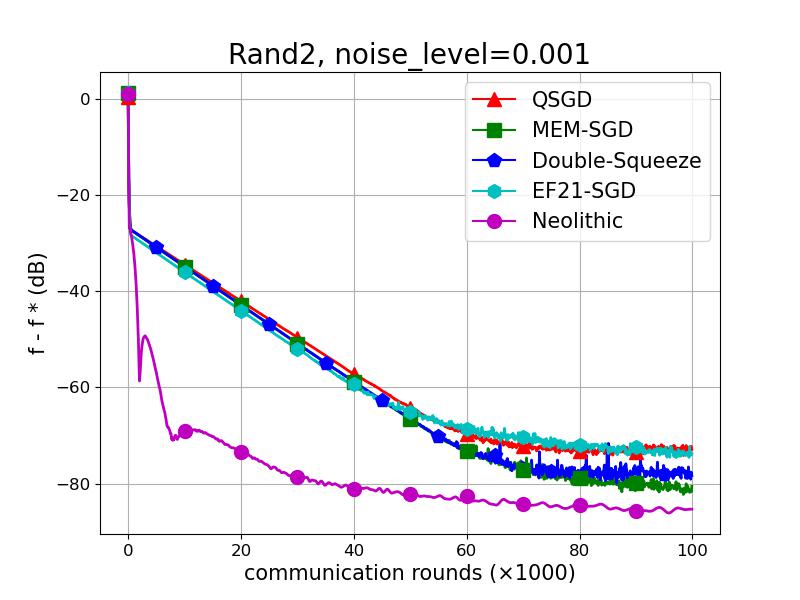}
\end{subfigure}
\begin{subfigure}{.4\textwidth}
\includegraphics[width=\textwidth]{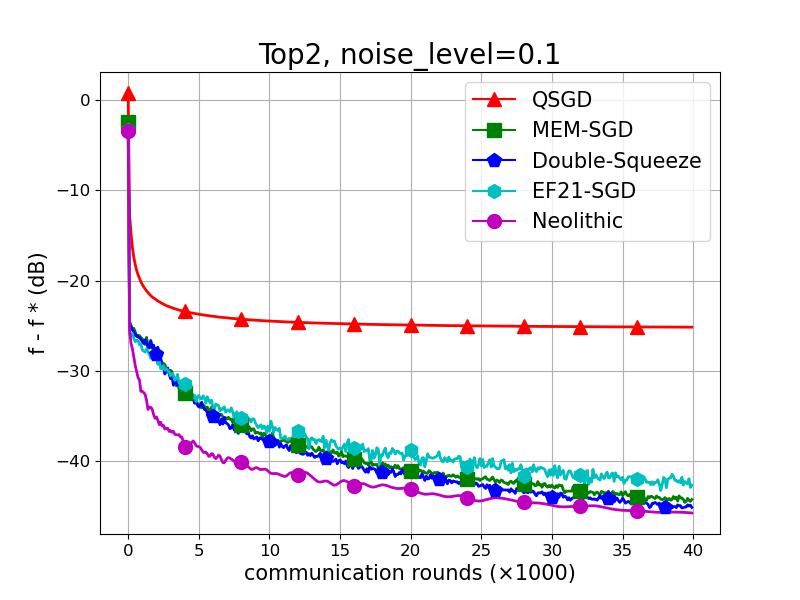}
\end{subfigure}
\begin{subfigure}{.4\textwidth}
\includegraphics[width=\textwidth]{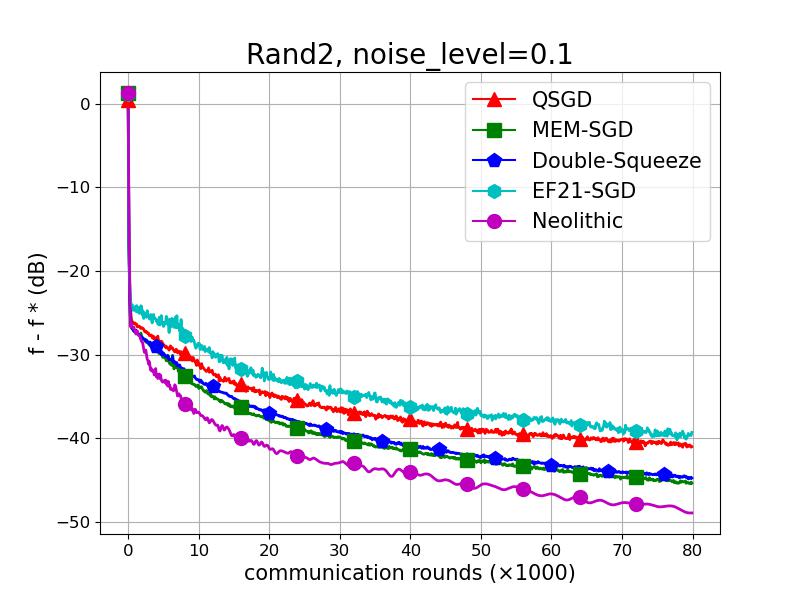}
\end{subfigure}
\caption{\small Convergence results of various algorithms on distributed least square problem. The $y$-axis represents $f-f^\star$ (dB) and the $x$-axis indicates the total communication rounds (in units of 
 thousands). All curves are averaged over 20 trials.} 
\label{fig:ls}
\end{figure}

We evaluate the contractive compressors by using the Top-2 compressor, while for the simulations with unbiased compressors, we employ the uRand-2 (unbiased Rand-2 via post scaling) compressor. The learning rates for all algorithms are set using the formula $\min\{c_1/L,\eta_0\cdot(\mathrm{iters}+1)^{-c_2}\}$, where constants $c_1$ and $c_2$ are tuned to best fit in each algorithm. We implement single-stage NEOLITHIC algorithm with $R=5$ and parameters $p,\eta,\gamma$ tuned to best fit the algorithm where $p$ is a constant, $\eta$ has formation $\min\{c_1/L,\eta_0\cdot(\mathrm{iters}+1)^{-c_2}\}$ and $\gamma=\gamma_0\cdot(\mathrm{iters}+1)^{-c_2}$. 
When we embed the uRand-2 compressors to MEM-SGD, Double-Squeeze and EF21-SGD, we observe a divergence in the earliest iterations, even when scaling down the learning rate by $10^{10}$. Consequently, we used the Rand-2 compressor without post scaling instead. Nonetheless, we continue to use the uRand-2 in Q-SGD and NEOLITHIC. To ensure fair comparison, we compressed only the messages from the workers to the server in Double-Squeeze.

\begin{figure}[h]
\centering
\begin{subfigure}{.4\textwidth}
\includegraphics[width=\textwidth]{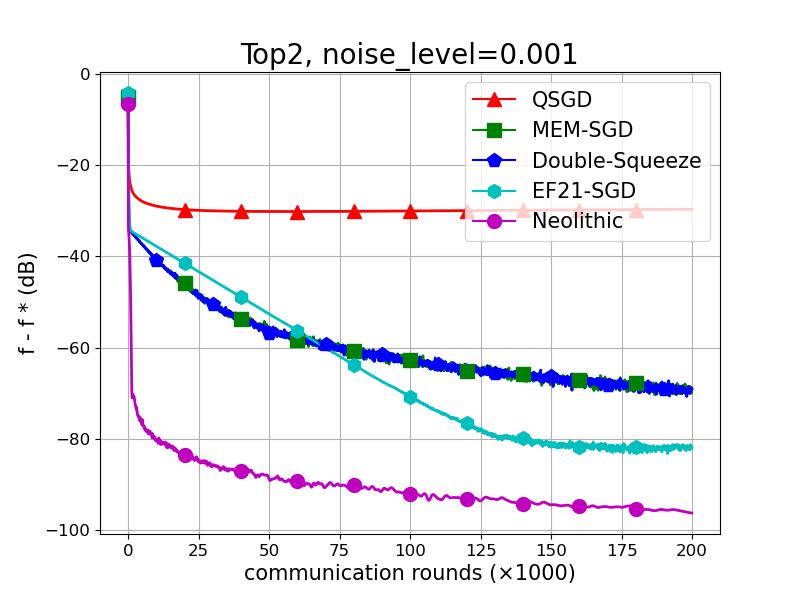}
\end{subfigure}
\begin{subfigure}{.4\textwidth}
\includegraphics[width=\textwidth]{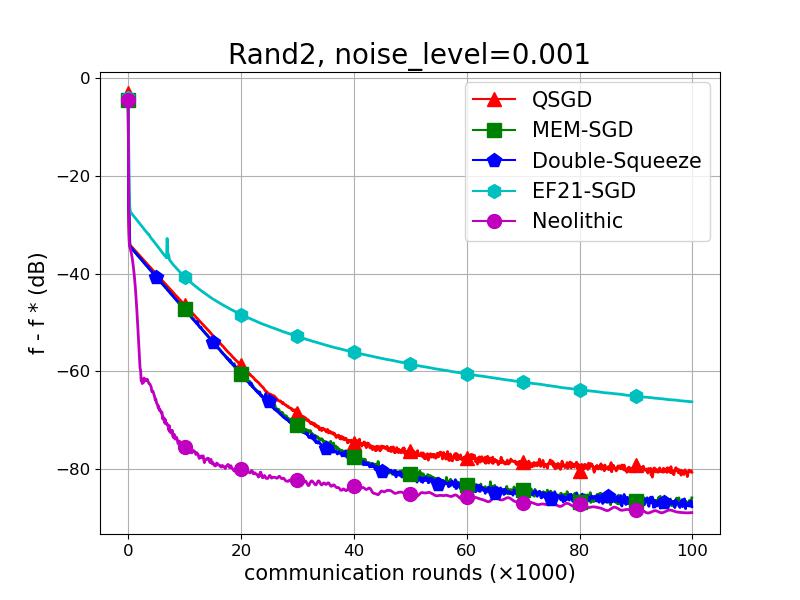}
\end{subfigure}
\begin{subfigure}{.4\textwidth}
\includegraphics[width=\textwidth]{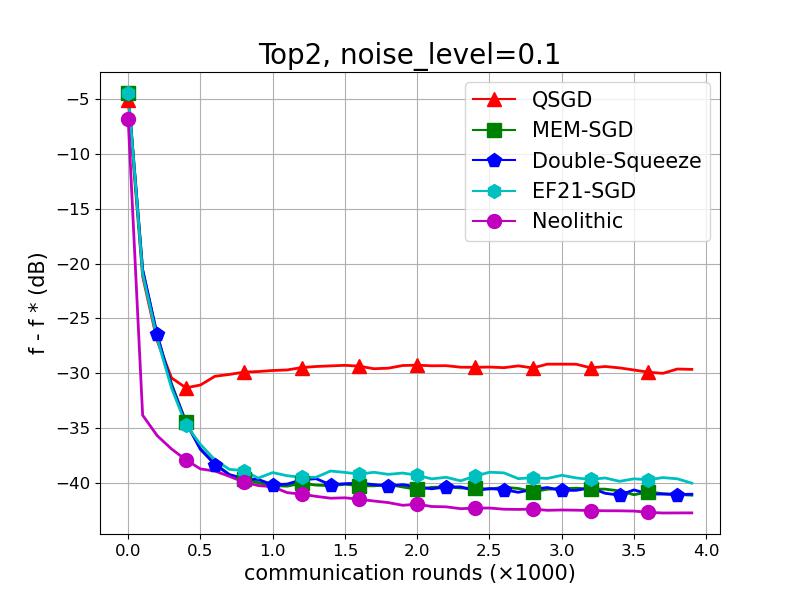}
\end{subfigure}
\begin{subfigure}{.4\textwidth}
\includegraphics[width=\textwidth]{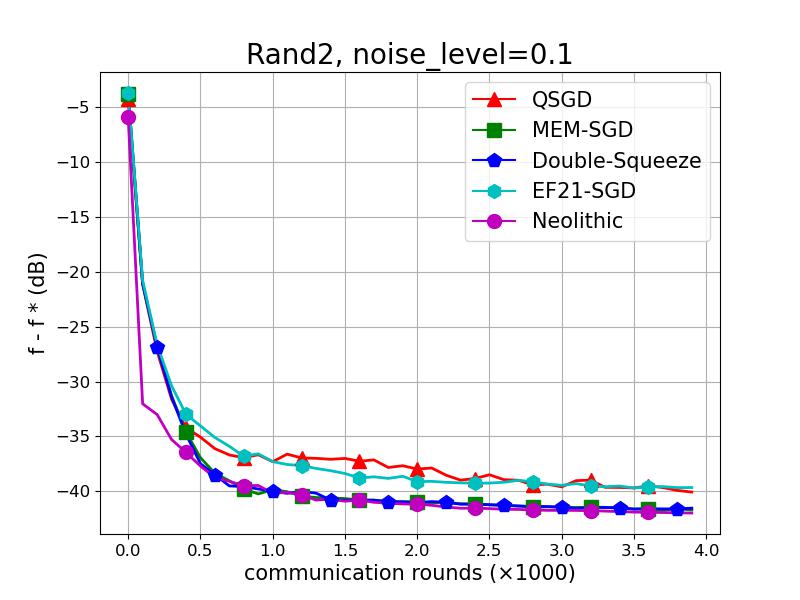}
\end{subfigure}
\caption{\small Convergence results of various algorithms on distributed logistic regression problem. The $y$-axis represents $f-f^\star$ (dB) and the $x$-axis indicates the total communication rounds (in units of thousands). All curves are averaged over 20 trials.}
\label{fig:lr}
\end{figure}

In Figure \ref{fig:ls}, we present the performance of all algorithms. When operating in the small noise setting, the convergence of each algorithm is dominated by the term involving the condition number of the objective function $\kappa = L/\mu$. In the top two plots, NEOLITHIC presents a clear superiority over other baselines due to its improved complexity in the scaling of $\kappa$  thanks to the acceleration mechanism. However, in the large noise setting, the term involving the variance of the stochastic gradient $\sigma^2$ governs the performance of each algorithm, and all algorithms perform closely. However, NEOLITHIC still outperforms baselines by a visible margin.

\subsection{Generally-Convex Scenario: Distributed Logistic Regression }\label{subsec:lr}
Consider the generally-convex distributed logistic regression problem in which each worker $i$ holds a function
$$
f_i(x)=\frac{1}{M}\sum_{i=1}^M\ln(1+\exp(-b_{i,m}a_{i,m}^\top x)), \quad \forall i\in \{1,\cdots, n\}. 
$$
where $A_i=(a_{i,1}^\top,\cdots,a_{i,M}^\top)^\top$, $x_i^\star$ are generated similarly to the least square experiments, with the same values of $n$, $M$ and $d$. The label $b_{i,m}$ is generated independently through Bernoulli distributions with a probability $\PP(b_{i,m}=1)=(1+\exp(-a_{i,m}^\top x_i^\star))^{-1}$. We set the gradient oracles, algorithms, compressors and hyperparameters similarly to those in the least square experiments.

Figure \ref{fig:lr} displays the performance of all algorithms. In the small noise regime, smoothness of the objective function dominates the convergence of each algorithm. NEOLITHIC exhibits a significant advantage over other baselines in this scenario, with an improved complexity of $\cO((1+\omega)\sqrt{L/\epsilon})$ or $\cO(\delta^{-1}\sqrt{L/\epsilon}\ln(1/\epsilon))$, as shown in the top two plots. In the large noise setting, stochastic gradient variance $\sigma^2$ largely determines the performance of each algorithm, resulting in similar performances across all algorithms. Nevertheless, NEOLITHIC outperforms the baselines by a mild margin.

\subsection{Influence of MSC round $R$}
We investigate the impact of MSC rounds on NEOLITHIC by varying $R$ in an experiment with fixed hyperparameters. We consider the least square problem with $n=30$, $M=60$, and $d=50$. We use fixed values of $\eta$, $\gamma$, $p$, and compressors from the set $\{\mbox{Top-2, Rand-2, uRand-2}\}$ while varying $R$ from 1 to 50. We perform the experiment under varying settings of data heterogeneity and noise scales. In the  case of small heterogeneity, we use $e_i\sim\mathcal{N}(0,0.01I_d)$, while in the case of big heterogeneity, we use $e_i\sim\cN(0,I_d)$. The noise scale is identical to the previous experiment. The results in Figure \ref{fig:R} suggest that a moderate value of $R$ is necessary to balance compression error, while larger values of $R$ may cause performance degradation due to over-accumulation in gradient queries. This phenomenon is more evident in cases with big heterogeneity, where local gradients are less informative than global gradients, and in cases with small noise, where a small minibatch is sufficient to reduce the gradient noise. Therefore, selecting the appropriate value of $R$ involves a trade-off between reducing the compression error and saving gradient computation.


\begin{figure}[t]
\centering
\begin{subfigure}{.4\textwidth}
\includegraphics[width=\textwidth]{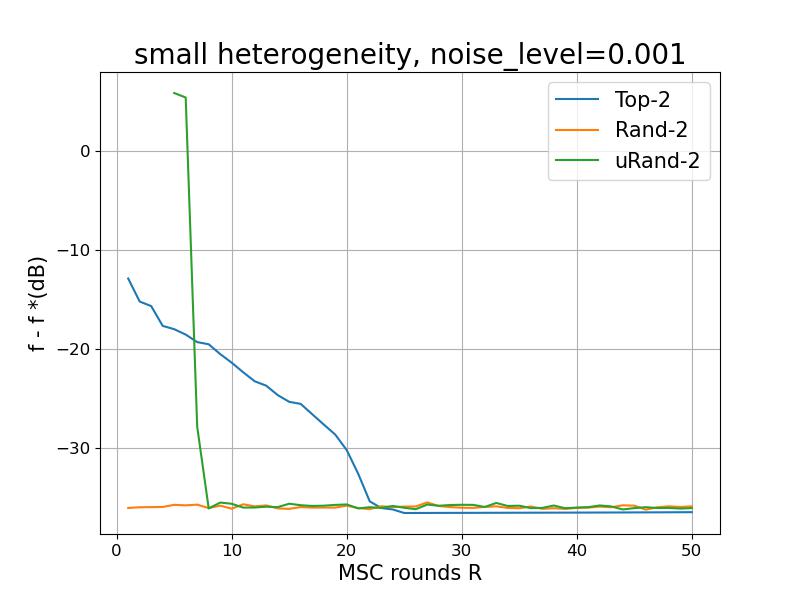}
\end{subfigure}
\begin{subfigure}{.4\textwidth}
\includegraphics[width=\textwidth]{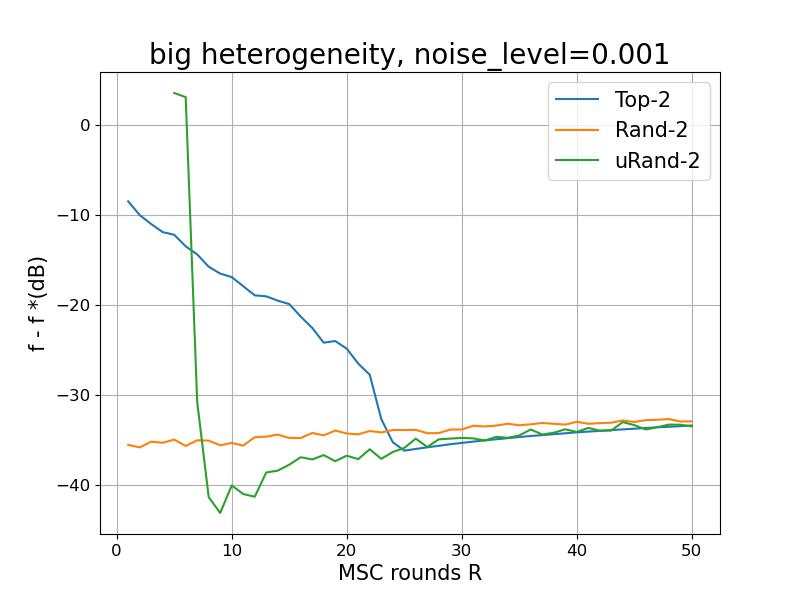}
\end{subfigure}
\begin{subfigure}{.4\textwidth}
\includegraphics[width=\textwidth]{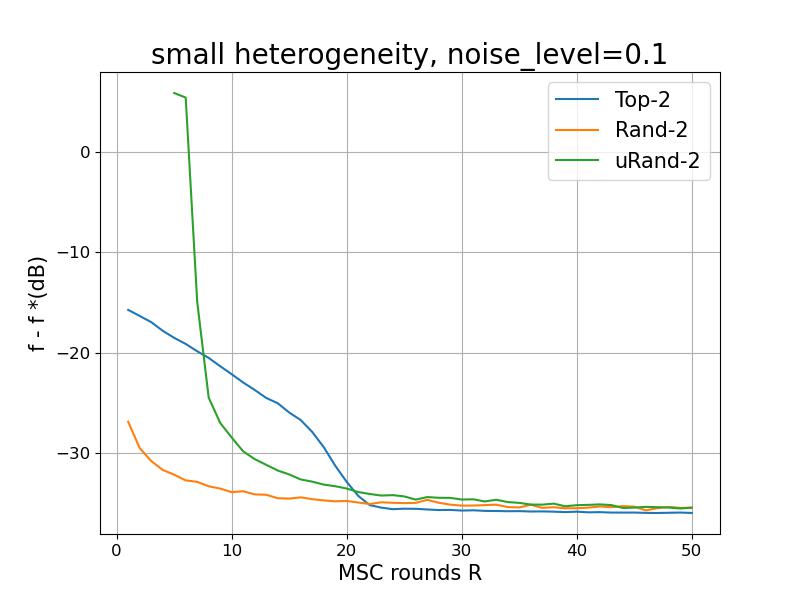}
\end{subfigure}
\begin{subfigure}{.4\textwidth}
\includegraphics[width=\textwidth]{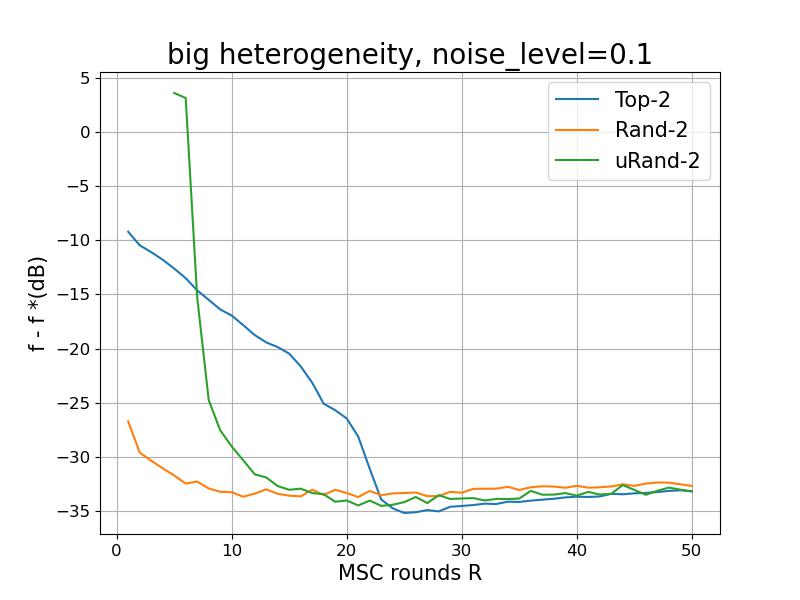}
\end{subfigure}
\caption{\small Best precision over $10,000$ total communication rounds for different MSC rounds and compressors, under varying data heterogeneity and gradient noise scales. All curves are averaged over 20 trails.}
\label{fig:R}
\end{figure}

\section{Conclusion}
This paper focuses on distributed stochastic algorithms for minimizing smooth objective functions under two representative compressors: unbiased and contractive. We establish convergence lower bounds for algorithms using these compressors in strongly-convex, generally-convex, and non-convex scenarios. To bridge the gap between the established lower bounds and existing upper bounds, we introduce NEOLITHIC, an algorithm that almost attains the lower bounds (up to logarithmic factors) under mild conditions. We employ a novel multi-step compression module to achieve state-of-the-art convergence rates. Extensive experimental results support our findings.

Despite the promising results presented in this paper, there are still several open questions yet to be answered. One important question is how to achieve the lower bound by removing the $\ln(1/\epsilon)$ term from the complexity upper bounds of NEOLITHIC in the generally-convex and non-convex scenario with contractive compressors. 
Further research in these directions is needed to fully understand the performance limits of compression algorithms in distributed stochastic optimization.


\bibliography{main}

\begin{thebibliography}{10}

\bibitem{Agarwal2010InformationTheoreticLB}
A.~Agarwal, P.~L. Bartlett, P.~Ravikumar, and M.~J. Wainwright.
\newblock Information-theoretic lower bounds on the oracle complexity of
  stochastic convex optimization.
\newblock {\em IEEE Transactions on Information Theory}, 58:3235--3249, 2010.

\bibitem{Agarwal2015ALB}
A.~Agarwal and L.~Bottou.
\newblock A lower bound for the optimization of finite sums.
\newblock In {\em International Conference on Machine Learning}, 2015.

\bibitem{alghunaim2021unified}
S.~A. Alghunaim and K.~Yuan.
\newblock A unified and refined convergence analysis for non-convex
  decentralized learning.
\newblock {\em arXiv preprint arXiv:2110.09993}, 2021.

\bibitem{Alistarh2017QSGDCS}
D.~Alistarh, D.~Grubic, J.~Li, R.~Tomioka, and M.~Vojnovic.
\newblock {QSGD}: Communication-efficient {SGD} via gradient quantization and
  encoding.
\newblock In {\em Advances in Neural Information Processing Systems}, 2017.

\bibitem{Alistarh2018TheCO}
D.~Alistarh, T.~Hoefler, M.~Johansson, S.~Khirirat, N.~Konstantinov, and
  C.~Renggli.
\newblock The convergence of sparsified gradient methods.
\newblock In {\em Advances in Neural Information Processing Systems}, 2018.

\bibitem{AllenZhu2018HowTM}
Z.~Allen-Zhu.
\newblock How to make the gradients small stochastically: Even faster convex
  and nonconvex {SGD}.
\newblock In {\em Advances in Neural Information Processing Systems}, 2018.

\bibitem{Arjevani2019LowerBF}
Y.~Arjevani, Y.~Carmon, J.~C. Duchi, D.~J. Foster, N.~Srebro, and B.~E.
  Woodworth.
\newblock Lower bounds for non-convex stochastic optimization.
\newblock {\em ArXiv}, abs/1912.02365, 2019.

\bibitem{Arjevani2015CommunicationCO}
Y.~Arjevani and O.~Shamir.
\newblock Communication complexity of distributed convex learning and
  optimization.
\newblock In {\em Advances in Neural Information Processing Systems}, 2015.

\bibitem{Balkanski2018ParallelizationDN}
E.~Balkanski and Y.~Singer.
\newblock Parallelization does not accelerate convex optimization: Adaptivity
  lower bounds for non-smooth convex minimization.
\newblock {\em ArXiv}, abs/1808.03880, 2018.

\bibitem{Basu2020QsparseLocalSGDDS}
D.~Basu, D.~Data, C.~B. Karakus, and S.~N. Diggavi.
\newblock {Qsparse-Local-SGD}: Distributed {SGD} with quantization,
  sparsification, and local computations.
\newblock {\em IEEE Journal on Selected Areas in Information Theory},
  1:217--226, 2020.

\bibitem{Bernstein2018signSGDCO}
J.~Bernstein, Y.-X. Wang, K.~Azizzadenesheli, and A.~Anandkumar.
\newblock {SignSGD}: compressed optimisation for non-convex problems.
\newblock In {\em International Conference on Machine Learning}, 2018.

\bibitem{Beznosikov2020OnBC}
A.~Beznosikov, S.~Horvath, P.~Richt{\'a}rik, and M.~H. Safaryan.
\newblock On biased compression for distributed learning.
\newblock {\em ArXiv}, abs/2002.12410, 2020.

\bibitem{carmon2020lower}
Y.~Carmon, J.~C. Duchi, O.~Hinder, and A.~Sidford.
\newblock Lower bounds for finding stationary points i.
\newblock {\em Mathematical Programming}, 184(1):71--120, 2020.

\bibitem{Carmon2021LowerBF}
Y.~Carmon, J.~C. Duchi, O.~Hinder, and A.~Sidford.
\newblock Lower bounds for finding stationary points ii: first-order methods.
\newblock {\em Mathematical Programming}, 185:315--355, 2021.

\bibitem{chen2012diffusion}
J.~Chen and A.~H. Sayed.
\newblock Diffusion adaptation strategies for distributed optimization and
  learning over networks.
\newblock {\em IEEE Transactions on Signal Processing}, 60(8):4289--4305, 2012.

\bibitem{chen2018lag}
T.~Chen, G.~Giannakis, T.~Sun, and W.~Yin.
\newblock {LAG}: Lazily aggregated gradient for communication-efficient
  distributed learning.
\newblock In {\em Advances in Neural Information Processing Systems}, pages
  5050--5060, 2018.

\bibitem{Diakonikolas2019LowerBF}
J.~Diakonikolas and C.~Guzm{\'a}n.
\newblock Lower bounds for parallel and randomized convex optimization.
\newblock In {\em Conference on Learning Theory}, 2019.

\bibitem{duchi2016introductory}
J.~C. Duchi.
\newblock {\em Introductory Lectures on Stochastic Convex Optimization},
  volume~87.
\newblock Park City Mathematics Institute, Graduate Summer School Lectures,
  2016.

\bibitem{Fatkhullin2021EF21WB}
I.~Fatkhullin, I.~Sokolov, E.~A. Gorbunov, Z.~Li, and P.~Richt{\'a}rik.
\newblock Ef21 with bells \& whistles: Practical algorithmic extensions of
  modern error feedback.
\newblock {\em ArXiv}, abs/2110.03294, 2021.

\bibitem{Foster2019TheCO}
D.~J. Foster, A.~Sekhari, O.~Shamir, N.~Srebro, K.~Sridharan, and B.~E.
  Woodworth.
\newblock The complexity of making the gradient small in stochastic convex
  optimization.
\newblock In {\em Conference on Learning Theory}, 2019.

\bibitem{ghadimi2013optimal}
S.~Ghadimi and G.~Lan.
\newblock Optimal stochastic approximation algorithms for strongly convex
  stochastic composite optimization, ii: shrinking procedures and optimal
  algorithms.
\newblock {\em SIAM Journal on Optimization}, 23(4):2061--2089, 2013.

\bibitem{Gorbunov2021MARINAFN}
E.~A. Gorbunov, K.~Burlachenko, Z.~Li, and P.~Richt{\'a}rik.
\newblock Marina: Faster non-convex distributed learning with compression.
\newblock {\em ArXiv}, 2021.

\bibitem{haddadpour2021federated}
F.~Haddadpour, M.~M. Kamani, A.~Mokhtari, and M.~Mahdavi.
\newblock Federated learning with compression: Unified analysis and sharp
  guarantees.
\newblock In {\em International Conference on Artificial Intelligence and
  Statistics}, pages 2350--2358. PMLR, 2021.

\bibitem{Horvath2019NaturalCF}
S.~Horvath, C.-Y. Ho, L.~Horvath, A.~N. Sahu, M.~Canini, and P.~Richt{\'a}rik.
\newblock Natural compression for distributed deep learning.
\newblock {\em ArXiv}, abs/1905.10988, 2019.

\bibitem{Horvath2019StochasticDL}
S.~Horvath, D.~Kovalev, K.~Mishchenko, S.~U. Stich, and P.~Richt{\'a}rik.
\newblock Stochastic distributed learning with gradient quantization and
  variance reduction.
\newblock {\em arXiv: Optimization and Control}, 2019.

\bibitem{Huang2022LowerBA}
X.~Huang, Y.~Chen, W.~Yin, and K.~Yuan.
\newblock Lower bounds and nearly optimal algorithms in distributed learning
  with communication compression.
\newblock In {\em Advances in Neural Information Processing Systems (NeurIPS)},
  2022.

\bibitem{huang2022optimal}
X.~Huang and K.~Yuan.
\newblock Optimal complexity in non-convex decentralized learning over
  time-varying networks.
\newblock {\em arXiv preprint arXiv:2211.00533}, 2022.

\bibitem{Huang2021Improved}
X.~Huang, K.~Yuan, X.~Mao, and W.~Yin.
\newblock Improved analysis and rates for variance reduction under
  without-replacement sampling orders.
\newblock In {\em Advances in Neural Information Processing Systems (NeurIPS)},
  2021.

\bibitem{Jiang2018ALS}
P.~Jiang and G.~Agrawal.
\newblock A linear speedup analysis of distributed deep learning with sparse
  and quantized communication.
\newblock In {\em Advances in Neural Information Processing Systems}, 2018.

\bibitem{karimireddy2020scaffold}
S.~P. Karimireddy, S.~Kale, M.~Mohri, S.~Reddi, S.~Stich, and A.~T. Suresh.
\newblock Scaffold: Stochastic controlled averaging for federated learning.
\newblock In {\em International Conference on Machine Learning}, 2020.

\bibitem{Karimireddy2019ErrorFF}
S.~P. Karimireddy, Q.~Rebjock, S.~U. Stich, and M.~Jaggi.
\newblock Error feedback fixes {SignSGD} and other gradient compression
  schemes.
\newblock In {\em International Conference on Machine Learning}, 2019.

\bibitem{Kingma2015AdamAM}
D.~P. Kingma and J.~Ba.
\newblock Adam: A method for stochastic optimization.
\newblock {\em CoRR}, abs/1412.6980, 2015.

\bibitem{koloskova2021improved}
A.~Koloskova, T.~Lin, and S.~U. Stich.
\newblock An improved analysis of gradient tracking for decentralized machine
  learning.
\newblock In {\em Advances in Neural Information Processing Systems}, 2021.

\bibitem{koloskova2020unified}
A.~Koloskova, N.~Loizou, S.~Boreiri, M.~Jaggi, and S.~Stich.
\newblock A unified theory of decentralized {SGD} with changing topology and
  local updates.
\newblock In {\em International Conference on Machine Learning}, pages
  5381--5393. PMLR, 2020.

\bibitem{li2014communication}
M.~Li, D.~G. Andersen, A.~J. Smola, and K.~Yu.
\newblock Communication efficient distributed machine learning with the
  parameter server.
\newblock {\em Advances in Neural Information Processing Systems}, 27, 2014.

\bibitem{li2020acceleration}
Z.~Li, D.~Kovalev, X.~Qian, and P.~Richt{\'a}rik.
\newblock Acceleration for compressed gradient descent in distributed and
  federated optimization.
\newblock {\em arXiv preprint arXiv:2002.11364}, 2020.

\bibitem{li2021canita}
Z.~Li and P.~Richt{\'a}rik.
\newblock Canita: Faster rates for distributed convex optimization with
  communication compression.
\newblock {\em Advances in Neural Information Processing Systems},
  34:13770--13781, 2021.

\bibitem{Lian2017CanDA}
X.~Lian, C.~Zhang, H.~Zhang, C.-J. Hsieh, W.~Zhang, and J.~Liu.
\newblock Can decentralized algorithms outperform centralized algorithms? a
  case study for decentralized parallel stochastic gradient descent.
\newblock In {\em Advances in Neural Information Processing Systems}, 2017.

\bibitem{lin2021quasi}
T.~Lin, S.~P. Karimireddy, S.~U. Stich, and M.~Jaggi.
\newblock Quasi-global momentum: Accelerating decentralized deep learning on
  heterogeneous data.
\newblock In {\em International Conference on Machine Learning}, 2021.

\bibitem{Liu2021LinearCD}
X.~Liu, Y.~Li, R.~Wang, J.~Tang, and M.~Yan.
\newblock Linear convergent decentralized optimization with compression.
\newblock In {\em International Conference on Learning Representations}, 2021.

\bibitem{liu2019communication}
Y.~Liu, W.~Xu, G.~Wu, Z.~Tian, and Q.~Ling.
\newblock Communication-censored admm for decentralized consensus optimization.
\newblock {\em IEEE Transactions on Signal Processing}, 67(10):2565--2579,
  2019.

\bibitem{lu2020moniqua}
Y.~Lu and C.~De~Sa.
\newblock Moniqua: Modulo quantized communication in decentralized {SGD}.
\newblock In {\em International Conference on Machine Learning}, pages
  6415--6425. PMLR, 2020.

\bibitem{Lu2021OptimalCI}
Y.~Lu and C.~D. Sa.
\newblock Optimal complexity in decentralized training.
\newblock In {\em International Conference on Machine Learning}, 2021.

\bibitem{mcmahan2017communication}
B.~McMahan, E.~Moore, D.~Ramage, S.~Hampson, and B.~A. y~Arcas.
\newblock Communication-efficient learning of deep networks from decentralized
  data.
\newblock In {\em Artificial Intelligence and Statistics}, 2017.

\bibitem{Mishchenko2019DistributedLW}
K.~Mishchenko, E.~A. Gorbunov, M.~Tak{\'a}c, and P.~Richt{\'a}rik.
\newblock Distributed learning with compressed gradient differences.
\newblock {\em ArXiv}, abs/1901.09269, 2019.

\bibitem{mishchenko2022proxskip}
K.~Mishchenko, G.~Malinovsky, S.~Stich, and P.~Richt{\'a}rik.
\newblock Proxskip: Yes! local gradient steps provably lead to communication
  acceleration! finally!
\newblock In {\em International Conference on Machine Learning}, pages
  15750--15769. PMLR, 2022.

\bibitem{nedic2009distributed}
A.~Nedic and A.~Ozdaglar.
\newblock Distributed subgradient methods for multi-agent optimization.
\newblock {\em IEEE Transactions on Automatic Control}, 54(1):48--61, 2009.

\bibitem{Nesterov1983AMF}
Y.~Nesterov.
\newblock A method for unconstrained convex minimization problem with the rate
  of convergence o(1/k2).
\newblock In {\em Doklady an ussr}, volume~29, 1983.

\bibitem{nesterov2003introductory}
Y.~Nesterov.
\newblock {\em Introductory lectures on convex optimization: A basic course},
  volume~87.
\newblock Springer Science \& Business Media, 2003.

\bibitem{nesterov1983method}
Y.~E. Nesterov.
\newblock A method of solving a convex programming problem with convergence
  rate $o(1/k^2)$.
\newblock {\em Doklady Akademii Nauk}, 269(3):543--547, 1983.

\bibitem{patarasuk2009bandwidth}
P.~Patarasuk and X.~Yuan.
\newblock Bandwidth optimal all-reduce algorithms for clusters of workstations.
\newblock {\em Journal of Parallel and Distributed Computing}, 69(2):117--124,
  2009.

\bibitem{Philippenko2021PreservedCM}
C.~Philippenko and A.~Dieuleveut.
\newblock Preserved central model for faster bidirectional compression in
  distributed settings.
\newblock In {\em Advances in Neural Information Processing Systems}, 2021.

\bibitem{Philippenko2020BidirectionalCI}
C.~Philippenko and A.~Dieuleveut.
\newblock Bidirectional compression in heterogeneous settings for distributed
  or federated learning with partial participation: tight convergence
  guarantees.
\newblock {\em ArXiv}, 2022.

\bibitem{pu2021distributed}
S.~Pu and A.~Nedi{\'c}.
\newblock Distributed stochastic gradient tracking methods.
\newblock {\em Mathematical Programming}, 187:409--457, 2021.

\bibitem{qian2021error}
X.~Qian, P.~Richt{\'a}rik, and T.~Zhang.
\newblock Error compensated distributed sgd can be accelerated.
\newblock {\em Advances in Neural Information Processing Systems},
  34:30401--30413, 2021.

\bibitem{Richtrik2021EF21AN}
P.~Richt{\'a}rik, I.~Sokolov, and I.~Fatkhullin.
\newblock Ef21: A new, simpler, theoretically better, and practically faster
  error feedback.
\newblock {\em ArXiv}, abs/2106.05203, 2021.

\bibitem{Richtarik20223PCTP}
P.~Richt'arik, I.~Sokolov, I.~Fatkhullin, E.~Gasanov, Z.~Li, and E.~A.
  Gorbunov.
\newblock 3pc: Three point compressors for communication-efficient distributed
  training and a better theory for lazy aggregation.
\newblock In {\em International Conference on Machine Learning}, 2022.

\bibitem{Safaryan2020UncertaintyPF}
M.~H. Safaryan, E.~Shulgin, and P.~Richt{\'a}rik.
\newblock Uncertainty principle for communication compression in distributed
  and federated learning and the search for an optimal compressor.
\newblock {\em Information and Inference: A Journal of the IMA}, 2021.

\bibitem{Seide20141bitSG}
F.~Seide, H.~Fu, J.~Droppo, G.~Li, and D.~Yu.
\newblock 1-bit stochastic gradient descent and its application to
  data-parallel distributed training of speech dnns.
\newblock In {\em INTERSPEECH}, 2014.

\bibitem{stich2019local}
S.~U. Stich.
\newblock Local {SGD} converges fast and communicates little.
\newblock In {\em International Conference on Learning Representations}, 2019.

\bibitem{Stich2018SparsifiedSW}
S.~U. Stich, J.-B. Cordonnier, and M.~Jaggi.
\newblock Sparsified {SGD} with memory.
\newblock In {\em Advances in Neural Information Processing Systems}, 2018.

\bibitem{tang2018d}
H.~Tang, X.~Lian, M.~Yan, C.~Zhang, and J.~Liu.
\newblock $ d^2$: Decentralized training over decentralized data.
\newblock In {\em International Conference on Machine Learning}, pages
  4848--4856, 2018.

\bibitem{Tang2019DoubleSqueezePS}
H.~Tang, X.~Lian, T.~Zhang, and J.~Liu.
\newblock Doublesqueeze: Parallel stochastic gradient descent with double-pass
  error-compensated compression.
\newblock {\em ArXiv}, abs/1905.05957, 2019.

\bibitem{Tyurin2022DASHADN}
A.~Tyurin and P.~Richt'arik.
\newblock Dasha: Distributed nonconvex optimization with communication
  compression, optimal oracle complexity, and no client synchronization.
\newblock {\em ArXiv}, 2022.

\bibitem{Wangni2018GradientSF}
J.~Wangni, J.~Wang, J.~Liu, and T.~Zhang.
\newblock Gradient sparsification for communication-efficient distributed
  optimization.
\newblock In {\em Advances in Neural Information Processing Systems}, 2018.

\bibitem{Wen2017TernGradTG}
W.~Wen, C.~Xu, F.~Yan, C.~Wu, Y.~Wang, Y.~Chen, and H.~H. Li.
\newblock Terngrad: Ternary gradients to reduce communication in distributed
  deep learning.
\newblock In {\em Advances in Neural Information Processing Systems}, 2017.

\bibitem{Wu2018ErrorCQ}
J.~Wu, W.~Huang, J.~Huang, and T.~Zhang.
\newblock Error compensated quantized {SGD} and its applications to large-scale
  distributed optimization.
\newblock In {\em International Conference on Machine Learning}, 2018.

\bibitem{Xie2020CSERCS}
C.~Xie, S.~Zheng, O.~Koyejo, I.~Gupta, M.~Li, and H.~Lin.
\newblock Cser: Communication-efficient {SGD} with error reset.
\newblock In {\em Advances in Neural Information Processing Systems}, 2020.

\bibitem{xin2020improved}
R.~Xin, U.~A. Khan, and S.~Kar.
\newblock An improved convergence analysis for decentralized online stochastic
  non-convex optimization.
\newblock {\em IEEE Transactions on Signal Processing}, 2020.

\bibitem{Xu2020CompressedCF}
H.~Xu, C.-Y. Ho, A.~M. Abdelmoniem, A.~Dutta, E.~H. Bergou, K.~Karatsenidis,
  M.~Canini, and P.~Kalnis.
\newblock Compressed communication for distributed deep learning: Survey and
  quantitative evaluation.
\newblock {\em Technical report}, 2020.

\bibitem{yu2019linear}
H.~Yu, R.~Jin, and S.~Yang.
\newblock On the linear speedup analysis of communication efficient momentum
  {SGD} for distributed non-convex optimization.
\newblock In {\em International Conference on Machine Learning}, pages
  7184--7193. PMLR, 2019.

\bibitem{Yuan2021RemovingDH}
K.~Yuan, S.~A. Alghunaim, and X.~Huang.
\newblock Removing data heterogeneity influence enhances network topology
  dependence of decentralized {SGD}.
\newblock {\em arXiv preprint arXiv:2105.08023}, 2021.

\bibitem{yuan2020influence}
K.~Yuan, S.~A. Alghunaim, B.~Ying, and A.~H. Sayed.
\newblock On the influence of bias-correction on distributed stochastic
  optimization.
\newblock {\em IEEE Transactions on Signal Processing}, 2020.

\bibitem{Yuan2021DecentLaMDM}
K.~Yuan, Y.~Chen, X.~Huang, Y.~Zhang, P.~Pan, Y.~Xu, and W.~Yin.
\newblock Decentlam: Decentralized momentum {SGD} for large-batch deep
  training.
\newblock In {\em 2021 IEEE/CVF International Conference on Computer Vision
  (ICCV)}, pages 3009--3019, 2021.

\bibitem{Yuan2022RevistOC}
K.~Yuan, X.~Huang, Y.~Chen, X.~Zhang, Y.~Zhang, and P.~Pan.
\newblock Revisiting optimal convergence rate for smooth and non-convex
  stochastic decentralized optimization.
\newblock In {\em Advances in Neural Information Processing Systems (NeurIPS)},
  2022.

\bibitem{Yuan2016OnTC}
K.~Yuan, Q.~Ling, and W.~Yin.
\newblock On the convergence of decentralized gradient descent.
\newblock {\em SIAM Journal of Optimization}, 26:1835--1854, 2016.

\bibitem{Zeiler2012ADADELTAAA}
M.~D. Zeiler.
\newblock Adadelta: An adaptive learning rate method.
\newblock {\em ArXiv}, abs/1212.5701, 2012.

\bibitem{Zhao2022BEERFO}
H.~Zhao, B.~Li, Z.~Li, P.~Richt'arik, and Y.~Chi.
\newblock Beer: Fast o(1/t) rate for decentralized nonconvex optimization with
  communication compression.
\newblock In {\em Advances in Neural Information Processing Systems}, 2022.

\bibitem{Zhao2019GlobalMC}
S.-Y. Zhao, Y.-P. Xie, H.~Gao, and W.-J. Li.
\newblock Global momentum compression for sparse communication in distributed
  {SGD}.
\newblock {\em ArXiv}, abs/1905.12948, 2019.

\bibitem{Zhou2019LowerBF}
D.~Zhou and Q.~Gu.
\newblock Lower bounds for smooth nonconvex finite-sum optimization.
\newblock {\em ArXiv}, abs/1901.11224, 2019.

\end{thebibliography}
\bibliographystyle{abbrv}

\allowdisplaybreaks
\appendix

\newpage

\newpage
\appendix

\section{Lower Bounds}\label{app:lower}
In this section, we provide the proofs for Theorem \ref{thm:lower-bounds} and \ref{thm:lower-bounds-contractive}. Without loss of generality, we assume algorithms to start from $x^{(0)}=0$ throughout this section. Otherwise we can consider the translated objectives $\tilde{f}_i(x):=f_i(x-x^{(0)})$.

To measure the optimization progress when algorithms starts from $x^{(0)}=0$,
we denote the $k$-th coordinate of a vector $x\in\RR^d$ by $[x]_k$ for $k=1,\dots,d$, and 
 let $\prog(x)$ be 
\begin{equation*}
    \prog(x):=\begin{cases}
    0, & \text{if $x=0$},\\
    \max_{1\leq k\leq d}\{k:[x]_k\neq 0\},& \text{otherwise}.
    \end{cases}
\end{equation*}
Similarly, for a set of multiple points $\cX=\{x_1,x_2,\dots\}$, we define $\prog(\cX):=\max_{x\in\cX}\prog(x)$.
As defined in \cite{Arjevani2019LowerBF,carmon2020lower}, a  function $f$ is called zero-chain if it satisfies
\begin{equation*}
    \prog(\nabla f(x))\leq \prog(x)+1,\quad\forall\,x\in\RR^d,
\end{equation*}
which implies that starting from $x^{(0)}=0$, a single gradient evaluation can only earn at most one more non-zero coordinate for the model parameters. 

Now we consider the setup of distributed learning with communication compression. For each worker $i$ and $t\geq 1$, we let $y^{(t)}_i$ be the point at which worker $i$ queries its $t$-th gradient oracle during the optimization procedure. 

Between the $t$-th and $(t+1)$-th gradient queries, each worker is allowed to communicate with the server by transmitting compressed messages. For worker $i$, we let $\cV^{(t)}_{i}$ denote the set of messages
that  worker $i$ aims to send to the server, \ie, the vectors before compression. 
Due to  communication compression, the vectors received by the server from worker $i$, which we denote by $\hat{\cV}^{(t)}_{i}$, are the compressed version of ${\cV}^{(t)}_{i}$ with some underlying compressors $C_i$, \ie, $\hat{\cV}^{(t)}_{i}\triangleq \{C_i(v):v\in{\cV}^{(t)}_{i}\}$.
Note that $\hat{\cV}^{(t)}_{i}$ is a set that may contain multiple vectors, and its cardinality equals the rounds of communication. After receiving the compressed messages from all workers, the server will broadcast  some messages back to all workers. We let   $\cU^{(t)}$  denote the set of messages that the server sends to workers.

Following the above description, we now extend the zero-respecting property \cite{carmon2020lower,Carmon2021LowerBF} to the setting of centralized distributed optimization with communication compression, which originally appears in single-node (stochastic) optimization.
\begin{definition}[\sc Zero-Respecting Algorithms]\label{def:zero-repsect}
    We say a distributed algorithm $A$ is zero-respecting if for any $t\geq 1$ and $1\leq k\leq d$, the following conditions are met:
\begin{enumerate}
    \item If worker $i$ queries at $y^{(t)}_i$ with $[y^{(t)}_i]_k\neq 0$, then one of the following must be true: 
    \begin{equation*}
        \begin{cases}
            \text{there exists some $1\leq s< t$ such that $[O_i(y^{(s)}_i;\zeta^{(s)}_i)]_k\neq 0$};\\
            \text{there exists some $1\leq s< t$ such that worker $i$ has received some $u\in\cU^{(s)}$ with $[u]_k\neq 0$};\\
            \text{there exists some $1\leq s<t$ such that worker $i$ has  a  compressed message  $\hat{v}\in\hat{\cV}_{i}^{(s)}$ with $[\hat{v}]_k\neq 0$}.
        \end{cases}
    \end{equation*}
    \item If worker $i$ aims to send some $v\in\cV_{i}^{(t)}$ with $[v]_k\neq 0$ to the server,  then one of the following must be true: 
    \begin{equation*}
        \hspace{-7mm}\begin{cases}
            \text{there exists some $1\leq s\leq  t$ such that $[O_i(y^{(s)}_i;\zeta^{(s)}_i)]_k\neq 0$};\\
            \text{there exists some $1\leq s< t$ such that worker $i$ has received some $u\in\cU^{(s)}$ with $[u]_k\neq 0$};\\
            \text{there exists some $1\leq s<t$ such that worker $i$ has  compressed some  $\hat{v}\in\cV_{i}^{(s)}$ with $[\hat{v}]_k\neq 0$}.
        \end{cases}
    \end{equation*}
    \item If the server aims to broadcast some $u\in\cU^{(t)}$ with $[u]_k\neq 0$ to workers, there exists some $1\leq s\leq  t$ such that the server has received some $v \in\cup_{1\leq i\leq n}\cV_{i}^{(s)}  $ with $[v]_k\neq 0$.
\end{enumerate}
\end{definition}

In essence, the above zero-respecting property requires that any expansion of non-zero coordinates in  $y^{(t)}_i$, $\cV_i^{(t)}$ associated with worker $i$ is  attributed to its historical local gradient updates, local compression, or synchronization with the server.
Meanwhile, it also requires that any expansion of non-zero coordinate in vectors held, including the final algorithmic output, in the server is due to the  received compressed messages from workers.
One can easily  verify  that most existing distributed algorithms with communication compression, including those listed in Table \ref{tab:unbiased} and Table \ref{tab:contractive}, are zero-respecting.

Next, we outline the proofs for the lower bounds. In each case  of Theorem \ref{thm:lower-bounds} and \ref{thm:lower-bounds-contractive}, we separately prove the two terms in the lower bound by constructing two hard-to-optimize examples respectively.
The construction of the example can be conducted in four steps: first, constructing zero-chain local functions $\{f_i\}_{i=1}^n$; second, constructing compressors $\{C_i\}_{i=1}^n\subseteq \cU_\omega$ or $\{C_i\}_{i=1}^n\subseteq \cC_\delta$ and independent gradient oracles $\{O_i\}_{i=1}^n\subseteq \cO_{\sigma^2}$ that hamper algorithms to expand the non-zero coordinates of model parameters; third, establishing a limitation, in terms of the non-zero coordinates of model parameters, 
for zero-respecting algorithms utilizing the predefined compression protocol 
with $T$ gradient queries and compressed communication on each worker; last, translating this limitation into the lower bound of the complexity measure defined in \eqref{eqn:measure} and \eqref{eqn:measure-nonconvex}.

In particular, we will use the following lemma in the analysis of the third step.
\begin{lemma}\label{lem:small-prob}
Given a constant $p\in[0,1]$ and 
random variables $\{B^{(t)}\}_{t=0}^T$ with $T\geq 1/p$ such that $B^{(t)}\leq B^{(t-1)}+1$ and  $\PP(B^{(t)}\leq B^{(t-1)}\mid \{B^{(r)}\}_{r=0}^{t-1})\geq 1-p$ for any $1\leq t\leq T$, it holds with probability at least $1-e^{-1}$ that $B^{(T)}\leq B^{(0)}+epT$.
\end{lemma}
\begin{proof}
Without loss of generality, we assume $B^{(0)}=0$; otherwise we can translate the variables as $\tilde{B}^{(t)}:=B^{(t)}-B^{(0)}$. Therefore, we have for any constant $c\geq 0$
\begin{align}\label{eqn:Lvbsubvs}
    \PP(B^{(T)}\geq cpT)&\leq e^{-cpT}\EE[e^{B^{(T)}}]=e^{-cpT}\EE\left[\exp\left(\sum_{t=1}^{T}\left(B^{(t)}-B^{(t-1)}\right)\right)\right].
\end{align}
Since $B^{(t)}\leq B^{(t-1)}+1$ and  $\PP(B^{(t)}\leq B^{(t-1)}\mid \{B^{(r)}\}_{r=0}^{t-1})\geq 1-p$, we have
\begin{align}
&\EE\left[\exp\left(\sum_{r=1}^{t}\left(B^{(r)}-B^{(r-1)}\right)\right)\,\Big|\, \{B^{(r)}\}_{r=0}^{t-1}\right]\\
    =&\exp\left(\sum_{r=1}^{t-1}\left(B^{(r)}-B^{(r-1)}\right)\right)\EE\left[\exp\left(B^{(t)}-B^{(t-1)}\right)\,\Big|\, \{B^{(r)}\}_{r=0}^{t-1}\right]\\
    \leq &\exp\left(\sum_{r=1}^{t-1}\left(B^{(r)}-B^{(r-1)}\right)\right)(p e + (1-p)).\label{eqn:vcvxdasdasdz}
\end{align}
Taking the expectation of \eqref{eqn:vcvxdasdasdz} with respect to $\{B^{(r)}\}_{r=0}^{t-1}$, we have 
\begin{equation}\label{eqn:vnisnfidfa}
\EE\left[\exp\left(\sum_{r=1}^{(t)}\left(B^{(r)}-B^{(r-1)}\right)\right)\right]\leq \EE\left[\exp\left(\sum_{r=1}^{t-1}\left(B^{(r)}-B^{(r-1)}\right)\right)\right] (p e + (1-p))
\end{equation}
Iterating \eqref{eqn:vnisnfidfa} over $t=1,\dots,T$, we obtain 
\[
\EE\left[\exp\left(\sum_{t=1}^{T}\left(B^{(t)}-B^{(t-1)}\right)\right)\right]\leq (1+(e-1)p)^T\leq e^{(e-1)pT},
\]
which, combined with \eqref{eqn:Lvbsubvs}, leads to
\[
\PP(B^{(T)}\geq cpT)\leq e^{(e-1-c)pT}.
\]
Finally, letting $c=e$ and using $pT\geq 1$ completes the proof.
\end{proof}

\subsection{Proof of Theorem \ref{thm:lower-bounds}}\label{app:lower-bounds}
\subsubsection{Strongly-Convex Case}\label{app:lower-sconvex}

\textbf{Example 1.}
In this example, we prove the lower bound $\Omega((1+\omega)\sqrt{{L}/{\mu}}\ln\left({\mu\Delta_x}/{\epsilon}\right))$.

\vspace{2mm}
\noindent (Step 1.) We assume the variable $x\in\ell_2\triangleq\{([x]_1,[x]_2,\dots,):\sum_{r=1}^{\infty}[x]_r^2<\infty\}$ to be infinitely dimensional and square-summable for simplicity. It is easy to adapt the argument for finitely dimensional variables as long as the dimension is proportionally larger than $T$. Let $M $ be
\begin{align*}
    M=\left[\begin{array}{ccccc}
2 & -1 & & &\\
-1 & 2 & -1 & &  \\
& -1 & 2 & -1 &   \\
& & \ddots & \ddots & \ddots 
\end{array}\right]\in\RR^{\infty\times \infty},
\end{align*}
then it is easy to see $0\preceq M\preceq 4I$. Without loss of generality, we assume $n$ is even, otherwise we can consider the case of $n-1$. Let $E_1\triangleq\{j:1\leq j\leq n/2\}$ and $E_2\triangleq\{j:n/2 < j\leq n\}$, and let
\begin{align}
    f_i(x)=\begin{cases}
    \frac{\mu}{2}\|x\|^2+\frac{L-\mu}{4}\left([x]_1^2+\sum_{r\geq 1}([x]_{2r}-[x]_{2r+1})^2-2\lambda [x]_{1}\right),&\text{if }i\in E_1;\\
    \frac{\mu}{2}\|x\|^2+\frac{L-\mu}{4}\sum_{r\geq 1}([x]_{2r-1}-[x]_{2r})^2,&\text{if }i\in E_2.
    \end{cases}
\end{align}
where $\lambda\in\RR$ is to be specified.
It is easy to see that $[x]_1^2+\sum_{r\geq 1}([x]_{2r}-[x]_{2r+1})^2-2\lambda [x]_{1}$ and $\sum_{r\geq 1}([x]_{2r-1}-[x]_{2r})^2$ are convex and $4$-smooth.
More importantly, all $f_i$'s defined above are zero-chain functions and satisfy
\begin{align}\label{eqn:jvisnvsdfads}
    \prog(\nabla f_i(x))
    \begin{cases}
    =\prog(x) +1,&\text{if } \{\prog(x) \text{ is even and } i\in E_1\}\cup\{\prog(x) \text{ is odd and }i\in E_2\};\\
    \leq \prog(x), &\text{otherwise}.
    \end{cases}
\end{align}
We further have $f(x)=\frac{1}{n}\sum_{i=1}^nf_i(x)=\frac{\mu}{2}\|x\|^2+\frac{L-\mu}{8}\left(x^\top Mx-2\lambda [x]_1\right)$.
For the functions defined above, we also establish that 
\begin{lemma}\label{lem:sc-communi}
Denote $\kappa \triangleq L/\mu\geq 1$, then it holds that for any $x$ and $r\geq 1$ satisfying $\prog(x)\leq r$,
\begin{align*}
f(x)-\min_{x}f(x) \geq \frac{\mu}{2} \left(1-\frac{2}{\sqrt{\kappa}+1}\right)^{2r}\|x^{(0)}-x^\star\|^2.
\end{align*}
\end{lemma}
\begin{proof}
The minimum $x^\star$ of function $f$ satisfies $\left(\frac{L-\mu}{4}M+\mu\right)x-\lambda\frac{L-\mu}{4}e_1=0$, which is equivalent to 
\begin{align}
    \frac{2(\kappa+1)}{\kappa-1}[x]_1-[x]_2&=\lambda ,
    \nonumber \\
    -[x]_{j-1}+ \frac{2(\kappa+1)}{\kappa-1}[x]_j-[x]_{j+1}&=0,\quad \forall\,j\geq 2.\label{eqn:vpwnmfwq}
\end{align}
Note that $q=(\sqrt{\kappa}-1)/(\sqrt{\kappa}+1)$ the only root of the equation $q^2- \frac{2(\kappa+1)}{\kappa-1}q+1=0$ that is smaller than $1$. Then it is straight forward to check $x^\star =\left([x^\star]_j=\lambda q^j\right)_{j\geq 1}$
satisfies \eqref{eqn:vpwnmfwq}. By the strong convexity of $f$, $x^\star$ is the unique solution. Therefore, when $\prog(x)\leq r$, it holds that 
\begin{align*}
    \|x-x^\star\|^2\geq \sum\limits_{j=r+1}^\infty \lambda^2q^{2j}=\lambda^2\frac{q^{2(r+1)}}{1-q^2}=q^{2r}\|x^{(0)}-x^\star\|^2.
\end{align*}
Finally, using the strong convexity of $f$ leads to the conclusion.
\end{proof}

In the proof of Lemma \ref{lem:sc-communi}, we have
\begin{equation*}
    \|x^{(0)}-x^\star\|^2=\lambda^2\sum_{j=1}^\infty q^{2j}=\lambda^2\frac{q^2}{1-q^2}
\end{equation*}
Therefore, for any given $\Delta_x>0$, letting $\lambda=\sqrt{((1-q^2) {\color{black}\Delta_x})/{q^2}}$ results in $\|x^{(0)}-x^\star\|^2=\Delta_x$. Therefore, our construction ensures $\{f_i\}_{i=1}^n\in \cF_{L,\mu}^{\Delta_x}$.

\vspace{2mm}
\noindent (Step 2.)  We consider the gradient oracles that return the full-batch gradients, \ie, $O_i(x)=\nabla f_i(x)$, $\forall\,x$ and $1\leq i\leq n$, meaning that there is no gradient stochasticity  during the entire optimization procedure.
As for the construction of $\omega$-unbiased compressors, we consider $\{C_i\}_{i=1}^n$ to be the scaled random sparsification operators with shared randomness (\ie, all workers share the same random seed), where the scaling is imposed to ensure unbiasedness.
Specifically, in a compressed communication step, each coordinate on all workers has a probability $(1+\omega)^{-1}$ to be chosen to be communicated, and once being chosen, its $(1+\omega)$-multiplication  will be transmitted. The  indexes of chosen coordinates are identical across all workers due to the shared randomness and are sampled uniformly randomly per communication. 
Since each coordinate  has probability $(1+\omega)^{-1}$ to be chosen, we have  for any $x\in\RR^d$
\begin{equation*}
   \EE[C_i(x)]=\EE\left[\left((1+\omega)[x]_k\mathds{1}\{k \text{ is chosen}\}\right)_{k\geq 1}\right]=\left((1+\omega)[x]_k\PP(k \text{ is chosen})\right)_{k\geq 1}=x,
\end{equation*}
and 
\begin{align*}
   \EE[\|C_i(x)-x\|^2]=&\sum_{k\geq1}\EE\left[\left((1+\omega)[x]_k\mathds{1}\{k \text{ is chosen}\}-[x]_k\right)^2\right]\\
   =&\sum_{k\geq 1}[x]_k^2\left(\omega^2\PP(k \text{ is chosen})+\PP(k \text{ is not chosen})\right)=\omega\sum_{k\geq 1} [x]_k^2=\omega \|x\|^2.
\end{align*}
Therefore, the above construction gives $\{C_i\}_{i=1}^n\subseteq\cU_\omega$.

\vspace{2mm}
\noindent  (Step 3.) For any $t\in \NN_{+}$ and $1\leq i\leq n$, let ${v}^{(t)}_i$ be the vector that worker $i$ aims to send to the server at the $t$-th communication step.
Due to communication compression, the server can only receive the compressed message $C_i(v^{(t)}_{i})$, which we denote by $\hat{v}^{(t)}_{i}\triangleq C_i(v^{(t)}_{i})$.
Similarly, we let $u^{(t)}$ be the vector that   the server broadcasts to all workers at the $t$-th communication.

Since the algorithmic $\hat{x}^{(t)}$ output by the server is lies in the linear space spanned by received messages, we have 
\begin{equation}\label{eqn:gjvowfnoqwmfgqw=0s}
    \prog(\hat{x}^{(t)})\leq \max_{1\leq r\leq t}\max_{1\leq i\leq n}\prog(\hat{v}^{(r)}_{i}).
\end{equation}
We next bound $B^{(t)}\triangleq \max_{1\leq r\leq  t}\max_{1\leq i\leq n}\prog(\hat{v}_i^{(r)})$ ($B^{(0)}:=0$) by showing that $\{B^{(t)}\}_{t=0}^\infty$ satisfies Lemma \ref{lem:small-prob} with $p=(1+\omega)^{-1}$.

Since algorithm $A$ satisfies the zero-respecting property, due to \eqref{eqn:jvisnvsdfads}, each worker can only achieve one more non-zero coordinate locally by updates based on local gradients alone. In other words, upon receiving messages $\{u_i^{(r)}\}_{r=1}^{t-1}$ from the server,  worker $i$ can at most increase the number of non-zero coordinates of its local vectors by $1$. As a result, we have
\begin{equation}\label{eqn:gjvowfnoqwmfgqw=1s}
    \prog(v_i^{(t)})\leq \max_{1\leq r< t}\prog(u^{(r)}_{i})+1.
\end{equation}
Using the third point of Definition \ref{def:zero-repsect}, we have
\begin{equation}\label{eqn:gjvowfnoqwmfgqw=2s}
    \prog(u^{(r)}_i)\leq\max_{1\leq q\leq r} \max_{1\leq j\leq n}\prog(\hat{v}_j^{(q)}).
\end{equation}
Combining \eqref{eqn:gjvowfnoqwmfgqw=1s} and \eqref{eqn:gjvowfnoqwmfgqw=2s}, we reach 
\begin{equation}\label{eqn:ghoewjgoegqwegftqw}
    \max_{1\leq r\leq t} \max_{1\leq i\leq n}\prog({v}_i^{(r)})\leq \max_{1\leq r< t}\max_{1\leq i\leq n}\prog(\hat{v}_i^{(r)})+1=B^{(t-1)}+1.
\end{equation}

It follows the definition of the constructed $C_i$ in Step 2 that  $\max_{1\leq i\leq n}\prog(\hat{v}^{(t)}_{i})\leq  \max_{1\leq i\leq n}\prog(v^{(t)}_{i})$ and thus 
\begin{equation}
    B^{(t)}\leq \max_{1\leq r\leq t} \max_{1\leq i\leq n}\prog({v}_i^{(r)})\leq \max_{1\leq r< t}\max_{1\leq i\leq n}\prog(\hat{v}_i^{(r)})+1=B^{(t-1)}+1.
\end{equation}
Furthermore, since  compressors $\{C_i\}_{i=1}^n$ are with the shared randomness, they
have a probability $\omega/(1+\omega)$ to zero out the $\max_{1\leq i\leq n}\prog(v^{(t)}_{i})$-th coordinate in the communication. As a result, we have 
\begin{equation}
    \PP\left(\max_{1\leq i\leq n}\prog(\hat{v}^{(t)}_{i})< \max_{1\leq i\leq n}\prog(v^{(t)}_{i})\right)\geq \omega /(1+\omega).
\end{equation}
On the event $\max_{1\leq i\leq n}\prog(\hat{v}^{(t)}_{i})< \max_{1\leq i\leq n}\prog(v^{(t)}_{i})$ of probability at least $\omega/(1+\omega)$, we have 
\begin{equation*}
    B^{(t)}=\max\{B^{(t-1)},\max_{1\leq i\leq n}\prog(\hat{v}_i^{(t)})\}\leq \max\{B^{(t-1)}, \max_{1\leq i\leq n}\prog(v^{(t)}_{i})-1\}\leq B^{(t-1)}
\end{equation*}
where the last inequality is due to \eqref{eqn:ghoewjgoegqwegftqw}.
To summarize, we have $B^{(t)}\leq B^{(t-1)}+1$ and $\PP(B^{(t)}\leq B^{(t-1)}\mid \{B^{(r)}\}_{r=0}^{t-1})\geq \omega/(1+\omega)$.

Using Lemma \ref{lem:small-prob}, we have that for any $t\geq (1+\omega)^{-1}$, it holds with probability at least $1-e^{-1}$ that 
$B^{(t)}\leq et/(1+\omega)$ and consequently $\prog(\hat{x}^{(t)})\leq et/(1+\omega)$ due to \eqref{eqn:gjvowfnoqwmfgqw=0s}.

\vspace{2mm}
\noindent  (Step 4.) Using Lemma \ref{lem:sc-communi} and  that $\prog(\hat{x}^{(t)})\leq et/(1+\omega)$  with probability at least $1-e^{-1}$, we obtain
\begin{align}\label{eqn:vncxiv}
    \EE[f(\hat{x}^{(t)})]-\min_x f(x)\geq & \frac{(1-e^{-1})\mu}{2}\left(1-\frac{2}{\sqrt{\kappa}+1}\right)^{2et/(1+\omega)}\Delta_x\\
    =&\Omega\left(\mu\Delta_x \exp\left({-\frac{4et}{(\sqrt{\kappa}+1)(1+\omega)}}\right)\right).
\end{align}
Therefore, to ensure $\EE[f(\hat{x}^{(T)})]-\min_x f(x)\leq \epsilon$ with $0<\epsilon\leq\frac{(1-e^{-1})\mu}{2}\left(1-\frac{2}{\sqrt{\kappa}+1}\right)^{2e}\Delta_x \triangleq c_\kappa \mu\Delta_x$, \eqref{eqn:vncxiv} implies that
the lower bound
$T=\Omega((1+\omega)\sqrt{\kappa}\ln(\mu \Delta_x/\epsilon))$.

\vspace{2mm}
\noindent\textbf{Example 2.} The lower bound $\Omega(\sigma^2/{\mu n\epsilon})$ follows the same analysis as the lower bound $\Omega(\sigma^2/(\mu n T))$ of \cite[Theorem 3]{Yuan2022RevistOC}. We thus omit the proof here.

\subsubsection{Generally-Convex Case}\label{app:lower-convex}
\textbf{Example 1.}
In this example, we prove the lower bound $\Omega((1+\omega)(L\Delta_x/\epsilon)^{1/2})$.

\vspace{2mm}
\noindent (Step 1.) We assume the variable $x\in\RR^d$ where $d$ can be sufficiently large and is to be determined.  Let $M $ be
\begin{align*}
    M=\left[\begin{array}{ccccc}
2 & -1 & & & \\
-1 & 2 & -1 & & \\
& \ddots & \ddots & \ddots &\\
& &  -1 & 2 & -1\\
& & & -1 & 2
\end{array}\right]\in\RR^{d\times d},
\end{align*}
then it is easy to see $0\preceq M\preceq 4I$. Similar to example 1 of Appendix \ref{app:lower-sconvex}, we assume $n$ is even and let $E_1\triangleq\{j:1\leq j\leq n/2\}$, $E_2\triangleq\{j: n/2 < j\leq n\}$, and
\begin{align}
    f_i(x)=\begin{cases}
    \frac{L}{4}\left([x]_1^2+\sum_{r\geq 1}([x]_{2r}-[x]_{2r+1})^2-2\lambda [x]_{1}\right),&\text{if }i\in E_1;\\
    \frac{L}{4}\sum_{r\geq 1}([x]_{2r-1}-[x]_{2r})^2,&\text{if }i\in E_2.
    \end{cases}
\end{align}
where $\lambda\in\RR$ is to be specified.
It is easy to see that $[x]_1^2+\sum_{r\geq 1}([x]_{2r}-[x]_{2r+1})^2-2\lambda [x]_{1}$ and $\sum_{r\geq 1}([x]_{2r-1}-[x]_{2r})^2$ are convex and $4$-smooth, which implies all $f_i$ are $L$-smooth.
We further have $f(x)=\frac{1}{n}\sum_{i=1}^nf_i(x)=\frac{L}{8}\left(x^\top Mx-2\lambda [x]_1\right)$. 
The $f_i$s defined above are also zero-chain functions satisfying \eqref{eqn:jvisnvsdfads}.

Following \cite{nesterov2003introductory}, it is easy to verify that the optimum of $f$ satisfies
\begin{equation*}
    x^\star = \left(\lambda\left(1-\frac{k}{d+1}\right)\right)_{1\leq k\leq d}\quad \text{and}\quad f(x^\star)=\min_x f(x)=-\frac{\lambda^2 Ld }{8(d+1)}.
\end{equation*}
More generally, it holds for any $0\leq k\leq d$ that
\begin{equation}\label{eqn:k-optimum}
    \min_{x: \,\prog(x)\leq k}f(x)=-\frac{\lambda^2 Lk}{8(k+1)}.
\end{equation}
Since $\|x^{(0)}-x^\star\|^2=\frac{\lambda^2}{(d+1)^2}\sum_{k=1}^dk^2=\frac{\lambda^2d(2d+1)}{6(d+1)}\leq \frac{\lambda^2 d}{3}$, letting $\lambda=\sqrt{3\Delta_x/d}$, we have $\{f_i\}_{i=1}^n\in \cF_{L,0}^{\Delta_x}$.

\vspace{2mm}
\noindent (Step 2.) Same as step 2 of example 1 of the strongly-convex case, we consider the full-batch gradient oracles and the scaled random
sparsification $\{C_i\}_{i=1}^n$ with shared randomness.

\vspace{2mm}
\noindent (Step 3.) Following the same argument as step 3 of example 1 of the strongly-convex case, we have that for any $t\geq (1+\omega)^{-1}$, it holds with probability at least $1-e^{-1}$ that 
 $\prog(\hat{x}^{(t)})\leq et/(1+\omega)$.

\vspace{2mm}
\noindent (Step 4.)
Thus, combining \eqref{eqn:k-optimum}, we have
\begin{align}
    \EE[f(\hat{x}^{(t)})]-\min_x f(x)\geq &(1-e^{-1})\frac{\lambda^2L}{8}\left(\frac{d}{d+1}-\frac{et/(1+\omega)}{1+et/(1+\omega)}\right)\\
=& (1-e^{-1})\frac{3L\Delta_x}{8d}\left(\frac{d}{d+1}-\frac{et/(1+\omega)}{1+et/(1+\omega)}\right)
\end{align}
Letting $d=1+et/(1+\omega)$, we further have 
\begin{equation}
    \EE[f(\hat{x}^{(t)})]-\min_x f(x)\geq \frac{3(1-e^{-1})L\Delta_x}{16et(1+\omega)^{-1}(1+2et(1+\omega)^{-1})}=\Omega\left(\frac{(1+\omega)^2L\Delta_x}{t^2}\right).\label{eqn:vxinbsd}
\end{equation}
Therefore, to ensure $\EE[f(\hat{x}^{(T)})]-\min_x f(x)\leq \epsilon$ with $0<\epsilon\leq\frac{3(1-e^{-1})}{16e(1+2e)}L\Delta_x\triangleq cL\Delta_x$, \eqref{eqn:vncxiv} implies that
the lower bound
$T=\Omega((1+\omega)(L\Delta /\epsilon)^\frac{1}{2})$.

\vspace{2mm}
\noindent \textbf{Example 2.}
The lower bound $\Omega(\Delta \sigma^2/(n\epsilon^2))$ is proved by reducing the optimization problem to a statistical testing problem. Then the lower bound for the optimization problem can transformed fro the power limitation of the statistical testing problem. Such a proof idea  has appeared in \cite[Theorem 1]{Agarwal2010InformationTheoreticLB}, \cite[Theorem 5.2.10]{duchi2016introductory}, \cite[Theorem 3]{Yuan2022RevistOC}. 

We consider all functions $f_i=f$ are homogeneous and  $C_i$ are identity, meaning that there is no compression error in the optimization procedure. We then choose two functions $f^1,\,f^{-1}\in\cF_{L,0}^\Delta$ which are close enough to each other so that $f^1$ and $f^{-1}$ are hard to distinguish. The indistinguishability argument follows the standard Le Cam's method in hypothesis testing in statistics.
Next, we carefully show that the indistinguishability between $f^1$ and $f^{-1}$ can be properly translated  into the lower bound of the algorithmic performance.

We consider the $1$-dimensional case, \ie, $d=1$.
Let 
\[
f^{v}= \begin{cases}
    \sigma p (x-v\sqrt{\Delta_x})-\frac{\sigma^2 p^2}{2L },&\text{if }x>v\sqrt{\Delta_x}+\frac{\sigma p}{L};\\
    \frac{L}{2}(x-v\sqrt{\Delta_x})^2,&\text{if }x\in[v\sqrt{\Delta_x} -\frac{\sigma p }{L},v\sqrt{\Delta_x}+\frac{\sigma p }{L}];\\
    -\sigma p (x-v\sqrt{\Delta_x})-\frac{\sigma^2 p^2}{2L},&\text{if }x<v\sqrt{\Delta_x}-\frac{\sigma p}{L},
\end{cases}
\]
where $v\in\{\pm 1\}$, and $p\in[0,\min\{\frac{4}{5}, \frac{L\sqrt{\Delta}}{2\sigma}\}]$ is a parameter to be determined. Clearly, 
\[
\nabla f^{v}= \begin{cases}
    \sigma p ,&\text{if }x>v\sqrt{\Delta_x} +\frac{\sigma p}{L};\\
    L(x-v\sqrt{\Delta_x}),&\text{if }x\in[v\sqrt{\Delta_x} -\frac{\sigma p }{L},v\sqrt{\Delta_x}+\frac{\sigma p }{L}];\\
    -\sigma p ,&\text{if }x<v\sqrt{\Delta_x}-\frac{\sigma p}{L},
\end{cases}
\]
so $f^v$ is convex and $L$-smooth.
The optimum $x^{v,\star}$ of $f^v$ is $v\sqrt{\Delta_x}$ with function value $f^{v,\star}\triangleq\min_{x\in\RR^d} f^{v}(x)=0$. Therefore, for each $v\in\{\pm1\}$, our construction ensures $\{f_i\}_{i=1}^n\in \cF_{L,0}^{\Delta_x}$ when all $f_i$ are chosen to be $f^v$ for any $v\in\{\pm 1\}$.

For each $v\in\{\pm1\}$ and $1\leq i\leq n$, we construct the stochastic gradient oracle as follows:
\begin{equation*}
    O_i^v(x)=\begin{cases}
        \sigma & \text{with probability } \frac{1}{2}+\frac{\nabla f^v(x)}{2\sigma };\\
        -\sigma & \text{with probability } \frac{1}{2}-\frac{\nabla f^v(x)}{2\sigma };
    \end{cases}
\end{equation*}
that is, the query output at given point $x$ follows a Bernoulli distribution.
Since $|\nabla f^v|\leq \sigma p$, we have $\{\frac{1}{2}+\frac{\nabla f^v(x)}{2\sigma },\frac{1}{2}-\frac{\nabla f^v(x)}{2\sigma }\}\subseteq[\frac{1-p}{2},\frac{1+p}{2}]\subseteq[0,1]$. So the stochastic oracles are well-defined. Furthermore, it is easy to find that 
\begin{equation*}
    \EE[O_i^v(x)]=\left(\frac{1}{2}+\frac{\nabla f^v(x)}{2\sigma }\right) \sigma-\left(\frac{1}{2}-\frac{\nabla f^v(x)}{2\sigma }\right) \sigma=\nabla f^v(x)
\end{equation*}
and 
\begin{equation*}
    \EE[\|O_i^v(x)-\nabla f^v(x)\|^2]=\left(\frac{1}{2}+\frac{\nabla f^v(x)}{2\sigma }\right) (\sigma-\nabla f^v(x))^2+\left(\frac{1}{2}-\frac{\nabla f^v(x)}{2\sigma }\right) (\sigma+\nabla f^v(x))^2=\sigma^2-|\nabla f^v(x)|^2\leq \sigma^2.
\end{equation*}
Therefore, for each $v\in\{\pm1\}$, our construction ensures $\{O_i^v\}_{i=1}^n\in \cO_{\sigma^2}$.

Let $\vS^{(t)}:=\{(\vy^{(r)}\triangleq(y^{(r)}_1,\dots,y^{(r)}_n),\tilde{\vg}(\vy^{(r)})\triangleq (\tilde{g}_1(y^{(r)}_1),\dots,\tilde{g}_n(y^{(r)}_n))\}_{r=0}^{t-1}$ be set of variables corresponding to the sequence of the inputs and outputs of gradient queries on all workers.  Let $P^{v,(t)}:=P(\vS^{(t)}\mid f^v)$ be the distribution of  $\vS^{(t)}$ if the underlying functions are chosen as $f_1=\cdots=f_n=f^v$.

Since $\Delta_x \geq 2\sigma p /L$ by the choice of  $p$, we have
\begin{align*}
    &\sup \left\{\delta \geq 0: (f^v(x) -f^{v,\star}-\delta)( f^{-v}(x) -f^{-v,\star}-\delta)\leq 0,\forall\, x \right\}\\
=&f^1(0)-f^{1,\star}=\sigma p\left(\sqrt{\Delta_x}- \frac{\sigma p}{L}\right)\geq \frac{\sigma\sqrt{\Delta_x} p}{2}.
\end{align*}                   
Therefore, using \cite[Proposition 5.1.6, Theorem 5.2.4]{duchi2016introductory}, we have, for any optimization procedure $\hat{x}^{(t)}$ based on the gradient queries, that
\begin{align}\label{eqn:vbmidcb}
    \max_{v\in\{\pm1\}}\EE[f^{v}(\hat{x}^{(t)})]-f^{v,\star}\geq \frac{\sigma \sqrt{\Delta_x} p}{4} \left(1-\sqrt{\frac{1}{2}D_\mathrm{KL}(P^{1,(t)}\,\|\,P^{-1,(t)}})\right).
\end{align}
By the construction of the gradient oracles, for any query point $x$, the query output $O_i^v(x)$ follows a Bernoulli distribution whose mean lies in $[\frac{1-p}{2},\frac{1+p}{2}]$.  Therefore, we have for any $x$ and $1\leq i\leq n$,
\begin{equation}
    \label{eqn:bvndfisvsd}
    D_\mathrm{KL}(P(\tilde{g}_i(x)\mid f^1)\|P(\tilde{g}_i(x)\mid f^{-1}))
    \leq D_\mathrm{KL}\left(\mathrm{Ber}\left(\frac{1+p}{2}\right)\,\Big\|\,\mathrm{Ber}\left(\frac{1-p}{2}\right)\right) =p \ln(\frac{1+p}{1-p})\leq 3p^2,
\end{equation}
where the last inequality is because $p\leq 4/5$. Using \eqref{eqn:bvndfisvsd}, by the independence of the oracles and the additivity of the KL divergence, we have 
\begin{equation}\label{eqn:jvomdergs}
    D_\mathrm{KL}(P^{1,(t)}\,\|\,P^{-1,(t)})=\sum_{i=1}^n\sum_{r=1}^t D_\mathrm{KL}(P(\tilde{g}_i(y_i^{(r)})\mid f^1)\|P(\tilde{g}_i(y_i^{(r)})\mid f^{-1}))\leq 3nT p^2.
\end{equation}
Therefore, plugging \eqref{eqn:jvomdergs} into \eqref{eqn:vbmidcb} and setting $p=\min\{\frac{L\sqrt{\Delta_x}}{2\sigma}, \frac{1}{\sqrt{6nt}}\}$, we reach 
\begin{equation*}
    \max_{v\in\{\pm1\}}\EE[f^{v}(\hat{x}^{(t)})]-f^{v,\star}\geq \frac{\sigma\sqrt{\Delta_x} p}{8}\geq \min\left\{\frac{\sigma \Delta_x}{8\sqrt{6nt}}, \frac{L\Delta_x}{16}\right\}.
\end{equation*}
Furthermore, for any $\epsilon< \frac{L\Delta_f}{16}$ and any algorithm $A$, there is one instance associated with either $v=1$ or $v=-1$ such that at at least $\Omega(\Delta \sigma^2/(n\epsilon^2))$ gradient queries are needed to reach $\EE[f^{v}(\hat{x})]-f^{v,\star}\leq \epsilon$.

\subsubsection{Non-Convex Case}

We first state some key non-convex zero-chain functions that will be used to facilitate the proof for the non-convex case.
\begin{lemma}[Lemma 2 of \cite{Arjevani2019LowerBF}]\label{lem:basic-fun}
Let function 
\begin{equation*}
    h(x):=-\psi(1) \phi([x]_{1})+\sum_{j=1}^{d-1}\Big(\psi(-[x]_j) \phi(-[x]_{j+1})-\psi([x]_j) \phi([x]_{j+1})\Big)
\end{equation*}
where for $\forall\, z \in \mathbb{R},$
$$
\psi(z)=\begin{cases}
0 & z \leq 1 / 2; \\
\exp \left(1-\frac{1}{(2 z-1)^{2}}\right) & z>1 / 2, 
\end{cases} \quad \quad \mbox{and} \quad \quad  \phi(z)=\sqrt{e} \int_{-\infty}^{z} e^{-\frac{1}{2} t^{2}} \mathrm{d}t.
$$
The function $h(x)$ satisfies the following properties:
\begin{enumerate}
    \item $h$ is zero-chain, \emph{\ie}, $\prog(\nabla h(x))\leq \prog(x)+1$ for all $x\in\RR^d$.
    \item $h(x)-\inf_{x} h(x)\leq \Delta_0 d$, $\forall\,x\in\RR^d$ with $\Delta_0=12$.
    \item $h$ is $L_0$-smooth with $L_0=152$.
    \item $\|\nabla h(x)\|_\infty\leq G_\infty $, $\forall\,x\in\RR^d$ with $G_\infty = 23$.
    \item $\|\nabla h(x)\|_\infty\ge 1 $ for any $x\in\RR^d$ with $[x]_d=0$. 
\end{enumerate}
\end{lemma}

Similarly, if we split $h$ into two components, we have the following results:
\begin{lemma}\label{lem:basic-fun2}
Letting functions 
\begin{equation*}
    h_1(x):=-2\psi(1) \phi([x]_{1})+2\sum_{j \text{ even, } 0< j<d}\Big(\psi(-[x]_j) \phi(-[x]_{j+1})-\psi([x]_j) \phi([x]_{j+1})\Big)
\end{equation*}
and 
\begin{equation*}
    h_2(x):=2\sum_{j \text{ odd, } 0<j<d}\Big(\psi(-[x]_j) \phi(-[x]_{j+1})-\psi([x]_j) \phi([x]_{j+1})\Big),
\end{equation*}
then $h_1$ and $h_2$ satisfy  the following properties:
\begin{enumerate}
    \item $\frac{1}{2}(h_1+h_2)=h$, where $h$ is defined in Lemma \ref{lem:basic-fun}.
    \item $h_1$ and $h_2$ are zero-chain, \emph{\ie}, $\prog(\nabla h_i(x))\leq \prog(x)+1$ for all $x\in\RR^d$ and $i=1,2$. Furthermore, if $\prog(x)$ is odd, then $\prog(\nabla h_1(x))\leq \prog(x)$; if $\prog(x)$ is even, then $\prog(\nabla h_2(x))\leq \prog(x)$.
    \item $h_1$ and $h_2$ are also $L_0$-smooth with ${L_0}=152$. 
\end{enumerate}
\end{lemma}
\begin{proof}
The first property follows the definitions of $h_1$, $h_2$, and $h$. The second property follows  \cite[ Lemma 1]{carmon2020lower} that $\psi^{(m)}(0)=0$ for any $m\in\mathbb{N}$. Now we prove the third property. Noting that the Hessian of $h_k$ for $k=1,2$ is tridiagonal and symmetric, we have, by the Schur test, for any $x\in\RR^d$ and $k=1,2$ that 
\begin{align}
    \|\nabla^2 h_k(x)\|_2\leq&\sqrt{\|\nabla^2 h_k(x)\|_1\|\nabla^2 h_k(x)\|_\infty}=\|\nabla^2 h_k(x)\|_1.\label{eqn:jvoqfqfvcdv}
\end{align}
Furthermore, by using \cite[Observation 2]{Arjevani2019LowerBF}:
\begin{equation*}
    0 \leq \psi \leq e, \;\; 0 \leq \psi^{\prime} \leq \sqrt{54 / e}, \;\;\left|\psi^{\prime \prime}\right| \leq 32.5,\;\; 0 \leq \phi \leq \sqrt{2 \pi e}, \;\; 0 \leq \phi^{\prime} \leq \sqrt{e}\text{ \;and }\left|\phi^{\prime \prime}\right| \leq 1,
\end{equation*}
it is easy to verify
\begin{align}
\|\nabla^2 h_k(x)\|_1
    \leq& 2\max\left\{\sup _{z \in \mathbb{R}}\left|\psi^{\prime \prime}(z)\right| \sup _{z \in \mathbb{R}}|\phi(z)|,\sup _{z \in \mathbb{R}}\left|\psi(z)\right| \sup _{z \in \mathbb{R}}|\phi^{\prime \prime}(z)|\right\}+2 \sup _{z \in \mathbb{R}}\left|\psi^{\prime}(z)\right| \sup _{z \in \mathbb{R}}\left|\phi^{\prime}(z)\right|\leq 152,\label{eqn:vjowemfqvgdwv}
\end{align}
where $\psi^{\prime}(z)$ and $\psi^{\prime \prime}(z)$ are the first- and second-order derivative of $\psi(z)$, respectively; $\phi^{\prime}(z)$ and $\phi^{\prime \prime}(z)$ are the first- and second-order derivative of $\phi(z)$, respectively. Combining \eqref{eqn:vjowemfqvgdwv} with \eqref{eqn:jvoqfqfvcdv} leads to the conclusion.
\end{proof}

\vspace{2mm}
\noindent \textbf{Example 1.}
In this example, we prove the lower bound $\Omega((1+\omega)L\Delta_f/\epsilon)$.

\vspace{2mm}
\noindent (Step 1.) We let $f_i=L\lambda^2 h_1(x/\lambda)/L_0$, $\forall\,1\leq i\leq n/2$ and $f_i=L\lambda^2 h_2(x/\lambda)/L_0$, $\forall\,n/2<i\leq n$, where $h_1$ and $h_2$ are defined in Lemma \ref{lem:basic-fun2}, and $\lambda>0$ will be specified later. By the definitions of $h_1$ and $h_2$, we have that $f_i$, $\forall\,1\leq i\leq n$, is  zero-chain  and $f(x)=\frac{1}{n}\sum_{i=1}^n f_i(x)=L\lambda^2 h(x/\lambda)/L_0$. Since $h_1$ and $h_2$ are also $L_0$-smooth, $\{f_i\}_{i=1}^n$ are $L$-smooth. Furthermore, since
\begin{equation*}
    f(0)-\inf_x f(x)=\frac{L\lambda^2}{L_0}(h(0)-\inf_x h(x)) {\leq}\frac{L\lambda^2\Delta_0d}{L_0},
\end{equation*}
to ensure $f_i\in\cF_{L}^{\Delta_f}$, it suffices to let 
\begin{equation}\label{eqn:jgowemw}
    \frac{L\lambda^2\Delta_0d}{L_0}= \Delta_f, \quad \text{\ie,}\quad \lambda= \sqrt{\frac{L_0 \Delta_f}{L\Delta_0 d}}.
\end{equation}

\vspace{2mm}
\noindent (Step 2.) Same as step 2 of example 1 of the strongly-convex case, we consider the full-batch gradient oracles and the scaled random
sparsification $\{C_i\}_{i=1}^n$ with shared randomness.

\vspace{2mm}
\noindent (Step 3.) Following the same argument as step 3 of example 1 of the strongly-convex case, we have that for any $t\geq (1+\omega)^{-1}$, it holds with probability at least $1-e^{-1}$ that 
 $\prog(\hat{x}^{(t)})\leq et/(1+\omega)$.

\vspace{2mm}
\noindent (Step 4.)
Let $d = 1+et/(1+\omega)$, by using  the last point in Lemma \ref{lem:basic-fun} and recalling \eqref{eqn:jgowemw}, we have
\begin{equation}\label{eqn:Lvjowenfq22}
    \EE[\|\nabla f(\hat{x})\|^2]\geq (1-e^{-1})\frac{L^2\lambda^2}{L_0^2}=(1-e^{-1})\frac{L\Delta_f}{L_0\Delta_0 d}=\frac{(1-e^{-1})L\Delta_f}{L_0\Delta_0 (1+et/(1+\omega))}=\Omega\left((1+\omega)\frac{L\Delta_f}{t}\right).
\end{equation}
Therefore, to ensure $\EE[\|\nabla f(\hat{x}^{(T)})\|^2]\leq \epsilon$ with $0<\epsilon\leq\frac{(1-e^{-1})}{L_0\Delta_0(1+e)}L\Delta_f\triangleq cL\Delta_f$, \eqref{eqn:Lvjowenfq22} implies that
the lower bound
$T=\Omega((1+\omega)L\Delta_f /\epsilon)$.

\vspace{2mm}
\noindent \textbf{Example 2.} The lower bound $\Omega(\Delta_f L \sigma^2/( n\epsilon^2))$ follows the existing result of \cite[Theorem 1]{Lu2021OptimalCI} and of \cite[Theorem 2]{Yuan2022RevistOC}. We thus omit the proof here.

\subsection{Proof of Theorem \ref{thm:lower-bounds-contractive}}
\label{app-lower-bounds-contractive}
Theorem \ref{thm:lower-bounds-contractive} essentially follows the same analyses as in Theorem \ref{thm:lower-bounds}. The only difference is in the constructions of contractive compressors
where we shall not impose the scaling procedure in compression outputs.
Given this difference, one can easily verify that
\begin{align*}
   \EE[\|C_i(x)-x\|^2]=&\sum_{k\geq 1}\EE\left[\left([x]_k\mathds{1}\{k \text{ is chosen}\}-[x]_k\right)^2\right]=\sum_{k\geq 1}[x]_k^2\PP(k \text{ is not chosen})=\frac{\omega}{1+\omega}\|x\|^2.
\end{align*}
Therefore, we have $\{C_i\}_{i=1}^n\subseteq\cC_{\omega/(1+\omega)}$.
Note that the scaling procedure does not change $\prog$, and thus 
has no effect on the argument based on non-zero coordinates. By setting  ${\omega}/{(1+\omega)} =1-\delta$, \ie, $\omega = \delta^{-1}-1$, we can easily adapt the proof of Theorem \ref{thm:lower-bounds} to prove Theorem \ref{thm:lower-bounds-contractive}.

\section{Upper Bounds}

\subsection{MSC Property}
\subsubsection{Proof of Lemma \ref{lm:MSC}}
\label{app-MSC-lemma}
\begin{itemize}
    \item \textbf{Contractive compressor: }From the definition of MSC$(\cdot,C,R)$ we have intermediate variables $v^{(0)}=0$ and $v^{(r)}=v^{(r-1)}+C(x-v^{(r-1)})$, $r=1,\cdots,R$. Since $C$ is $\delta$-contractive, we have
\begin{align*}
    \EE[\|\mathrm{MSC}(x,C,R)-x\|^2]=&\EE[\|v^{(R-1)}+C(x-v^{(R-1)})-x\|^2]\\
    =&\EE\left[\EE[\|C(x-v^{(R-1)})+v^{(R-1)}-x\|^2\mid v^{(R-1)}]\right]\\
    \le&(1-\delta)\EE[\|v^{(R-1)}-x\|^2].
\end{align*}
Iterating the above inequality with respect to $r=R-1,\cdots,1,0$, we reach
\begin{align*}
    &\EE[\|\mathrm{MSC}(x,C,R)-x\|^2]\le(1-\delta)\EE[\|v^{(R-1)}-x\|^2]\\
    \le&(1-\delta)^2\EE[\|v^{(R-2)}-x\|^2]\le\cdots\le(1-\delta)^R\EE[\|v^{(0)}-x\|^2]=(1-\delta)^R\|x\|^2.
\end{align*}
\item\textbf{Unbiased compressor: }From the definition of MSC$(\cdot,C,R)$ we have $v^{(0)}=0$ and $v^{(r)}=v^{(r-1)}+\frac1{1+\omega}C(x-v^{(r-1)})$, $r=1,\cdots,R$. We first prove by induction that for any $0\le r\le R$ it holds that
\begin{align}
    \mathbb{E}[v^{(r)}]=&\left[1-\left(\frac{\omega}{1+\omega}\right)^r\right]x\quad \text{and}\quad
    \mathbb{E}[\|v^{(r)}-x\|^2]\le\left(\frac{\omega}{1+\omega}\right)^r\|x\|^2\label{induction-variance}
\end{align}
It is obvious that
\eqref{induction-variance} holds for $r=0$. Assume 
and \eqref{induction-variance} holds for $r-1$ $(1\le r\le R)$, we have
\begin{align}
    \mathbb{E}[v^{(r)}]=&\mathbb{E}\left[v^{(r-1)}+\frac1{1+\omega}C(x-v^{(r-1)})\right]\nonumber\\
    =&\mathbb{E}[v^{(r-1)}]+\frac1{1+\omega}\mathbb{E}\left[\mathbb{E}[C(x-v^{(r-1)})\mid v^{(r-1)}]\right]\nonumber\\
    =&\mathbb{E}[v^{(r-1)}]+\frac1{1+\omega}\mathbb{E}[x-v^{(r-1)}],
\end{align}
thus by the induction hypothesis we have
\begin{align}
     \EE[v^{(r)}]=&\frac\omega{1+\omega}\cdot \left[1-\left(\frac\omega{1+\omega}\right)^{r-1}\right]x+\frac1{1+\omega}x=\left[1-\left(\frac\omega{1+\omega}\right)^r\right]x.
\end{align}
On the other hand, it holds that
\begin{align}
    \mathbb{E}[\|v^{(r)}-x\|^2]=&\mathbb{E}\left[\left\|v^{(r-1)}+\frac1{1+\omega}C(x-v^{(r-1)})-x\right\|^2\right]\nonumber\\
    =&\mathbb{E}\left[\left\|\frac1{1+\omega}\left(v^{(r-1)}-x-C(v^{(r-1)}-x)\right)+\frac\omega{1+\omega}(v^{(r-1)}-x)\right\|^2\right]\nonumber\\
    \le&\frac1{(1+\omega)^2}\cdot\omega\mathbb{E}[\|v^{(r-1)}-x\|^2]+\frac{2\omega}{(1+\omega)^2}\mathbb{E}[\langle v^{(r-1)}-x-C(v^{(r-1)}-x),v^{(r-1)}-x\rangle]\\
    &+\frac{\omega^2}{(1+\omega)^2}\mathbb{E}[\|v^{(r-1)}-x\|^2].
\end{align}
Using the unbiasedness of $C$ and the induction hypothesis we have
\begin{align}
    \EE[\|v^{(r)}-x\|^2]\le&\frac{\omega}{1+\omega}\EE[\|v^{(r-1)}-x\|^2]\le\frac{\omega}{1+\omega}\left(\frac{\omega}{1+\omega}\right)^{r-1}\|x\|^2=\left(\frac{\omega}{1+\omega}\right)^r\|x\|^2.
\end{align}
Now we've proved \eqref{induction-variance}. Consequently, by the definition of MSC$(C,R)$, it holds that
\begin{align}
    \mathbb{E}[\mathrm{MSC}(x,C,R)]=&\frac1{1-(\frac\omega{1+\omega})^R}\mathbb{E}[v^{(R)}]=\frac1{1-(\frac\omega{1+\omega})^R}\cdot\left[1-\left(\frac\omega{1+\omega}\right)^R\right]x=x\nonumber
\end{align}
and 
\begin{align}
    &\EE[\|\mathrm{MSC}(x,C,R)-x\|^2]=\EE\left[\left\|\frac{1}{1-\alpha}(v^{(R)}-x)+\left(\frac{1}{1-\alpha}-1\right)x\right\|^2\right]\\
    \le&\left(\frac{1}{1-\alpha}\right)^2\alpha\|x\|^2+\left(\frac{1}{1-\alpha}-1\right)^2\|x\|^2+\frac{2}{1-\alpha}\left(\frac{1}{1-\alpha}-1\right)\EE[\langle v^{(R)}-x,x\rangle]\\
    =&\left(\frac{1}{1-\alpha}\right)^2\alpha\|x\|^2-\left(\frac{1}{1-\alpha}-1\right)^2\|x\|^2=\frac{\alpha}{1-\alpha}\|x\|^2\le(1+\omega)\left(\frac{\omega}{1+\omega}\right)^R\|x\|^2,
\end{align}
where $\alpha:=(\omega/(1+\omega))^R$, which completes the proof.
\end{itemize}
\hfill$\square$

\subsection{Useful Notations and Lemmas}
We define the following notations for  convenience:
\begin{align*}
    \tilde{\omega}:=&\begin{cases}(1+\omega)\left(\frac{\omega}{1+\omega}\right)^R,& \mbox{in the unbiased case},\\
    (1-\delta)^R, & \mbox{in the contractive case}.\end{cases}\\
    \tilde{\sigma}^2:=&\frac{\sigma^2}{R},\\
    T:=&KR \mbox{ be the total number of compressed communication rounds/gradient queries},\\
    g^{(k)}:=&\frac1n\sum_{i=1}^ng_i^{(k)},\\
    w^{(k)}:=&(\lambda+1)z^{(k)}-\lambda x^{(k)}\overset{(a)}{=}\frac{(\lambda+1)p-\gamma_k}{p-\gamma_k}z^{(k)}-\frac{\lambda p}{p-\gamma_k}y^{(k)},\\
    A_k:=&\frac{\lambda\gamma_k}{(\lambda+1)p-\gamma_k},\\
    B_k:=&\frac{[(\lambda+1)-\lambda\gamma_k]\eta}{p\gamma_k},\\
    C_k:=&1-\frac{\gamma_k}{(\lambda+1)p-\lambda\gamma_k}.
\end{align*}
where equality (a) is due to the definition of $y^{(k)}$ in Algorithm \ref{alg:NEOLITHIC}, and we may assume $p>\gamma_k$ so that 
\begin{align}
    0\le A_k<1,\quad B_k>0,\quad 0<C_k<1.\label{eq-ass-abc}
\end{align}

By applying the update rules, the following relations hold:
\begin{align}
    w^{(k+1)}=&(1-A_k)w^{(k)}+A_ky^{(k)}-B_k\hat{g}^{(k)},\label{eq-wabc01}\\
    y^{(k)}-C_kx^{(k)}=&(1-C_k)[(1-A_k)w^{(k)}+A_ky^{(k)}].\label{eq-wabc02}
\end{align}

With these notations at hand, we have the following useful lemmas for the convergence analysis.

\begin{lemma}\label{lm:newasp}
If functions $\{f_i\}_{i=1}^n$ are $L$-smooth, it holds that
\begin{align}
    \frac{1}{n}\sum_{i=1}^n\|\nabla f_i(x)\|^2\leq 2L (f(x)-f^\star)+2LG^\star, \quad\forall x\in\RR^d
\end{align}
where $G^\star\triangleq f^\star - \frac{1}{n}\sum_{i=1}^n f_i^\star$.
\end{lemma}
\begin{proof}
    By $L$-smoothness property, we have
    \begin{align}
        \|\nabla f_i(x)\|^2\le 2L(f_i(x)-f_i^*),\label{eq:pflm-newasp}
    \end{align}
    for $i=1,2,\cdots,n$. Averaging \eqref{eq:pflm-newasp} yields
    \begin{align}
        \frac{1}{n}\sum_{i=1}^n\|\nabla f_i(x)\|^2\le\frac{2L}{n}\sum_{i=1}^n(f_i(x)-f_i^\star)=2L(f(x)-f^\star)+2LG^\star.
    \end{align}
\end{proof}

\begin{lemma}\label{lm-Egk}
	Under Assumptions \ref{asp:convex}, \ref{asp:gd-noise} and \ref{ass:unbiased}, it holds that 
	\begin{align}\label{eq-important-inequalities-sto}
		\mathbb{E}[\|\hat{g}^{(k)}\|^2] \le 4\tilde{\omega} L\mathbb{E}[f(x^{(k+1)}) - f^\star]  +  \left(1+\frac{4\tilde{\omega}L\eta}{p}\right)\|\nabla f(y^{(k)})\|^2 + {\color{black}4 \tilde{\omega} LG^\star} + \left(\tilde{\omega} + \frac{1}{n}\right)\tilde{\sigma}^2.
	\end{align}
\end{lemma}

\begin{proof} By Assumptions \ref{asp:gd-noise}, \ref{ass:unbiased} and Lemma \ref{lm:MSC}, we have
\begin{align}
    \mathbb{E}[\|\hat{g}^{(k)}\|^2]=&\mathbb{E}\left[\left\|\frac{1}{n}\sum_{i=1}^n\Big(\mathrm{MSC}(g_i^{(k)},C,R)-g_i^{(k)}+g_i^{(k)}-\nabla f_i(y^{(k)})+\nabla f_i(y^{(k)})\Big)\right\|^2\right]\nonumber\\
    =& \mathbb{E}\left[\left\|\frac{1}{n}\Big(\mathrm{MSC}(g_i^{(k)},C,R)-g_i^{(k)}\Big)\right\|^2\right]+\mathbb{E}\left[\left\|\frac{1}{n}\sum_{i=1}^n\Big(g_i^{(k)}-\nabla f_i(y^{(k)})\Big)\right\|^2\right]+\|\nabla f(y^{(k)})\|^2\nonumber\\
    \le& \frac{1}{n}\sum_{i=1}^n\mathbb{E}[\|\mathrm{MSC}(g_i^{(k)},C,R)-g_i^{(k)}\|^2]+\frac{1}{n^2}\sum_{i=1}^n\mathbb{E}[\|g_i^{(k)}-\nabla f_i(y^{(k)})\|^2]+\|\nabla f(y^{(k)})\|^2,\nonumber
\end{align}
Note that
\begin{align}
    \mathbb{E}[\|g_i^{(k)}-\nabla f_i(y^{(k)})\|^2]\le\tilde{\sigma}^2,\nonumber
\end{align}
and that
\begin{align}
    &\mathbb{E}[\|\mathrm{MSC}(g_i^{(k)},C,R)-g_i^{(k)}\|^2]\le\tilde{\omega}\mathbb{E}[\|g_i^{(k)}\|^2]\nonumber\\
    \le&\tilde{\omega}\mathbb{E}[\|g_i^{(k)}-\nabla f_i(y^{(k)})\|^2]+2\tilde{\omega}\|\nabla f_i(y^{(k)})-\nabla f_i(x^\star)\|^2+2\tilde{\omega}\|\nabla f_i(x^\star)\|^2\nonumber\\
    \le&2\tilde{\omega}\|\nabla f_i(y^{(k)})-\nabla f_i(x^\star)\|^2+2\tilde{\omega}\|\nabla f_i(x^\star)\|^2+\tilde{\omega}\tilde{\sigma}^2,\quad\forall\, 1\leq i\leq n,\nonumber
\end{align}
we obtain
\begin{align}
    \mathbb{E}[\|g^{(k)}\|^2]\le\frac{2\tilde{\omega}}{n}\sum_{i=1}^n\|\nabla f_i(y^{(k)})-\nabla f_i(x^\star)\|^2+\frac{2\tilde{\omega}}{n}\sum_{i=1}^n\|\nabla f_i(x^\star)\|^2+\|\nabla f(y^{(k)})\|^2+\left(\tilde{\omega}+\frac{1}{n}\right)\tilde{\sigma}^2.\label{zn2398nz08-sto}
\end{align}
	Since $f_i(x)$ is convex and $L$-smooth, we have 
	\begin{align}
		f_i(x^\star) + \langle \nabla f_i(x^\star), y^{(k)} - x^\star \rangle + \frac{1}{2L}\|\nabla f_i(x^\star) - \nabla f_i(y^{(k)})\|^2 \le f_i(y^{(k)}) \nonumber
	\end{align}
	which implies that 
	\begin{align}
		\|\nabla f_i(x^\star) - \nabla f_i(y^{(k)})\|^2  \le 2L \Big( f_i(y^{(k)}) - f_i(x^\star) -  \langle \nabla f_i(x^\star), y^{(k)} - x^\star \rangle  \Big). \nonumber
	\end{align}
	This together with $\frac{1}{n}\sum_{i=1}^n\nabla f_i(x^\star) = 0$ leads to
	\begin{align}\label{znzn239887-sto}
		\frac{1}{n}\sum_{i=1}^n \|\nabla f_i(y^{(k)}) - \nabla f_i(x^\star) \|^2 \le 2L (f(y^{(k)}) - f^\star). 
	\end{align}
	Substituting \eqref{znzn239887-sto} to \eqref{zn2398nz08-sto} and applying Lemma \ref{lm:newasp}, we obtain
	\begin{align}
	    \mathbb{E}[\|g^{(k)}\|^2] \le 4\tilde{\omega} L(f(y^{(k)}) - f^\star)  +  \|\nabla f(y^{(k)})\|^2 + {\color{black}4 \tilde{\omega} LG^\star} + \left(\tilde{\omega} + \frac{1}{n}\right)\tilde{\sigma}^2.\label{eq-sdfadfa}
	\end{align}
	By the update rule $x^{(k+1)}=y^{(k)}-\eta g^{(k)}/p$ and Assumption \ref{asp:convex}, we have 
	\begin{align}
	    f(y^{(k)})-f^\star\le&\mathbb{E}[f(x^{(k+1)})-f^\star-\langle\nabla f(y^{(k)}),x^{(k+1)}-y^{(k)}\rangle]\nonumber\\
	    =&\mathbb{E}[f(x^{(k+1)})-f^\star]+\frac{\eta}{p}\mathbb{E}[\langle\nabla f(y^{(k)}),g^{(k)}\rangle]\nonumber\\
	    =&\mathbb{E}[f(x^{(k+1)})-f^\star]+\frac{\eta}{p}\|\nabla f(y^{(k)})\|^2.\label{eq-aswesfdsfd}
	\end{align}
	Combining \eqref{eq-sdfadfa} and \eqref{eq-aswesfdsfd}, we achieve the result in \eqref{eq-important-inequalities-sto}.
\end{proof}

\begin{lemma}\label{lm-strongly-convex-1}
For any $x\in\mathbb{R}^d$, it holds that
\begin{align}
    \langle\mathbb{E}[\hat{g}^{(k)}],(1-A_k)w^{(k)}+A_ky^{(k)}-x\rangle=&-\frac{1}{2B_k}\mathbb{E}[\|w^{(k+1)}-x\|^2]+\frac{1-A_k}{2B_k}\|w^{(k)}-x\|^2+\frac{A_k}{2B_k}\|y^{(k)}-x\|^2\nonumber\\
    &-\frac{A_k(1-A_k)}{2B_k}\|w^{(k)}-y^{(k)}\|^2+\frac{1}{2}B_k\mathbb{E}[\|\hat{g}^{(k)}\|^2].\label{eq-lm-strongly-convex-1}
\end{align}
\end{lemma}
\begin{proof}
Recall \eqref{eq-wabc01} we have 
\begin{align}
    w^{(k+1)}=(1-A_k)w^{(k)}+A_ky^{(k)}-B_k\hat{g}^{(k)},
\end{align}
thus
\begin{align}
    \mathbb{E}[\|w^{(k+1)}-x\|^2]=&\mathbb{E}[\|(1-A_k)(w^{(k)}-x)+A_k(y^{(k)}-x)-B_k\hat{g}^{(k)}\|^2]\nonumber\\
    =&(1-A_k)\|w^{(k)}-x\|^2+A_k\|y^{(k)}-x\|^2-A_k(1-A_k)\|w^{(k)}-y^{(k)}\|^2+B_k^2\mathbb{E}[\|\hat{g}^{(k)}\|^2]\nonumber\\
    &-2\mathbb{E}[\langle(1-A_k)w^{(k)}+A_ky^{(k)}-x,B_k\hat{g}^{(k)}\rangle]\nonumber\\
    =&(1-A_k)\|w^{(k)}-x\|^2+A_k\|y^{(k)}-x\|^2-A_k(1-A_k)\|w^{(k)}-y^{(k)}\|^2+B_k^2\mathbb{E}[\|\hat{g}^{(k)}\|^2]\nonumber\\
    &-2B_k\langle\mathbb{E}[\hat{g}^{(k)}],(1-A_k)w^{(k)}+A_ky^{(k)}-x\rangle,\nonumber
\end{align}
which is equivalent to \eqref{eq-lm-strongly-convex-1}.
\end{proof}

\begin{lemma}\label{lm-public-base}
Under Assumptions \ref{asp:convex} and \ref{asp:gd-noise}, if \eqref{eq-ass-abc} holds, we have
\begin{align}
&\mathbb{E}[f(x^{(k+1)})-f(x)]+\frac{1-C_k}{2B_k}\mathbb{E}[\|w^{(k+1)}-x\|^2]\nonumber\\
\le&C_k[f(x^{(k)})-f(x)]+\frac{(1-A_k)(1-C_k)}{2B_k}\|w^{(k)}-x\|^2-\left(\frac{\mu(1-C_k)}{2}-\frac{A_k(1-C_k)}{2B_k}\right)\|y^{(k)}-x\|^2\nonumber\\
&+\left(\frac{L\eta^2}{2p^2}+\frac{B_k(1-C_k)}{2}\right)\mathbb{E}[\|\hat{g}^{(k)}\|^2]-(1-C_k)\mathbb{E}[\langle \hat{g}^{(k)}-g^{(k)},(1-A_k)w^{(k)}+A_ky^{(k)}-x\rangle]\nonumber\\
&-\frac{\eta}{p}\mathbb{E}[\langle\nabla f(y^{(k)}),\hat{g}^{(k)}-g^{(k)}\rangle]-\frac{\eta}{p}\|\nabla f(y^{(k)})\|^2.\label{eq-lm-public-0}
\end{align}

\end{lemma}
\begin{proof}
Using Assumption \ref{asp:convex}, we have
\begin{align}
\mathbb{E}[f(x^{(k+1)})]\le&\mathbb{E}\left[f(y^{(k)})+\langle\nabla f(y^{(k)}),x^{(k+1)}-y^{(k)}\rangle+\frac{L}{2}\|x^{(k+1)}-y^{(k)}\|^2\right]\nonumber\\
=&f(y^{(k)})+\mathbb{E}\left[\left\langle\nabla f(y^{(k)}),-\frac{\eta}{p}\hat{g}^{(k)}\right\rangle\right]+\frac{L}{2}\mathbb{E}\left[\left\|-\frac{\eta}{p}\hat{g}^{(k)}\right\|^2\right]\nonumber\\
=&f(y^{(k)})-\frac{\eta}{p}\|\nabla f(y^{(k)})\|^2+\frac{L\eta^2}{2p^2}\mathbb{E}[\|\hat{g}^{(k)}\|^2]-\frac{\eta}{p}\mathbb{E}[\langle\nabla f(y^{(k)}),\hat{g}^{(k)}-g^{(k)}\rangle],\label{eq-lm-strong-convex-123}
\end{align}
and that
\begin{align}
    f(y^{(k)})\le f(u)-\langle\nabla f(y^{(k)}),u-y^{(k)}\rangle-\frac{\mu}{2}\|u-y^{(k)}\|^2,\quad\forall u\in\RR^d.\label{eq-lm-strong-convex-124}
\end{align}
By adding \eqref{eq-lm-strong-convex-123} , $C_k$ times \eqref{eq-lm-strong-convex-124} (where $u=x^{(k)}$) and $(1-C_k)$ times \eqref{eq-lm-strong-convex-124} (where $u=x$), we obtain
\begin{align}
    &\EE[f(x^{(k+1)})-f(x)]\\
    \le&C_k(f(x^{(k)})-f(x))-\frac{\mu C_k}{2}\|x^{(k)}-y^{(k)}\|^2-\frac{\mu(1-C_k)}{2}\|x-y^{(k)}\|^2-\frac{\eta}{p}\|\nabla f(y^{(k)})\|^2\\
    &-\langle\nabla f(y^{(k)}),C_kx^{(k)}+(1-C_k)x-y^{(k)}\rangle-\frac{\eta}{p}\mathbb{E}\left[\left\langle\nabla f(y^{(k)}),\hat{g}^{(k)}-g^{(k)}\right\rangle\right]+\frac{L\eta^2}{2p^2}\mathbb{E}[\|\hat{g}^{(k)}\|^2]\label{eq-lm-strong-convex-125}
\end{align}
Applying \eqref{eq-wabc02} and the unbiasedness of $g^{(k)}$, we have 
\begin{align}
    &-\langle\nabla f(y^{(k)}),C_kx^{(k)}+(1-C_k)x-y^{(k)}\rangle\\
    =&(1-C_k)\langle\nabla f(y^{(k)}),(1-A_k)w^{(k)}+A_ky^{(k)}-x\rangle\\
    =&(1-C_k)\langle\EE[\hat{g}^{(k)}],(1-A_k)w^{(k)}+A_ky^{(k)}-x\rangle-(1-C_k)\EE[\langle\hat{g}^{(k)}-g^{(k)},(1-A_k)w^{(k)}+A_ky^{(k)}-x\rangle].\label{eq-lm-strong-convex-126}
\end{align}
Combining \eqref{eq-lm-strong-convex-125}\eqref{eq-lm-strong-convex-126} and applying Lemma \ref{lm-strongly-convex-1} leads to \eqref{eq-lm-public-0}.
\end{proof}

\begin{lemma}\label{lm-strongly-convex-2}
Under Assumptions \ref{asp:convex}, \ref{asp:gd-noise} and \ref{ass:unbiased}, if \eqref{eq-ass-abc} holds, we have
\begin{align}
&\left(1-2\tilde{\omega}L\Big(\frac{L\eta^2}{p^2}+B_k(1-C_k)\Big)\right)\mathbb{E}[f(x^{(k+1)})-f^\star]+\frac{1-C_k}{2B_k}\mathbb{E}[\|w^{(k+1)}-x^\star\|^2]\nonumber\\
\le&C_k(f(x^{(k)})-f^\star)+\frac{(1-A_k)(1-C_k)}{2B_k}\|w^{(k)}-x^\star\|^2-\left(\frac{\mu}{2}-\frac{A_k}{2B_k}\right)\|y^{(k)}-x^\star\|^2\nonumber\\
&+\left({\color{black}4\tilde{\omega}LG^\star}+\Big(\tilde{\omega}+\frac{1}{n}\Big)\tilde{\sigma}^2\right)\left(\frac{L\eta^2}{2p^2}+\frac{B_k(1-C_k)}{2}\right)-\left(\frac{\eta}{p}-\Big(\frac{L\eta^2}{2p^2}+\frac{B_k(1-C_k)}{2}\Big)\Big(1+\frac{4\tilde{\omega}L\eta}{p}\Big)\right)\|\nabla f(y^{(k)})\|^2.\label{eq-lm-unbiased}
\end{align}
\end{lemma}
\begin{proof}
By Lemma \ref{lm-public-base}, substituting $x$ by $x^\star$ in \eqref{eq-lm-public-0}, noting that $\mathbb{E}[\hat{g}^{(k)}-g^{(k)}]=0$ and further using Lemma \ref{lm-Egk}, we obtain \eqref{eq-lm-unbiased} immediately.
\end{proof}

\begin{lemma}\label{lm-R-omega}
If $R\ge 4(1+\omega)\ln(4(1+\omega))$, it holds that
\begin{align}
    R\left(\frac{\omega}{1+\omega}\right)^{R/2}\le1.\label{eq-R-tilde-omega}
\end{align}
Similarly, by letting $\delta=(1+\omega)^{-1}$, we have $R(1-\delta)^{R/2}\le1$ when $R\ge\max\{4,4/\delta\cdot\ln(2/\delta)\}$.
\end{lemma}

\begin{proof}
It is easy to  see that $\frac{\ln(x)}{x}$ monotonically increases in $(0,e]$ and monotonically decreases in $[e,+\infty)$. If $\omega\le1$, since $R\ge4$ we have
\begin{align}
\frac{\ln R}{R}\le\frac{\ln4}{4}=\frac{\ln2}{2}\le\frac{\ln(1+1/\omega)}{2}.\nonumber
\end{align}
If $\omega>1$, noting $4(1+\omega)\ln(4(1+\omega))\ge e$, we have
\begin{align}
    \frac{\ln R}{R}\le&\frac{\ln(4(1+\omega)\ln(4(1+\omega))}{4(1+\omega)\ln(4(1+\omega))}=\frac{\ln(4(1+\omega))+\ln(\ln(4(1+\omega)))}{4(1+\omega)\ln(4(1+\omega))}\le\frac{1}{2(1+\omega)}\le\frac{\ln(1+1/\omega)}{2}
\end{align}
where we use $\ln(1+\omega^{-1})\geq\omega^{-1}/(1+\omega^{-1})$ in the last inequality.
Thus under both cases we have
\begin{align}
R\left(\frac{\omega}{1+\omega}\right)^R=&\left(\frac{\omega}{1+\omega}\right)^{R-\ln R/\ln(1+1/\omega)}\le\left(\frac{\omega}{1+\omega}\right)^{R-\ln(1+1/\omega)/2\cdot R/\ln(1+1/\omega)}=\left(\frac{\omega}{1+\omega}\right)^{R/2},
\end{align}
\ie, \eqref{eq-R-tilde-omega} holds.
\end{proof}

\begin{lemma}\label{lm-contractive-important-terms}
Under Assumptions \ref{asp:gd-noise} and \ref{ass:contract},  denote $\tilde{\delta}=1-(1-\delta)^R$.
\begin{itemize}
\item Under Assumption \ref{asp:nonconvex}, we have
    \begin{equation}
        \EE\left[\left\|\frac{1}{n}\sum_{i=1}^n\left(\hat{g}_i^{(k)}-g_i^{(k)}\right)\right\|^2\right]\le(1-\tilde{\delta})\left(2L(f(y^{(k)})-f^\star)+2LG^\star+\tilde{\sigma}^2\right).\label{eq-newasp-1}
    \end{equation}
    \item Under Assumption \ref{asp:convex}, letting $(1-\tilde{\delta})L\eta\le p/4$,
    we have
    \begin{equation}\label{eq-l14c01}
        \EE\left[\left\|\frac{1}{n}\sum_{i=1}^n\left(\hat{g}_i^{(k)}-g_i^{(k)}\right)\right\|^2\right]\le(1-\tilde{\delta})\left(8L\EE[f(x^{(k+1)})-f^\star]+\frac{12L\eta}{p}\|\nabla f(y^{(k)})\|^2+8LG^\star+2\tilde{\sigma}^2\right).
    \end{equation}
    and further 
    \begin{align}
    &\EE[\|\hat{g}^{(k)}-\nabla f(y^{(k)})\|^2]\le(1-\tilde{\delta})\left(16L\EE[f(x^{(k+1)})-f^\star]+\frac{24L\eta}{p}\|\nabla f(y^{(k)})\|^2+16LG^\star+4\tilde{\sigma}^2\right)+\frac2n\tilde{\sigma}^2,\label{eq-l14c02}\\
    &\EE[\|\hat{g}^{(k)}\|^2]\le(1-\tilde{\delta})\left(16L\EE[f(x^{(k+1)})-f^\star]+\frac{24L\eta}{p}\|\nabla f(y^{(k)})\|^2+16LG^\star+4\tilde{\sigma}^2\right)\nonumber\\
    &\quad\quad\quad\quad\quad\quad+2\|\nabla f(y^{(k)})\|^2+\frac2n\tilde{\sigma}^2.\label{eq-l14c03}
    \end{align}
\end{itemize}
\end{lemma}
\begin{proof}
Recall that  $g_i^{(k)}=\frac{1}{R}\sum_{r=1}^R\nabla F(y^{(k)},\xi_i^{(k,r)})$. Using Young's inequality and Lemma \ref{lm:MSC}, we have
\begin{align}
&\mathbb{E}\left[\left\|\frac{1}{n}\sum_{i=1}^n\left(\hat{g}_i^{(k)}-g_i^{(k)}\right)\right\|^2\right]\leq \frac{1}{n}\sum_{i=1}^n\mathbb{E}\left[\left\|\mathrm{MSC}({g}_i^{(k)},C,R)-g_i^{(k)}\right\|^2\right]\le\frac{1-\tilde{\delta}}{n}\sum_{i=1}^n\mathbb{E}[\|{g}_i^{(k)}\|^2].\label{eqn:zxcvbnm01}
\end{align}
By Assumption \ref{asp:gd-noise}, we have 
\begin{align}
    \frac1n\sum_{i=1}^n\EE[\|g_i^{(k)}\|^2]\le\frac{1}{n}\sum_{i=1}^n\EE[\|\nabla f_i(y^{(k)})\|^2]+\tilde{\sigma}^2.\label{eqn:zxcvbnm02}
\end{align}
Furthermore, by applying Lemma \ref{lm:newasp}, 
\begin{itemize}
    \item under Assumption \ref{asp:nonconvex}, we have 
    \begin{align}
   \frac{1}{n}\sum_{i=1}^n\EE[\|\nabla f_i(y^{(k)})\|^2]\leq 2L(f(y^{(k)})-f^\star)+2LG^\star.\label{eqn:zxcvbnm03}
    \end{align}
    Combing \eqref{eqn:zxcvbnm01}\eqref{eqn:zxcvbnm02} and \eqref{eqn:zxcvbnm03} leads to \eqref{eq-newasp-1}. 
\item under Assumption \ref{asp:convex}, we have
\begin{align}
    \frac1n\sum_{i=1}^n\EE[\|\nabla f_i(y^{(k)})\|^2]\le&\frac2n\sum_{i=1}^n\EE[\|\nabla f_i(y^{(k)})-\nabla f_i(x^\star)\|^2]+\frac2n\sum_{i=1}^n\EE[\|\nabla f_i(x^\star)\|^2]\\
    \le&\frac2n\sum_{i=1}^n\EE[\|\nabla f_i(y^{(k)})-\nabla f_i(x^\star)\|^2]+4LG^\star.\label{eqn:zxcvbnm04}
\end{align}
By convexity and $L$-smooth, we have
\begin{align}
    \frac1n\sum_{i=1}^n\|\nabla f_i(y^{(k)})-\nabla f_i(x^\star)\|^2\le 2L(f(y^{(k)})-f^\star)\le2L\EE[f(x^{(k+1)})-f^\star-\langle\nabla f(y^{(k)}),x^{(k+1)}-y^{(k)}\rangle],\label{eqn:zxcvbnm05}
\end{align}
wherein the inner product can be further specialized by the definition of $x^{(k+1)}$ in NEOLITHIC as
\begin{align}
    -\EE[\langle\nabla f(y^{(k)}),x^{(k+1)}-y^{(k)}\rangle]=&\frac{\eta}{p}\EE[\langle\nabla f(y^{(k)}),\hat{g}^{(k)}\rangle]\\
    =&\frac{\eta}{p}\EE[\langle\nabla f(y^{(k)}),\hat{g}^{(k)}-g^{(k)}\rangle]+\frac{\eta}{p}\|\nabla f(y^{(k)})\|^2\\
    \le&\frac{3\eta}{2p}\|\nabla f(y^{(k)})\|^2+\frac{\eta}{2p}\EE[\|\hat{g}^{(k)}-g^{(k)}\|^2].\label{eqn:zxcvbnm06}
\end{align}
Combining \eqref{eqn:zxcvbnm01}\eqref{eqn:zxcvbnm02}\eqref{eqn:zxcvbnm04}\eqref{eqn:zxcvbnm05}\eqref{eqn:zxcvbnm06}, we know if $(1-\tilde{\delta})L\eta\le p/4$, \ie, $2(1-\tilde{\delta})L\eta/p\le1/2$, we obtain \eqref{eq-l14c01}. Further using
\begin{align}
    \EE[\|\hat{g}^{(k)}-\nabla f(y^{(k)})\|^2]\le&2\EE[\|\hat{g}^{(k)}-g^{(k)}\|^2]+2\EE[\|g^{(k)}-\nabla f(y^{(k)})\|^2],\\
    \EE[\|\hat{g}^{(k)}\|^2]\le&2\EE[\hat{g}^{(k)}-g^{(k)}\|^2]+2\EE[\|g^{(k)}-\nabla f(y^{(k)})\|^2]+2\|\nabla f(y^{(k)})\|^2,
\end{align}
we obtain \eqref{eq-l14c02} and \eqref{eq-l14c03}.
\end{itemize}
\end{proof}

\begin{lemma}\label{lm-contractive-strong}
Under Assumptions \ref{asp:convex}, \ref{asp:gd-noise}, and \ref{ass:contract}, it holds that
\begin{align}
    &(1-4\tilde{\omega}LN_k)\mathbb{E}[f(x^{(k+1)})-f^\star]+\frac{1-C_k}{2B_k}\mathbb{E}[\|w^{(k+1)}-x^\star\|^2]\nonumber\\
    \le&C_k(f(x^{(k)})-f^\star)+\frac{(1-A_k)(1-C_k)(1+M_k)}{2B_kM_k}\|w^{(k)}-x^\star\|^2-\frac{(1-C_k)}{2}\cdot\left(\mu-\frac{A_k(1+M_k)}{B_kM_k}\right)\|y^{(k)}-x^\star\|^2\nonumber\\
    &-\left(\frac{p\eta-2L\eta^2}{2p^2}-B_k(1-C_k)-\frac{6\tilde{\omega}L\eta N_k}{p}\right)\|\nabla f(y^{(k)})\|^2+4\tilde{\omega}N_kLG^\star+\left(\frac{L\eta^2}{np^2}+\frac{B_k(1-C_k)}{n}+\tilde{\omega}N_k\right)\tilde{\sigma}^2,\label{eq-lm-contractive-strong}
\end{align}
 if $\tilde{\omega}L\eta\le p/4$ and \eqref{eq-ass-abc} holds, where $\{M_k\}$ is an arbitrary positive sequence, and $N_k:=2L\eta^2/p^2+(2+M_k)B_k(1-C_k)+\eta/p$.
\end{lemma}
\begin{proof}
Substituting $x$ by $x^\star$ in Lemma \ref{lm-public-base}, noting that
\begin{align*}
    &-\mathbb{E}[\langle \hat{g}^{(k)}-g^{(k)},(1-A_k)w^{(k)}+A_ky^{(k)}-x^\star\rangle]\nonumber\\
    \le&\frac{B_kM_k}{2}\mathbb{E}[\|\hat{g}^{(k)}-g^{(k)}\|^2]+\frac{1-A_k}{2B_kM_k}\|w^{(k)}-x^\star\|^2+\frac{A_k}{2B_kM_k}\|y^{(k)}-x^\star\|^2
\end{align*}
and 
\begin{align*}
    -\mathbb{E}[\langle\nabla f(y^{(k)}),\hat{g}^{(k)}-g^{(k)}\rangle]\le\frac{1}{2}\|\nabla f(y^{(k)})\|^2+\frac{1}{2}\mathbb{E}[\|\hat{g}^{(k)}-g^{(k)}\|^2],
\end{align*}
and using Lemma \ref{lm-contractive-important-terms}, we obtain \eqref{eq-lm-contractive-strong}.
\end{proof}

\subsection{Non-Convex Case}
When $\gamma_k\equiv \gamma=p=1$, the algorithm reduces to
\begin{align}
g_i^{(k)}=\frac{1}{R}\sum_{r=0}^{R-1}\nabla F (x^{(k)};\xi_i^{(k,r)}),\quad 
    \hat{g}_i^{(k)}=\mathrm{MSC}\left(g_i^{(k)},C_i,R\right),\quad 
    x^{(k+1)}=&x^{(k)}-\frac{\eta}{n}\sum_{i=1}^n\hat{g}_i^{(k)},
\end{align}
and $y^{(k)}\equiv x^{(k)}$, for any $k\geq 0$.
The convergence analysis follows the proof for distributed SGD with inexact consensus.

\subsubsection{Proof of Theorem \ref{thm:upper-nc-contractive}}\label{thm:upper-nc}
    Using Assumption \ref{ass-eq-L-smooth}, we have
    \begin{align}
        f(x^{(k+1)})\leq &f(x^{(k)})+\langle \nabla f(x^{(k)}),x^{(k+1)}-x^{(k)}\rangle+\frac{L}{2}\|x^{(k+1)}-x^{(k)}\|^2\\
        \leq &f(x^{(k)})-\eta\left\langle \nabla f(x^{(k)}),\frac{1}{n}\sum_{i=1}^{n}\hat{g}_i^{(k)}\right\rangle+\frac{\eta^2 L}{2}\left\|\frac{1}{n}\sum_{i=1}^{n}\hat{g}_i^{(k)}\right\|^2\\
        = &f(x^{(k)})-\eta\left\langle \nabla f(x^{(k)}),\frac{1}{n}\sum_{i=1}^{n}{g}_i^{(k)}\right\rangle-\eta\left\langle \nabla f(x^{(k)}),\frac{1}{n}\sum_{i=1}^{n}\left(\hat{g}_i^{(k)}-g_i^{(k)}\right)\right\rangle +\frac{\eta^2 L}{2}\left\|\frac{1}{n}\sum_{i=1}^{n}\hat{g}_i^{(k)}\right\|^2.\label{eqn:vnigner}
    \end{align}
    Using Young's inequality, we have 
    \begin{align}
        -\eta\left\langle \nabla f(x^{(k)}),\frac{1}{n}\sum_{i=1}^{n}\left(\hat{g}_i^{(k)}-g_i^{(k)}\right)\right\rangle\leq \frac{\eta}{2}\|\nabla f(x^{(k)})\|^2+\frac{\eta}{2}\left\|\frac{1}{n}\sum_{i=1}^{n}\left(\hat{g}_i^{(k)}-g_i^{(k)}\right)\right\|^2.\label{eqn:vcixngvdsgf}
    \end{align}
    Plugging \eqref{eqn:vcixngvdsgf} into \eqref{eqn:vnigner}, taking the expectation, we have 
    \begin{align}
        \EE[f(x^{(k+1)})]\leq &\EE[f(x^{(k)})]-\frac{\eta}{2}\EE[\|\nabla f(x^{(k)})\|^2]+\frac{\eta}{2}\EE\left[\left\|\frac{1}{n}\sum_{i=1}^{n}\left(\hat{g}_i^{(k)}-g_i^{(k)}\right)\right\|^2\right]+\frac{\eta^2 L}{2}\EE\left[\left\|\frac{1}{n}\sum_{i=1}^{n}\hat{g}_i^{(k)}\right\|^2\right].
    \end{align}
    Further applying Young's inequality and supposing $\eta \leq \frac{1}{4L}$, we have
    \begin{align}
        \EE[f(x^{(k+1)})]\leq &\EE[f(x^{(k)})]-\frac{\eta}{2}\EE[\|\nabla f(x^{(k)})\|^2]+\eta\EE\left[\left\|\frac{1}{n}\sum_{i=1}^{n}\left(\hat{g}_i^{(k)}-g_i^{(k)}\right)\right\|^2\right]+{\eta^2 L}\EE\left[\left\|\frac{1}{n}\sum_{i=1}^{n}{g}_i^{(k)}\right\|^2\right].
    \end{align}
    Noting that  $\frac{1}{n}\sum_{i=1}^{n}{g}_i^{(k)}$ is an unbiased estimator of $\nabla f(x^{(k)})$, we reach
    \begin{align}
        \EE[f(x^{(k+1)})]\leq &\EE[f(x^{(k)})]-\frac{\eta(1-2\eta L)}{2}\EE[\|\nabla f(x^{(k)})\|^2]+\eta\EE\left[\left\|\frac{1}{n}\sum_{i=1}^{n}\left(\hat{g}_i^{(k)}-g_i^{(k)}\right)\right\|^2\right]+\frac{\eta^2 L\tilde \sigma^2}{n}\\
        \leq &\EE[f(x^{(k)})]-\frac{\eta}{4}\EE[\|\nabla f(x^{(k)})\|^2]+\eta\EE\left[\left\|\frac{1}{n}\sum_{i=1}^{n}\left(\hat{g}_i^{(k)}-g_i^{(k)}\right)\right\|^2\right]+\frac{\eta^2 L\tilde \sigma^2}{n}.\label{eqn:vnisdngsd}
    \end{align}
    Now, using Lemma \ref{lm-contractive-important-terms} (note that $y^{(k)}$ is equivalent to $x^{(k)}$ in this scenario) in \eqref{eqn:vnisdngsd}, we obtain
    \begin{align}
        \EE[f(x^{(k+1)})]\leq &\EE[f(x^{(k)})]-\frac{\eta}{4}\EE[\|\nabla f(x^{(k)})\|^2]+2\eta L(1-\tilde \delta ) (\EE[f(x^{(k)})]-f^\star) \\
        &\quad +\eta(1-\tilde \delta) (2LG^\star+\tilde \sigma^2)+\frac{\eta^2 L\tilde \sigma^2}{n}.\label{eqn:vnisdnfa}
    \end{align}
    Rearranging \eqref{eqn:vnisdnfa} leads to 
    \begin{align}
        &(1+2\eta L(1-\tilde \delta ) )^{-k-1}\EE[f(x^{(k+1)})]\\
        \leq &(1+2\eta L(1-\tilde \delta ) )^{-k}(\EE[f(x^{(k)})]-f^\star)-(1+2\eta L(1-\tilde \delta ) )^{-k-1}\frac{\eta}{4}\EE[\|\nabla f(x^{(k)})\|^2] \\
        &\quad +\eta(1+2\eta L(1-\tilde \delta ) )^{-k-1}\left((1-\tilde \delta) (2LG^\star+\tilde \sigma^2)+\frac{\eta L\tilde \sigma^2}{n}\right).\label{eqn:vnisdngsd1}
    \end{align}
    Suppose $2\eta L(1-\tilde \delta)\leq 1/(K+1)$ so that $1\leq (1+2\eta L(1-\tilde \delta ) )^{k+1}\leq (1+1/(K+1) )^{K+1}\leq e$.
Then iterating \eqref{eqn:vnisdngsd1} over $k=0,1,\dots, K $, we obtain
    \begin{align}
       \frac{1}{K+1}\sum_{k=0}^{K}\EE[\|\nabla f(x^{(k)})\|^2] \leq &
       \frac{1}{K+1}\sum_{k=0}^{K}e(1+2\eta L(1-\tilde \delta ) )^{-k-1}\EE[\|\nabla f(x^{(k)})\|^2]\\
       \leq &\cO\left(\frac{\EE[f(x^{(0)})-f^\star]}{\eta (K+1)}+(1-\tilde{\delta})\left(LG^\star+\frac{\sigma^2}{R}\right)+\frac{\eta L\sigma^2}{Rn}\right)\\
        =&\cO\left(\frac{\Delta_f}{\eta (K+1)}+(1-\tilde{\delta})\left(LG^\star+\frac{\sigma^2}{R}\right)+\frac{\eta L\sigma^2}{Rn}\right).\label{eqn:gvnisgds}
    \end{align}
    Letting 
    \[
    \eta = \min\left\{\frac{1}{4L}, \left(\frac{Rn\Delta_f}{(K+1)L\sigma^2}\right)^{1/2}, \frac{1}{2(1-\tilde \delta)(K+1)L}\right\},
    \]
   following \eqref{eqn:gvnisgds} and noting $T=KR$, we obtain 
    \[
    \frac{1}{(K+1)}\sum_{k=0}^{K}\EE[\|\nabla f(x^{(k)})\|^2]=\cO\left(\left(\frac{L\Delta_f\sigma^2}{nT}\right)^{1/2}+\frac{RL\Delta_f}{T}+(1-\tilde{\delta})\left(L(G^\star+\Delta_f)+\frac{\sigma^2}{R}\right)\right).
    \]
    Finally, setting 
    \[
    R = \left\lceil\frac{1}{\delta}\max\{\ln((L(G^\star+\Delta_f)+{\sigma^2})\delta T/(L\Delta_f)),1\}\right\rceil,
    \]
    then 
    \begin{equation}
        1-\tilde{\delta}=(1-\delta)^R\leq \exp(-R\delta )\leq \frac{L\Delta_f  }{\delta T (L(G^\star+\Delta_f)+\sigma^2)}\leq \frac{RL\Delta_f }{T(L(G^\star+\Delta_f)+\sigma^2/R)}
    \end{equation} and we thus have 
        \[
    \frac{1}{(K+1)}\sum_{k=0}^{K}\EE[\|\nabla f(x^{(k)})\|^2]=\cO\left(\left(\frac{L\Delta_f\sigma^2}{nT}\right)^{1/2}+\frac{RL\Delta_f}{T}\right)=\tilde{\cO}\left(\left(\frac{L\Delta_f\sigma^2}{nT}\right)^{1/2}+\frac{L\Delta_f\ln(T)}{\delta T}\right),
    \]
    which implies the complexity in \eqref{eq-upper-bound-contractive-nc}.

    For unbiased compressors, we additionally have 
    \[
    \EE\left[\left\langle \nabla f(x^{(k)}),\frac{1}{n}\sum_{i=1}^{n}\left(\hat{g}_i^{(k)}-g_i^{(k)}\right)\right\rangle\right]=0.
    \]
    As a result, we have 
    \begin{align}
        \EE[f(x^{(k+1)})]\leq &\EE[f(x^{(k)})]-\eta\EE[\| \nabla f(x^{(k)})\|^2]+\frac{\eta^2 L}{2}\EE\left[\left\|\frac{1}{n}\sum_{i=1}^{n}\hat{g}_i^{(k)}\right\|^2\right]\\
        =&\EE[f(x^{(k)})]-\eta\EE[\| \nabla f(x^{(k)})\|^2]+\frac{\eta^2 L}{2}\EE\left[\left\|\frac{1}{n}\sum_{i=1}^{n}(\hat{g}_i^{(k)}-g_i^{(k)})\right\|^2\right]+\frac{\eta^2 L}{2}\EE\left[\left\|\frac{1}{n}\sum_{i=1}^{n}{g}_i^{(k)}\right\|^2\right]\\
        \leq &\EE[f(x^{(k)})]-\eta\left(1-\frac{\eta L }{2}\right)\EE[\| \nabla f(x^{(k)})\|^2]+\frac{\eta^2 L}{2}\EE\left[\left\|\frac{1}{n}\sum_{i=1}^{n}(\hat{g}_i^{(k)}-g_i^{(k)})\right\|^2\right]+\frac{\eta^2 L\tilde \sigma^2}{n}.\label{eqn:vnignerdsad}
    \end{align}
    Applying Lemma \ref{lm-R-omega}, following the similar analysis as of the contractive compressors, and setting 
    \[
    \eta=\min\left\{\frac{1}{L},\left(\frac{\Delta_f}{(K+1)L[\tilde{\omega}(LG^\star+\sigma^2/R)+\sigma^2/(Rn)]}\right)^{1/2},\frac{1}{\sqrt{\tilde{\omega}(K+1)}L}\right\}
    \]
    and
    \[
    R=\left\lceil(1+\omega)\max\{\ln(n^2L^2(G^\star+\Delta_f)^2(1+\omega)/\sigma^4),\ln((1+\omega)/n)\}\right\rceil
    \]
    we can reach
        \[
    \frac{1}{K+1}\sum_{k=0}^{K}\EE[\|\nabla f(x^{(k)})\|^2]=\tilde{\cO}\left(\left(\frac{L\Delta_f\sigma^2}{nT}\right)^{1/2}+\frac{(1+\omega)L\Delta_f}{T}\right),
    \]
    which implies the complexity in \eqref{eq-upper-bound-unbiased-nc}.\hfill$\square$

\subsection{Generally Convex Case}\label{app:upper-gconvex}

\subsubsection{Proof of the Unbiased Compressor Case in Theorem \ref{thm:upper-gc-contractive}}\label{app:upper-convex-unbiased}
For convenience, we denote
\begin{align}
\eta_0:=\frac{1}{\frac{2\sqrt{2}(K+2)^{3/2}\sqrt{4\tilde{\omega}LG^\star+(\tilde{\omega}+1/n)\tilde{\sigma}^2}}{27L\sqrt{\Delta_x}}+1},
\end{align}
and set 
\begin{align}
    R=\left\lceil(1+\omega)\max\left\{4\ln(4(1+\omega)),\ln\left((1+\omega)(n+2)+\frac{4(1+\omega)^2(nLG^\star)^2}{\sigma^4}+\frac{729(1+\omega)^2nL^2\Delta_x}{8\sigma^2}\right)\right\}\right\rceil,\label{eq-thm-acc-sto-gen-cov-R}
\end{align}
\begin{align}
    \eta=\frac{\eta_0}{L}=\frac{27\sqrt{\Delta_x}}{2\sqrt{2}(K+2)^{3/2}\sqrt{4\tilde{\omega}LG^\star+(\tilde{\omega}+1/n)\tilde{\sigma}^2}+27L\sqrt{\Delta_x}}.\label{eq-thm-acc-sto-gen-cov-eta}
\end{align}
Note that $R\ge4(1+\omega)\ln(4(1+\omega))$, we have
\begin{align}
\sqrt{R}\tilde{\omega}\le R\tilde{\omega}\overset{(a)}{\le}(1+\omega)\left(\frac{\omega}{1+\omega}\right)^{R/2}=\sqrt{(1+\omega)\tilde{\omega}}\overset{(b)}{\le}\min\left\{\frac{\sigma^2}{2nLG^\star},\frac{2\sqrt{2}\sqrt{\sigma^2/n}}{27L\sqrt{\Delta_x}}\right\},\label{eq-thm-acc-sto-gen-cov-5}
\end{align}
where (a) is due to Lemma \ref{lm-R-omega}, and (b) is due to $\tilde{\omega}\le\min\left\{\frac{\sigma^4}{4(1+\omega)(nLG^\star)^2},\frac{8\sigma^2}{729(1+\omega)nL^2\Delta_x}\right\}$.
Substituting $\eta$ by $\eta_0/L$, $p$ by 2, $\gamma_k$ by $6/(k+3)$ in Lemma \ref{lm-strongly-convex-2}, and letting $\lambda=0$, we obtain
\begin{align}
&\left(\frac{\eta_0(k+3)^2}{18L}-\frac{(\eta_0^3+\eta_0^2)\tilde{\omega}(k+3)^2}{36L}\right)\mathbb{E}[f(x^{(k+1)})-f^\star]+\mathbb{E}[\|z^{(k+1)}-x^\star\|^2]\nonumber\\
\le&\left(\frac{\eta_0(k+2)^2}{18L}-\frac{\eta_0(k+4)}{18L}\right)\mathbb{E}[f(x^{(k)})-f^\star]-\frac{\eta_0^2(k+3)^2}{144L^2}\left(4-(\eta_0+1)(1+2\eta_0\tilde{\omega})\right)\cdot\|\nabla f(y^{(k)})\|^2\nonumber\\
&+\|z^{(k)}-x^\star\|^2+\frac{(\eta_0^3+\eta_0^2)(k+3)^2}{144L^2}\left(4\tilde{\omega}LG^\star+\left(\frac{1}{n}+\tilde{\omega}\right)\tilde{\sigma}^2\right),\quad k=0,1,\cdots,K-1.\label{eq-thm-acc-sto-gen-cov-1}
\end{align}
Note that $\eta_0\le1$, and \eqref{eq-thm-acc-sto-gen-cov-R} implies that $\tilde{\omega}\le\frac{1}{2}$, we know $4-(\eta_0+1)(1+2\eta_0\tilde{\omega})\ge0$, thus \eqref{eq-thm-acc-sto-gen-cov-1} can be further simplified as
\begin{align}
&\left(\frac{\eta_0(k+3)^2}{18L}-\frac{\eta_0^2\tilde{\omega}(k+3)^2}{18L}\right)\mathbb{E}[f(x^{(k+1)})-f^\star]+\mathbb{E}[\|z^{(k+1)}-x^\star\|^2]\nonumber\\
\le&\left(\frac{\eta_0(k+2)^2}{18L}-\frac{\eta_0(k+4)}{18L}\right)\mathbb{E}[f(x^{(k)})-f^\star]+\|z^{(k)}-x^\star\|^2+\frac{\eta_0^2(k+3)^2}{72L^2}\left(4\tilde{\omega}LG^\star+\left(\frac{1}{n}+\tilde{\omega}\right)\tilde{\sigma}^2\right).\label{eq-thm-acc-sto-gen-cov-2}
\end{align}
Summing up \eqref{eq-thm-acc-sto-gen-cov-2} from $k=0$ to $K-1$, we obtain
\begin{align}
&\left(\frac{\eta_0(K+2)^2}{18L}-\frac{\eta_0^2\tilde{\omega}(K+2)^2}{18L}\right)\mathbb{E}[f(x^{(K)})-f^\star]+\sum_{k=1}^{K-1}\left(\frac{\eta_0(k+2)^2}{18L}-\frac{\eta_0^2\tilde{\omega}(k+2)^2}{18L}\right)\mathbb{E}[f(x^{(k)})-f^\star]\nonumber\\
\le&\sum_{k=1}^{K-1}\left(\frac{\eta_0(k+2)^2}{18L}-\frac{\eta_0(k+4)}{18L}\right)\mathbb{E}[f(x^{(k)})-f^\star]+\|z^0-x^\star\|^2-\mathbb{E}[\|z^{K}-x^\star\|^2]\nonumber\\
&+\sum_{k=0}^{K-1}\frac{\eta_0^2(k+3)^2[4\tilde{\omega}LG^\star+(\tilde{\omega}+1/n)\tilde{\sigma}^2]}{72L^2},
\end{align}
thus we further have
\begin{align}
&\left(\frac{\eta_0(K+2)^2}{18L}-\frac{\eta_0^2\tilde{\omega}(K+2)^2}{18L}\right)\mathbb{E}[f(x^{(K)})-f^\star]\nonumber\\
\le&\sum_{k=1}^{K-1}\left(\frac{\eta_0^2\tilde{\omega}(k+2)^2}{18L}-\frac{\eta_0(k+4)}{18L}\right)\mathbb{E}[f(x^{(k)})-f^\star]+\|z^0-x^\star\|^2-\mathbb{E}[\|z^{(K)}-x^\star\|^2]\nonumber\\
&+\sum_{k=0}^{K-1}\frac{\eta_0^2(k+3)^2[4\tilde{\omega}LG^\star+(\tilde{\omega}+1/n)\tilde{\sigma}^2]}{72L^2}\nonumber\\
\overset{(a)}{\le}&\|z^0-x^\star\|^2+\frac{\eta_0^2(K+3)^3[4\tilde{\omega}LG^\star+(\tilde{\omega}+1/n)\tilde{\sigma}^2]}{216L^2}\nonumber\\
\le&\Delta_x+\frac{8\eta_0^2(K+2)^3[4\tilde{\omega}LG^\star+(\tilde{\omega}+1/n)\tilde{\sigma}^2]}{729L^2},\label{eq-thm-acc-sto-gen-cov-3}
\end{align}
where (a) is due to $\eta_0\tilde{\omega}\le\frac{1}{K+2}$ because 
\begin{align}
\eta_0\le\frac{27L\sqrt{\Delta_x}}{2\sqrt{2}(K+2)^{3/2}\sqrt{4\tilde{\omega}LG^\star+(\tilde{\omega}+1/n)\tilde{\sigma}^2}},\quad\mathrm{and}\ \tilde{\omega}\le\frac{2\sqrt{2}\sqrt{\tilde{\sigma}^2/n}}{27L\sqrt{\Delta_x}}\mathrm{(since \eqref{eq-thm-acc-sto-gen-cov-5} holds)}.
\end{align}
Also note that $\eta_0\le1$ and $\tilde{\omega}\le1/2$ implies
\begin{align}
\frac{\eta_0^2\tilde{\omega}(K+2)^2}{18L}\le\frac{\eta_0(K+2)^2}{36L},
\end{align}
we have
\begin{align}
\mathbb{E}[f(x^{(K)})-f^\star]\le&\frac{36L\Delta_x}{\eta_0(K+2)^2}+\frac{32\eta_0(K+2)[4\tilde{\omega}LG^\star+(\tilde{\omega}+1/n)\tilde{\sigma}^2]}{81L}\nonumber\\
\le&\frac{36L\Delta_x}{(K+2)^2}+\frac{16\sqrt{2\Delta_x}\cdot\sqrt{4\tilde{\omega}LG^\star+(\tilde{\omega}+1/n)\tilde{\sigma}^2}}{3\sqrt{K+2}}\nonumber\\
\le&\frac{36L\Delta_xR^2}{T^2}+\frac{16\sqrt{2\Delta_x}\cdot\sqrt{4R\tilde{\omega}LG^\star+(\tilde{\omega}+1/n)\sigma^2}}{3\sqrt{T}},\label{eq-thm-acc-sto-gen-cov-4}
\end{align}
Combining \eqref{eq-thm-acc-sto-gen-cov-4}\eqref{eq-thm-acc-sto-gen-cov-5} and note that $\tilde{\omega}\le\frac{1}{n}$, we obtain
\begin{align}
    \mathbb{E}[f(x^{(k)})-f^\star]\le\frac{16L\Delta_xR^2}{T^2}+\frac{12\sqrt{\Delta_x\sigma^2}}{\sqrt{nT}},
\end{align}
which implies the complexity in \eqref{eq-upper-bound-unbiased-gc}.\hfill$\square$

\subsubsection{Proof of the Contractive Case in Theorem \ref{thm:upper-gc-contractive}}\label{app:upper-convex-contractive}

Let $M_k=(k+1)^{3/2}$, $\lambda=0$ and set 
\begin{equation}
R=\left\lceil\max\left\{\frac{4}{\delta}\ln\left(\frac{4}{\delta}\right),\frac{1}{\delta}\ln\left(24K^3+\frac{100n^2K^3(LG^\star)^2}{\sigma^4}+5nK^{3/2}\right)\right\}\right\rceil,\label{eq-thm-contractive-gen-R}
\end{equation}
and
\begin{equation}
\eta=\frac{25\sqrt{2\Delta_x}}{(K+1)^{3/2}\sqrt{20K^{3/2}\tilde{\omega}LG^\star+(1/n+5K^{3/2}\tilde{\omega})\tilde{\sigma}^2}+25L\sqrt{2\Delta_x}},\label{eq-thm-contractive-gen-eta}
\end{equation}
by Lemma \ref{lm-contractive-strong} we have
\begin{align}
    &\left(1-\frac{4\tilde{\omega}\eta_0}{25}[2\eta_0+(k+1)^{3/2}+7]\right)\mathbb{E}[f(x^{(k+1)})-f^\star]+\frac{50L}{(k+2)^2\eta_0}\mathbb{E}[\|z^{(k+1)}-x^\star\|^2]\nonumber\\
    \le&\frac{k}{k+2}[f(x^{(k)})-f^\star]+\frac{50L[1+(k+1)^{-3/2}]}{(k+2)^2\eta_0}\|z^{(k)}-x^\star\|^2\nonumber+\left(\frac{\eta_0^2+\eta_0}{25nL}+\frac{\tilde{\omega}\eta_0[2\eta_0+(k+1)^{3/2}+7]}{25L}\right)\tilde{\sigma}^2\\
    &-\frac{\eta_0}{L}\left(\frac{3-2\eta_0}{50}-\frac{6\tilde{\omega}\eta_0[2\eta_0+(k+1)^{3/2}+7]}{125}\right)\|\nabla f(y^{(k)})\|^2+\frac{4\tilde{\omega}\eta_0[2\eta_0+(k+1)^{3/2}+7]}{25L}LG\star,\label{eq-thm-0110-1}
\end{align}
where $\eta_0:=L\eta$. Define $Q_0:=1$,  $Q_k:=\prod_{i=1}^{k}(1+i^{-3/2})$, $\forall k\ge1$, and $Q_\infty:=\prod_{i=1}^\infty(1+i^{-3/2})<+\infty$. We have the following inequalities hold for $0\le k\le K-1$:
\begin{equation*}
    \begin{cases}
    \eta_0\le1< Q_{k+1}, &\\
     1-\frac{4\tilde{\omega}\eta_0}{25}[2\eta_0+(k+1)^{3/2}+7]\ge1-\frac{8(k+1)^{3/2}}{5}\tilde{\omega}\ge1-\frac{(k+2)^{-3/2}}{1+(k+2)^{-3/2}}=\frac{Q_{k+1}}{Q_{k+2}},&\mathrm{since}\ \tilde{\omega}\le\frac{1}{8K^3},\\
     \frac{3-2\eta_0}{50}-\frac{6\tilde{\omega}\eta_0[2\eta_0+(k+1)^{3/2}+7]}{125}\ge\frac{1}{50}-\frac{12(k+1)^{3/2}}{25}\tilde{\omega}\ge0,&\mathrm{since}\ \tilde{\omega}\le\frac{1}{24K^{3/2}},\\
     \frac{2\tilde{\omega}\eta_0[2\eta_0+(k+1)^{3/2}+7]}{25L}\le\frac{4K^{3/2}\tilde{\omega}\eta_0}{5L},&\\
     \frac{\eta_0^2+\eta_0}{25nL}+\frac{\tilde{\omega}\eta_0[2\eta_0+(k+1)^{3/2}+7]}{25L}\le\frac{2\eta_0}{25nL}+\frac{2K^{3/2}\tilde{\omega}\eta_0}{5L}.&
     \end{cases}
\end{equation*}
Multiplying $(k+2)^2\eta_0/(50LQ_{k+1})$ to both sides of \eqref{eq-thm-0110-1} and noting the above inequalities, we obtain
\begin{align}
&\frac{(k+2)^2\eta_0}{50LQ_{k+2}}\mathbb{E}[f(x^{(k+1)})-f^\star]+\frac{1}{Q_{k+1}}\mathbb{E}[\|z^{(k+1)}-x^\star\|^2]\nonumber\\
\le&\frac{[(k+1)^2-1]\eta_0}{50LQ_{k+1}}\mathbb{E}[f(x^{(k)})-f^\star]+\frac{1}{Q_k}\mathbb{E}[\|z^{(k)}-x^\star\|^2]\nonumber\\
&+\frac{(k+2)^2\eta_0}{50L}\left[\frac{8K^{3/2}\tilde{\omega}\eta_0}{5L}LG^\star+\left(\frac{2\eta_0}{25nL}+\frac{2K^{3/2}\tilde{\omega}\eta_0}{5L}\right)\tilde{\sigma}^2\right], \quad k=0,1,\cdots,K-1.\label{eq-thm-0110-2}
\end{align}
Summing up \eqref{eq-thm-0110-2} for $k=0,1,\cdots,K-1$, we obtain
\begin{align}
\mathbb{E}[f(x^{(K)})-f^\star]\le&\frac{50LQ_{K+1}}{(K+1)^2\eta_0}\|z^0-x^\star\|^2+(K+1)\eta_0Q_{K+1}\cdot\left[\frac{4K^{3/2}\tilde{\omega}}{5L}LG^\star+\left(\frac{1}{25nL}+\frac{K^{3/2}\tilde{\omega}}{5L}\right)\tilde{\sigma}^2\right]\nonumber\\
\le&\frac{50Q_\infty L\Delta_x}{(K+1)^2}+\frac{2\sqrt{2\Delta_x}Q_\infty \sqrt{20K^{3/2}\tilde{\omega}LG^\star+(1/n+5K^{3/2}\tilde{\omega})\tilde{\sigma}^2}}{\sqrt{K+1}}\nonumber\\
\le&\frac{50Q_\infty L\Delta_xR^2}{T^2}+\frac{2\sqrt{2\Delta_x}Q_\infty\sqrt{20K^{3/2}R\tilde{\omega}LG\star+(1/n+5K^{3/2}\tilde{\omega})\sigma^2}}{\sqrt{T}},
\end{align}
by $R\tilde{\omega}\le\sigma^2/(10nK^{3/2}LG^\star)$ (since $\tilde{\omega}\le\sigma^4/(100n^2K^3(LG^\star)^2)$ and Lemma \ref{lm-R-omega} holds) and $\tilde{\omega}\le1/(5nK^{3/2})$, we obtain 
\begin{align}
    \EE[f(x^{(K)})-f^\star]\le\frac{50Q_{\infty}L\Delta_xR^2}{T^2}+\frac{4Q_\infty\sqrt{2\Delta_x\sigma^2}}{\sqrt{nT}},
\end{align}
which implies the complexity in \eqref{eq-upper-bound-contractive-gc}.\hfill$\square$

\subsection{Strongly Convex Case}

\subsubsection{Convergence of Algorithm \ref{alg:NEOLITHIC} in the Strongly Convex Scenario}\label{app:upper-sconvex-single}

\begin{theorem}\label{thm:app-sconvex}
Given $n\ge1$, precision $\epsilon>0$ and $G^\star=f^\star-\frac{1}{n}\sum_{i=1}^nf_i^\star$, in which $f^\star$ and $f_i^\star$ are minimum of global objective $f$ and local objective $f_i$ of worker $i$,
for any $\{f_i\}_{i=1}^n\subseteq\cF_{L,\mu}^{\Delta_x}$ with $\mu>0$ and $\{O_i\}_{i=1}^n\subseteq\cO_{\sigma^2}$ and $\{C_i\}_{i=1}^n\subseteq\cC_\delta$, if we let $\eta=1/L$, $\gamma_k\equiv\gamma=\sqrt{\mu/L}$,
\begin{equation}
R=\left\lceil\max\left\{\frac{4}{\delta}\ln\left(\frac{4}{\delta}\right),
\frac{1}{\delta}\ln\left(\frac{4n^2K^2(LG^\star)^2}{\sigma^4}+nK+\frac{96}{\gamma^2}\right)
\right\}\right\rceil
\end{equation}
and
\begin{equation}
p=\max\left\{5,\frac{\gamma T}{4R\ln(n\mu  g(x^{(0)})T/\sigma^2)}\right\},
\end{equation}
it holds that
\begin{equation}\label{eq-app-sconvex}
    \EE[f(x^{(K)})-f^\star]\le\exp\left(-\frac{\sqrt{\mu/L}T}{20R}\right)g(x^{(0)})+\frac{77\sigma^2}{\mu nT}+\frac{34\sigma^2\ln(n\mu g(x^{(0)})T/\sigma^2)}{\mu nT},
\end{equation}
and thus the iteration complexity to guarantee precision $\epsilon$ is $\tilde{\cO}\left(\frac{g(x^{(0)})}{\delta\epsilon}\sqrt{\frac{L}{\mu}}\ln\left(\frac{g(x^{(0)})}{\epsilon}\right)+\frac{\sigma^2}{\mu n\epsilon}\ln\left(\frac{g(x^{(0)})}{\epsilon}\right)\right)$, where $g(x):=f(x)-f^\star+\frac{25}{81}\mu\|x-x^\star\|^2$ and  $\tilde{\cO}(\cdot)$ hides logarithm factors of $n, \mu, L, \sigma, G^\star, \delta$.

\end{theorem}
\begin{proof}
Let $\lambda=1$, by Lemma \ref{lm-contractive-strong} we have
\begin{align}
    &(1-4\tilde{\omega}LN_k)\mathbb{E}[f(x^{(k+1)})-f^\star]+\frac{\gamma}{2p-\gamma}\cdot\frac{p\gamma L}{2(2-\gamma)}\mathbb{E}[\|w^{(k+1)}-x^\star\|^2]\nonumber\\
    \le&\left(1-\frac{\gamma}{2p-\gamma}\right)[f(x^{(k)})-f^\star]+\left(1-\frac{\gamma}{2p-\gamma}\right)\left(1+\frac{1}{M_k}\right)\cdot\frac{\gamma}{2p-\gamma}\cdot\frac{p\gamma L}{2(2-\gamma)}\|w^{(k)}-x^\star\|^2\nonumber\\
    &-\frac{\gamma}{2(2p-\gamma)}\cdot\left(\mu-\frac{\gamma}{2p-\gamma}\cdot\frac{p\gamma L}{2-\gamma}\cdot\left(1+\frac{1}{M_k}\right)\right)\|y^{(k)}-x^\star\|^2+4\tilde{\omega}N_kLG^\star\nonumber\\
    &-\left(\frac{1}{2pL}-\frac{1}{p^2L}-\frac{2-\gamma}{p\gamma L}\cdot\frac{\gamma}{2p-\gamma}-\frac{6\tilde{\omega}N_k}{p}\right)\|\nabla f(y^{(k)})\|^2+\left(\frac{1}{np^2L}+\frac{2-\gamma}{p\gamma Ln}\cdot\frac{\gamma}{2p-\gamma}+\tilde{\omega}N_k\right)\tilde{\sigma}^2.\label{eq-thm-0111-1}
\end{align}
If we further let
\begin{align*}
    M_k=\frac{8p(p-\gamma)}{\gamma(\gamma+2p)},
\end{align*}
we obtain
\begin{equation*}
    \begin{cases}
    N_k=\frac{M_k+2}{L}\cdot\frac{2-\gamma}{p(2p-\gamma)}+\frac{2}{p^2L}+\frac{1}{pL}\le\frac{6}{p\gamma L},&\mathrm{since}\ \gamma\le1<5\le p,\\
    1-4\tilde{\omega}N_k\ge1-\frac{24\tilde{\omega}}{p\gamma}\ge1-\frac{\gamma}{4p-\gamma},&\mathrm{since}\ \tilde{\omega}\le\frac{\gamma^2}{96},\\
    1-\frac{\gamma}{2p-\gamma}\le\left(1-\frac{\gamma}{4p}\right)\left(1-\frac{\gamma}{4p-\gamma}\right),&\\
    \left(1-\frac{\gamma}{2p-\gamma}\right)\left(1+\frac{1}{M_k}\right)=1-\frac{\gamma}{4p},&\\
    \mu-\frac{\gamma}{2p-\gamma}\cdot\frac{p\gamma L}{2-\gamma}\cdot\left(1+\frac{1}{M_k}\right)=\mu\left(1-\frac{4p-\gamma}{8(2-\gamma)(p-\gamma)}\right)\ge0,&\mathrm{since}\ \gamma\le1<5\le p,\\
    \frac{1}{2pL}-\frac{1}{p^2L}-\frac{2-\gamma}{p\gamma L}\cdot\frac{\gamma}{2p-\gamma}-\frac{6\tilde{\omega}N_k}{p}\ge\frac{1}{L}\left(\frac{1}{2p}-\frac{1}{p^2}-\frac{1}{p^2}-\frac{3}{8p^2}\right)\ge0,&\mathrm{since}\ \tilde{\omega}\le\frac{\gamma^2}{96},\\
    \frac{1}{p^2L}+\frac{2-\gamma}{p\gamma L}\cdot\frac{\gamma}{2p-\gamma}\le\frac{2}{p^2L},&\mathrm{since}\ \gamma\le1<5\le p.
    \end{cases}
\end{equation*}
Combining the above inequalities with \eqref{eq-thm-0111-1}, we obtain
\begin{align*}
    &\left(1-\frac{\gamma}{4p-\gamma}\right)\mathbb{E}[f(x^{(k+1)})-f^\star]+\frac{p\gamma^2L}{2(2-\gamma)(2p-\gamma)}\mathbb{E}[\|w^{(k+1)}-x^\star\|^2]\\
    \le&\left(1-\frac{\gamma}{4p}\right)\left(1-\frac{\gamma}{4p-\gamma}\right)[f(x^{(k)})-f^\star]+\left(1-\frac{\gamma}{4p}\right)\cdot\frac{p\gamma^2L}{2(2-\gamma)(2p-\gamma)}\|w^{(k)}-x^\star\|^2+\frac{24\tilde{\omega}}{p\gamma }G^\star\\
    &+\left(\frac{2}{np^2L}+\frac{6\tilde{\omega}}{p\gamma L}\right)\tilde{\sigma}^2,
\end{align*}
which further implies that
\begin{align}
    \left(1-\frac{\gamma}{4p-\gamma}\right)\mathbb{E}[f(x^{(k)})-f^\star]\le&\left(1-\frac{\sqrt{\mu/L}}{4p}\right)^k\left(1-\frac{\gamma}{4p-\gamma}\right)g(x^{(0)})+\frac{96\tilde{\omega}G^\star}{\gamma^2}+\left(\frac{8}{np\gamma L}+\frac{24\tilde{\omega}}{\gamma^2L}\right)\tilde{\sigma}^2.\label{eq:plmokn01}
\end{align}
By the selection of $p$, we have
\begin{align}
    \left(1-\frac{\sqrt{\mu/L}}{4p}\right)^kg(x^{(0)})\le&\left(1-\frac{\sqrt{\mu/L}}{20}\right)^{k}g(x^{(0)})+\left(1-\frac{\gamma}{4}\cdot\frac{4\ln(n\mu g(x^{(0)}) T/\sigma^2)}{\gamma K}\right)^{k}g(x^{(0)})\\
    \le&\left(1-\frac{\sqrt{\mu/L}}{20}\right)^{k}g(x^{(0)})+\frac{\sigma^2}{n\mu T}\label{eq:plmokn02}
\end{align}
and
\begin{align}
    \frac{8\tilde{\sigma}^2}{np\gamma L}\le&\frac{32\sigma^2\ln(n\mu g(x^{(0)})T/\sigma^2)}{n\mu T}.\label{eq:plmokn03}
\end{align}
Note that $R\tilde{\omega}\le\sigma^2/(2nKLG^\star)$ (since $\tilde{\omega}\le\sigma^4/(4n^2K^2(LG^\star)^2)$ and Lemma \ref{lm-R-omega} holds) and $\tilde{\omega}\le1/nK$, we have
\begin{align}
    \frac{96\tilde{\omega}LG^\star}{\gamma^2L}+\frac{24\tilde{\omega}}{\gamma^2L}\tilde{\sigma}^2\le76\left(1-\frac{\gamma}{4p-\gamma}\right)\frac{\sigma^2}{n\mu T}.\label{eq:plmokn04}
\end{align}
Combining \eqref{eq:plmokn01}\eqref{eq:plmokn02}\eqref{eq:plmokn03}\eqref{eq:plmokn04}, and let $k=K$ we obtain
\begin{align*}
    \mathbb{E}[f(x^{(K)})-f^\star]\le&\left(1-\frac{\sqrt{\mu/L}}{20}\right)^{k}g(x^{(0)})+\frac{77\sigma^2}{n\mu T}+\frac{34\sigma^2\ln(n\mu  g(x^{(0)})T/\sigma^2)}{n\mu T}\\
    \le&\exp\left(-\frac{\sqrt{\mu/L}T}{20R}\right)g(x^{(0)})+\frac{77\sigma^2}{n\mu T}+\frac{34\sigma^2\ln(n\mu g(x^{(0)})T/\sigma^2)}{n\mu T},
\end{align*}
which is exactly \eqref{eq-app-sconvex}.
\end{proof}

\begin{remark} Although the complexity stated in Theorem \ref{thm:app-sconvex} is sufficient for the analysis in the multi-stage scenario, it is rather coarse and has a huge gap over the lower bound. In fact, we can apply $g(x^{(0)})\le L\Delta_x$ to the convergence result in \eqref{eq-app-sconvex} to achieve a complexity as $\tilde{\cO}\left(\frac{1}{\delta}\sqrt{\frac{L}{\mu}}\ln\left(\frac{1}{\epsilon}\right)\ln\left(\ln\left(\frac{1}{\epsilon}\right)\right)+\frac{\sigma^2}{\mu n\epsilon}\ln\left(\frac{1}{\epsilon}\right)\right)$, which has a gap of $\ln(1/\epsilon)$ over the lower bound.
\end{remark}

\subsubsection{Proof of Theorem \ref{thm:upper-sc-contractive}}\label{app:upper-sconvex}

Set the parameters $K^{[s]}$, $R^{[s]}$, $\eta^{[s]}$, $\{\gamma_k\}_{k=0}^{K^{[s]}-1}$ according to Theorem \ref{thm:app-sconvex} in order to guarantee precision $\EE[f(x^{[s+1]})-f^\star]\le2^{-(s+1)}g(x^{[0]})$. By strongly convexity, this stop criterion implies $g(x^{[s+1]})\le\frac{131}{81}\cdot2^{-(s+1)}g(x^{[0]})$. Note that a total of $T^{(s)}=\tilde{\cO}\left(\frac{1}{\delta}\sqrt{\frac{L}{\mu}}+\frac{2^s\sigma^2}{\mu ng(x^{[0]})}\right)$ steps is sufficient to stop stage $s$, the stop criterion of the last stage guarantees $\EE[f(x^{[S]})-f^\star]\le\epsilon$, and $g(x^{[0]})\le L\Delta_x$, the total complexity is thus as described in \eqref{eq-upper-bound-contractive-sc}.\\
For unbiased compressors, we can use the following strategy to regard them as contractive ones. For any input vector $v$ to compress, we communicate through the same encoding results as the original compressors, while decoding these messages as $(1+\omega)^{-1}C(v)$ (instead of $C(v)$). By Lemma \ref{lem:un-con}, the effect of compressor $C$ under this scenario is equivalent to a $(1+\omega)^{-1}$-contractive compressor. Consequently, by directly replacing $\delta$ by $(1+\omega)^{-1}$ in \eqref{eq-upper-bound-contractive-sc}, we immediately obtain the complexity result as \eqref{eq-upper-bound-unbiased-sc}.  
\hfill$\square$

\end{document}